%% file: main.tex
\documentclass[english,10pt,twocolumn,letterpaper]{article}
\usepackage[T1]{fontenc}
\usepackage[utf8]{inputenc}
\usepackage{color}
\usepackage{babel}
\usepackage{array}
\usepackage{float}
\usepackage{booktabs}
\usepackage{multirow}
\usepackage{amsmath}
\usepackage{amsthm}
\usepackage{amssymb}
\usepackage{cancel}
\usepackage{graphicx}
\PassOptionsToPackage{normalem}{ulem}
\usepackage{ulem}
\ifx\hypersetup\undefined
  \AtBeginDocument{%
    \hypersetup{unicode=true}
  }
\else
  \hypersetup{unicode=true}
\fi

\makeatletter

\providecommand{\tabularnewline}{\\}
\floatstyle{ruled}
\newfloat{algorithm}{tbp}{loa}
\providecommand{\algorithmname}{Algorithm}
\floatname{algorithm}{\protect\algorithmname}

\theoremstyle{plain}
\newtheorem{prop}{\protect\propositionname}
\theoremstyle{plain}
\newtheorem{cor}{\protect\corollaryname}

\usepackage{cvpr}

\usepackage{xcolor}
\usepackage{caption}

\definecolor{cvprblue}{rgb}{0.21,0.49,0.74}
\usepackage[pagebackref,breaklinks,colorlinks,allcolors=cvprblue]{hyperref}
\usepackage{url}
\usepackage{algorithm,algpseudocode}
\usepackage{subcaption}

\def\paperID{5505} 
\def\confName{CVPR}
\def\confYear{2025}

\title{$h$-Edit: Effective and Flexible Diffusion-Based Editing via Doob's $h$-Transform}

\def\spaces{~~~~~~~~}

\author{
\medskip
Toan Nguyen$^*$\spaces{}Kien Do$^*$\spaces{}Duc Kieu\spaces{}Thin Nguyen\\
\emph{\{k.nguyen, k.do, v.kieu, thin.nguyen\}@deakin.edu.au}\\
Applied Artificial Intelligence Institute (A2I2), Deakin University, Australia\\
$^*$ Equal contribution
}

\makeatother

\providecommand{\corollaryname}{Corollary}
\providecommand{\propositionname}{Proposition}

\begin{document}
\renewcommand{\contentsname}{Table of Content for Appendix}

\twocolumn[{%
\renewcommand\twocolumn[1][]{#1}%
\maketitle 
\begin{center} 
\centering{}\includegraphics[width=0.99\textwidth]{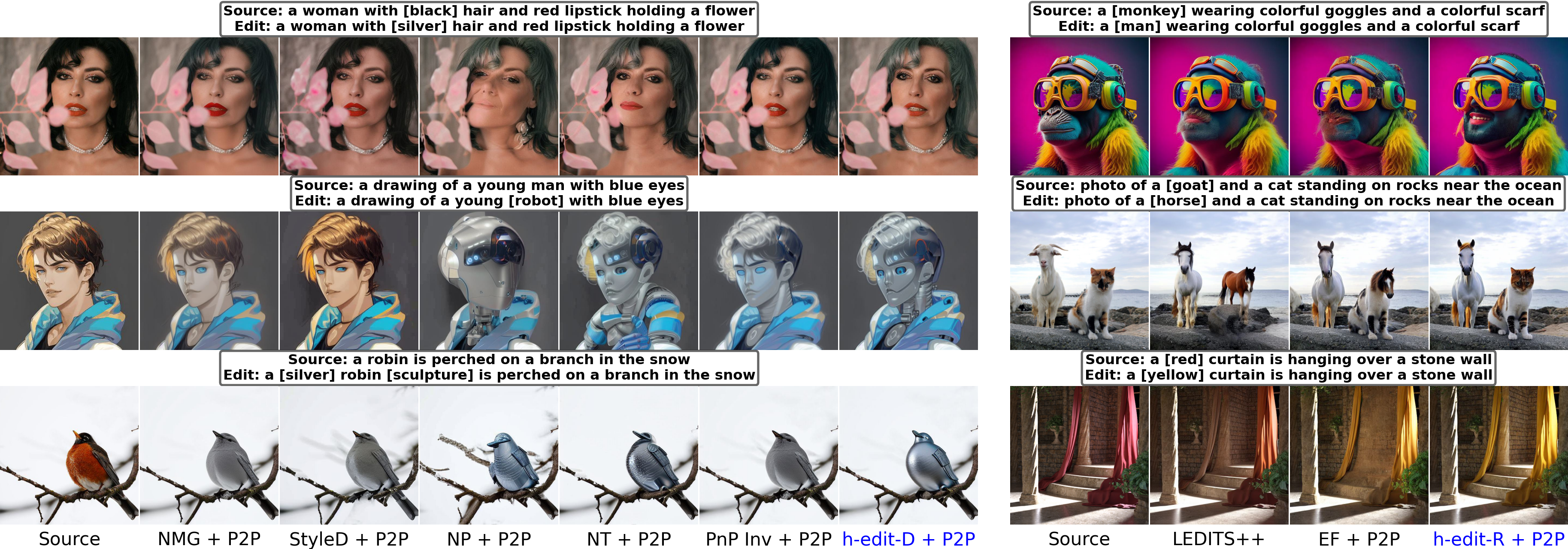}
\captionof{figure}{Qualitative comparison between $h$-Edit and other training-free editing baselines. Our method achieves more accurate and faithful edits than the baselines. Additional visualizations are provided in the Appendix.\label{fig:qualitative-text-guided}} 
\end{center} 
}]

\global\long\def\Model{\text{\emph{h}-Edit}}%
\global\long\def\argmax#1{\underset{#1}{\text{argmax }}}%
\global\long\def\argmin#1{\underset{#1}{\text{argmin }}}%
\global\long\def\Expect{\text{\ensuremath{\mathbb{E}}}}%
\global\long\def\orig{\text{orig}}%
\global\long\def\edit{\text{edit}}%
\global\long\def\rec{\text{rec}}%
\global\long\def\base{\text{base}}%
\global\long\def\style{\text{sty}}%
\global\long\def\source{\text{src}}%
\global\long\def\target{\text{tar}}%
\global\long\def\refer{\text{ref}}%
\global\long\def\face{\text{face}}%

\begin{abstract}
\input{abstract.tex}

\addtocontents{toc}{\protect\setcounter{tocdepth}{-1}}
\end{abstract}

\section{Introduction\label{sec:Introduction}}

\input{intro.tex}

\section{Preliminaries\label{sec:Prelim}}

\input{prelim.tex}

\section{Method\label{sec:Method}}

\input{method.tex}

\section{Related Work\label{sec:Related-Work}}

\input{related.tex}

\section{Experiments\label{sec:Experiments}}

\input{exp.tex}

\section{Conclusion\label{sec:Conclusion}}

\input{conclusion.tex}

\section*{Acknowledgement}

The experiments in this research were partially supported by AWS Cloud
services under the AWS Cloud Credit for Research Program, for which
Dr. Kien Do is the recipient.

\bibliographystyle{ieeenat_fullname}

\addtocontents{toc}{\protect\setcounter{tocdepth}{2}}

\newpage{}

\appendix
\onecolumn

\tableofcontents{}

\newpage{}

\input{appdx.tex}

\end{document}

%% file: abstract.tex
We introduce a theoretical framework for diffusion-based image editing
by formulating it as a reverse-time bridge modeling problem. This
approach modifies the backward process of a pretrained diffusion model
to construct a bridge that converges to an implicit distribution associated
with the editing target at time 0. Building on this framework, we
propose $h$-Edit, a novel editing method that utilizes Doob's $h$-transform
and Langevin Monte Carlo to decompose the update of an intermediate
edited sample into two components: a ``reconstruction'' term and
an ``editing'' term. This decomposition provides flexibility, allowing
the reconstruction term to be computed via existing inversion techniques
and enabling the combination of multiple editing terms to handle complex
editing tasks. To our knowledge, $h$-Edit is the first training-free
method capable of performing simultaneous text-guided and reward-model-based
editing. Extensive experiments, both quantitative and qualitative,
show that $h$-Edit outperforms state-of-the-art baselines in terms
of editing effectiveness and faithfulness. Our source code is available
at \url{https://github.com/nktoan/h-edit}.

%% file: intro.tex
Diffusion models \cite{sohl2015deep,song2019generative,ho2020denoising}
have established themselves as a powerful class of generative models,
achieving state-of-the-art performance in image generation \cite{song2021denoising}.
When combined with classifier-based~\cite{Dhariwal2021Diffusion}
or classifier-free guidance ~\cite{Ho2022Classifier}, these models
offer enhanced control, enabling a wide range of applications including
conditional generation~\cite{zhang2023adding,yu2023freedom}, image-to-image
translation~\cite{choi2021ilvr,saharia2022palette}, and image editing~\cite{meng2022sdedit,hertz2023prompttoprompt,huang2024diffusion}.
A prominent example is large-scale text-guided diffusion models~\cite{nichol2022glide,saharia2022photorealistic}
like Stable Diffusion (SD)~\cite{rombach2022high}, which have gained
widespread popularity for their ability to produce diverse high-quality
images that closely align with specified natural language descriptions. 

However, leveraging pretrained text-guided diffusion models for image
editing presents significant challenges, particularly in balancing
effective editing with faithful preservation of the unrelated content
in the original image. Moreover, combining text-guided editing with
other forms of editing to handle more complex requirements remains
a difficult task. Although recent advances in training-free image
editing have been proposed \cite{mokady2023null,hertz2023prompttoprompt,tumanyan2023plug,cho2024noise,huberman2024edit,ju2023direct},
most of these efforts focus on improving reconstruction quality through
better inversion techniques or attention map adjustment, while leaving
the editing part largely unchanged. Additionally, many of these methods
are based on heuristics or intuition, lacking a clear theoretical
foundation to justify their effectiveness. This limitation restricts
the generalization of these approaches to more complex scenarios where
multiple types of editing must be applied.

In this work, we aim to fill the theoretical gap by introducing a
theoretical framework for image editing, formulated as a \emph{reverse-time
bridge modeling} problem. Our approach modifies the \emph{backward}
process of a pretrained diffusion model using Doob's $h$-transform
\cite{doob1984classical,rogers2000diffusions,sarkka2019applied} to
create a bridge that converges to the distribution $p\left(x_{0}\right)h\left(x_{0},0\right)$
at time 0. Here, $p\left(x_{0}\right)$ represents the realism of
$x_{0}$, while $h\left(x_{0},0\right)$ captures the probability
that $x_{0}$ has the target property. To perform editing, we first
map the original image $x_{0}^{\orig}$ to its prior $x_{T}^{\orig}$
through the diffusion forward process. Starting from $x_{T}^{\edit}=x_{T}^{\orig}$,
we follow the bridge to generate an edited image $x_{0}^{\edit}$
by sampling from its transition kernel $p^{h}\left(x_{t-1}|x_{t}\right)$
using Langevin Monte Carlo (LMC) \cite{roberts1996exponential,welling2011bayesian}. 

Building on the decomposability of $p^{h}\left(x_{t-1}|x_{t}\right)$,
we propose \emph{$h$-Edit} - a novel editing method that disentangles
the update of $x_{t-1}^{\edit}$ into a \emph{``reconstruction''}
term $x_{t-1}^{\base}$ (capturing editing faithfulness) and an \emph{``editing''}
term (capturing editing effectiveness). This design provides significant
flexibility, as the editing term can be easily customized for different
tasks with minimal interference in non-edited regions. $h$-Edit updates
can be either explicit or implicit, with $\nabla\log h\left(x_{t},t\right)$
and $\nabla\log h\left(x_{t-1},t-1\right)$ being the corresponding
editing terms, respectively. In the latter case, $h$-Edit can also
be interpreted from an optimization perspective where $\log h\left(x_{t-1},t-1\right)$
is maximized w.r.t. $x_{t-1}$, taking $x_{t-1}^{\base}$ as the initial
value. This allows for multiple optimization steps to enhance editing
effectiveness.

While $x_{t-1}^{\base}$ can generally be estimated by leveraging
existing inversion techniques \cite{song2021denoising,mokady2023null,ju2023direct,huberman2024edit},
the computation of $\nabla\log h\left(x_{t-1},t-1\right)$ depends
on the chosen $h$-function. In this work, we present several key
designs of the $h$-function tailored to popular editing tasks, including
text-guided editing with SD and editing with external reward models
on clean data. Furthermore, by treating $\log h$ as a negative energy
function, we can easily combine multiple $h$-functions to create
a ``product of $h$-experts'', which enables compositional editing.

Through extensive experiments on a range of editing tasks - including
text-guided editing, combined text-guided and style editing, and face
swapping - we demonstrate strong editing capabilities of $h$-Edit.
Both quantitative and qualitative results indicate that $h$-Edit
not only significantly outperforms existing state-of-the-art methods
in text-guided editing but also excels in the two other tasks. Our
method effectively handles various difficult editing cases in the
PIE-Bench dataset where existing methods fall short. To our knowledge,
$h$-Edit is the \emph{first} diffusion-based training-free editing
method that supports simultaneous text-guided and reward-model-based
editing.

%% file: prelim.tex
\subsection{Diffusion Models}

Diffusion models \cite{sohl2015deep,song2019generative,ho2020denoising}
iteratively transform the data distribution $p\left(x_{0}\right)$
into the prior distribution $p\left(x_{T}\right)=\mathcal{N}\left(0,\mathrm{I}\right)$
via a \emph{predefined forward} stochastic process characterized by
$p\left(x_{t}|x_{t-1}\right)$, and learn the \emph{reverse} transition
distribution $p_{\theta}\left(x_{t-1}|x_{t}\right)$ to map $p\left(x_{T}\right)$
back to $p\left(x_{0}\right)$. Given the Gaussian form and Markov
property of $p\left(x_{t}|x_{t-1}\right)$, $p\left(x_{t}|x_{0}\right)$
is a Gaussian distribution $\mathcal{N}\left(a_{t}x_{0},\sigma_{t}^{2}\mathrm{I}\right)$,
allowing $x_{t}$ to be sampled from $p\left(x_{t}|x_{0}\right)$
as follows:
\begin{equation}
x_{t}=a_{t}x_{0}+\sigma_{t}\epsilon\label{eq:diffusion_sample}
\end{equation}
with $\epsilon\sim\mathcal{N}\left(0,\mathrm{I}\right)$. In DDPM
\cite{ho2020denoising}, $a_{t}=\sqrt{\bar{\alpha}_{t}}$ and $\sigma_{t}=\sqrt{1-\bar{\alpha}_{t}}$.
$p_{\theta}\left(x_{t-1}|x_{t}\right)$ is parameterized as a Gaussian
distribution $\mathcal{N}\left(\mu_{\theta,\omega,t,t-1}\left(x_{t}\right),\omega_{t,t-1}^{2}\mathrm{I}\right)$
with the mean
\begin{align}
 & \mu_{\theta,\omega,t,t-1}\left(x_{t}\right):=\nonumber \\
 & \ \ \ \frac{a_{t-1}}{a_{t}}x_{t}+\left(\sqrt{\sigma_{t-1}^{2}-\omega_{t,t-1}^{2}}-\frac{\sigma_{t}a_{t-1}}{a_{t}}\right)\epsilon_{\theta}\left(x_{t},t\right)\label{eq:backward_trans_prob}
\end{align}
Here, $\omega_{t,t-1}=\lambda\sigma_{t-1}\sqrt{1-\frac{a_{t}^{2}\sigma_{t-1}^{2}}{a_{t-1}^{2}\sigma_{t}^{2}}}$
with $\lambda\in\left[0,1\right]$. $\lambda=0$ and $\lambda=1$
correspond to DDIM sampling \cite{song2021denoising} and DDPM sampling
\cite{ho2020denoising}, respectively. Eq.~\ref{eq:backward_trans_prob}
implies that $x_{t-1}\sim p_{\theta}\left(x_{t-1}|x_{t}\right)$ is
given by:
\begin{equation}
x_{t-1}=\mu_{\theta,\omega,t,t-1}\left(x_{t}\right)+\omega_{t,t-1}z_{t}\label{eq:backward_sampling}
\end{equation}
with $z_{t}\sim\mathcal{N}\left(0,\mathrm{I}\right)$. Diffusion models
support conditional generation via classifier-based \cite{Dhariwal2021Diffusion}
and classifier-free \cite{Ho2022Classifier} guidances. The latter
is more prevalent, with Stable Diffusion (SD) \cite{rombach2022high}
serving as a notable example. In SD, both the unconditional and text-conditional
noise networks - $\epsilon_{\theta}\left(x_{t},t,\varnothing\right)$
and $\epsilon_{\theta}\left(x_{t},t,c\right)$ - are learned, and
their linear combination $\tilde{\epsilon}_{\theta}\left(x_{t},t,c\right):=w\epsilon_{\theta}\left(x_{t},t,c\right)+\left(1-w\right)\epsilon_{\theta}\left(x_{t},t,\varnothing\right)$,
with $w>0$ denoting the guidance weight, is often used for sampling.
This results in the following sampling step for SD:
\begin{equation}
x_{t-1}=\tilde{\mu}_{\theta,\omega,t,t-1}\left(x_{t},c\right)+\omega_{t,t-1}z_{t}\label{eq:DDIM_sampling_SD}
\end{equation}
where $\tilde{\mu}_{\theta,\omega,t,t-1}$ follows the same form as
$\mu_{\theta,\omega,t,t-1}\left(x_{t}\right)$ in Eq.~\ref{eq:backward_trans_prob}
but with $\epsilon_{\theta}\left(x_{t},t\right)$ replaced by $\tilde{\epsilon}_{\theta}\left(x_{t},t,c\right)$.

\subsection{Image Editing with Stable Diffusion}

The design of SD facilitates text-guided image editing which involves
modifying some attributes of the original image $x_{0}^{\orig}$ while
preserving other features (e.g., background) by adjusting the corresponding
text prompt $c^{\orig}$. A naive approach is mapping $x_{0}^{\orig}$
to $x_{T}^{\orig}$ using DDIM inversion w.r.t. $c^{\orig}$, followed
by generating $x_{0}^{\edit}$ from $x_{T}^{\edit}=x_{T}^{\orig}$
via DDIM sampling (Eq.~\ref{eq:DDIM_sampling_SD}) w.r.t. $c^{\edit}$
- the edited version of $c^{\orig}$. DDIM inversion is the reverse
of DDIM sampling, which achieves nearly exact reconstruction in the
unconditional case~\cite{song2021denoising,hertz2023prompttoprompt}.
For SD, DDIM inversion is expressed as:
\begin{equation}
x_{t}=\frac{a_{t}}{a_{t-1}}x_{t-1}+\left(\sigma_{t}-\frac{\sigma_{t-1}a_{t}}{a_{t-1}}\right)\tilde{\epsilon}_{\theta}\left(x_{t-1},t-1,c\right)\label{eq:DDIM_inversion_SD}
\end{equation}
However, there is a mismatch between $\tilde{\epsilon}_{\theta}\left(x_{t},t,c^{\edit}\right)$
and $\tilde{\epsilon}_{\theta}\left(x_{t-1},t-1,c^{\orig}\right)$
during sampling and inversion, causing $x_{0}^{\edit}$ to be significantly
different from $x_{0}^{\orig}$. Therefore, much of the research on
SD text-guided image editing focuses on improving reconstruction.
These inversion methods can be broadly classified into deterministic-inversion-based
\cite{mokady2023null,li2023stylediffusion,dong2023prompt,ju2023direct}
and random-inversion-based \cite{wu2023latent,huberman2024edit} techniques.
Edit Friendly (EF) \cite{huberman2024edit} - a state-of-the-art random-inversion-based
method - can be formulated under the following framework:
\begin{align}
u_{t}^{\orig} & =x_{t-1}^{\orig}-\tilde{\mu}_{\theta,\omega,t,t-1}\left(x_{t}^{\orig},c^{\orig}\right)\label{eq:opt_free_inversion_1}\\
x_{t-1}^{\edit} & =\tilde{\mu}_{\theta,\omega,t,t-1}\left(x_{t}^{\edit},c^{\edit}\right)+u_{t}^{\orig}\label{eq:opt_free_inversion_2}
\end{align}
Here, $u_{t}^{\orig}$ serves as a residual term that ensures non-edited
features from $x_{t-1}^{\orig}$ are retained in the edited version
$x_{t-1}^{\edit}$. For EF, the set $\left\{ x_{t}^{\orig}\right\} _{t=1}^{T}$
is constructed by sampling $x_{t}^{\orig}$ from $p\left(x_{t}|x_{0}^{\orig}\right)$
for each $t$ in parallel. Interestingly, this set can also be built
sequentially through DDIM inversion as per Eq.~\ref{eq:DDIM_inversion_SD}
(with $c^{\orig}$ replacing $c$).

\subsection{Diffusion Bridges and Doob's \emph{h}-transform}

Although various definitions of bridges exist in the literature \cite{de2021diffusion,liu2022let,liu2023i2sb,li2023bbdm,tong2024simulation,Kieu2025},
we adopt the perspective of \cite{liu2023learning,zhou2024denoising,Kieu2025}
and regard bridges as special stochastic processes that converge to
a \emph{predefined} sample $\hat{x}_{T}$ at time $T$ almost surely.
A bridge can be derived from a \emph{base} (or \emph{reference}) Markov
process through Doob's $h$-transform \cite{doob1984classical,rogers2000diffusions,sarkka2019applied}.
If the base process is a diffusion process described by the SDE $dx_{t}=f\left(x_{t},t\right)dt+g\left(t\right)dw_{t}$,
the corresponding bridge is governed by the following SDE:
\begin{equation}
dx_{t}=\left(f\left(x_{t},t\right)+{g\left(t\right)}^{2}\nabla\log h\left(x_{t},t\right)\right)dt+g\left(t\right)dw_{t}\label{eq:prelim_doob_h_SDE_form}
\end{equation}
where $h\left(x_{t},t\right)=p\left(\hat{x}_{T}|x_{t}\right)$. When
$f\left(x_{t},t\right)$ is a linear function of $x_{t}$, $h\left(x_{t},t\right)$
simplifies into a Gaussian distribution that can be expressed in closed
form \cite{zhou2024denoising}.

%% file: method.tex
\subsection{Editing as Reverse-time Bridge Modeling\label{subsec:Image-Editing-as-Reverse-time-Bridge-Modeling}}

In this section, we introduce a \emph{novel} theoretical framework
for image editing with diffusion models by framing it as a \emph{reverse-time
bridge modeling} problem. This idea stems from our insight that we
can generate images $x_{0}$ exhibiting the target properties $\mathcal{Y}$
(e.g., style, shape, color, object type, ...) by constructing a bridge
from the \emph{backward} process that converges to an \emph{implicit}
distribution associated with $\mathcal{Y}$. Our framework stands
apart from most existing bridge models~\cite{liu2023learning,somnath2023aligned,zhou2024denoising}
which focus solely on the (non-parameterized) \emph{forward} process
and assume an \emph{explicit} target sample $\hat{x}_{0}$ (or set
of samples $\left\{ \hat{x}_{0}\right\} $).

To construct this bridge, we modify the transition distribution $p_{\theta}\left(x_{t-1}|x_{t}\right)$
of the backward process using Doob's $h$-transform \cite{doob1984classical,sarkka2019applied}
as follows:
\begin{equation}
p_{\theta}^{h}\left(x_{t-1}|x_{t}\right)=p_{\theta}\left(x_{t-1}|x_{t}\right)\frac{h\left(x_{t-1},t-1\right)}{h\left(x_{t},t\right)}\label{eq:doob_transform_bwd}
\end{equation}
Here, $h\left(x_{t},t\right)$ is a positive real-valued function
that satisfies the following conditions for all $t\in\left[1,T\right]$:
\begin{align}
h\left(x_{t},t\right) & =\int p_{\theta}\left(x_{t-1}|x_{t}\right)h\left(x_{t-1},t-1\right)dx_{t-1}\label{eq:h_cond_bwd_1}\\
h\left(x_{0},0\right) & =p_{\mathcal{Y}}\left(x_{0}\right)\label{eq:h_cond_bwd_2}
\end{align}
where $p_{\mathcal{Y}}\left(x_{0}\right)$ is a \emph{predefined}
distribution quantifying how likely $x_{0}$ possesses the attributes
$\mathcal{Y}$. $p_{\mathcal{Y}}\left(x_{0}\right)=0$ if $x_{0}$
does not have the attributes $\mathcal{Y}$ and $>$ 0 otherwise.
For clarity in the subsequent discussion, we will omit the parameter
$\theta$ in $p_{\theta}\left(x_{t-1}|x_{t}\right)$ and $p_{\theta}^{h}\left(x_{t-1}|x_{t}\right)$,
referring to them simply as $p\left(x_{t-1}|x_{t}\right)$ and $p^{h}\left(x_{t-1}|x_{t}\right)$.

\begin{figure*}
\begin{centering}
\includegraphics[width=0.9\textwidth]{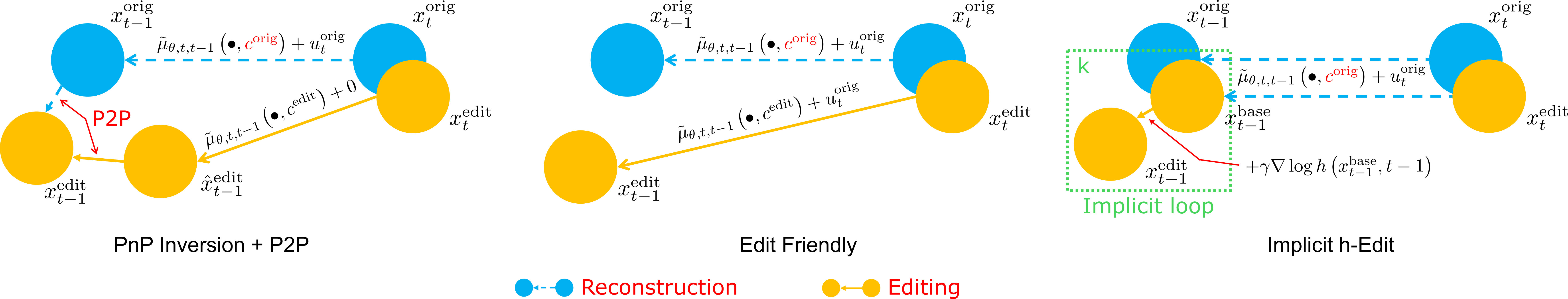}
\par\end{centering}
\caption{Overview of implicit $h$-Edit in comparison with PnP Inversion +
P2P \cite{ju2023direct} and Edit Friendly \cite{huberman2024edit}.\label{fig:model-overview}}
\end{figure*}

It can be shown that $h\left(x_{t},t\right)=\Expect_{p\left(x_{0}|x_{t}\right)}\left[h\left(x_{0},0\right)\right]$
(Appdx.~\ref{subsec:Derivation-of-the-h-func}) and the bridge constructed
in this manner forms a reverse-time Markov process with the transition
distribution $p^{h}\left(x_{t-1}|x_{t}\right)$. At time 0, this process
converges to a distribution formally stated in Proposition~\ref{prop:Doob_h_bridge_bwd}
below:
\begin{prop}
Consider a reverse-time Markov process with the transition distribution
$p\left(x_{t-1}|x_{t}\right)$ and a positive real-value function
$h\left(x_{t},t\right)$ satisfying Eqs.~\ref{eq:h_cond_bwd_1},
\ref{eq:h_cond_bwd_2} for all $t\in\left[1,T\right]$. If we construct
a bridge from this Markov process such that its transition distribution
$p^{h}\left(x_{t-1}|x_{t}\right)$ is defined as in Eq.~\ref{eq:doob_transform_bwd},
then the bridge is also a reverse-time Markov process. Moreover, if
the distribution at time $T$ of the bridge, $p^{h}\left(x_{T}\right)$,
is set to $\frac{p\left(x_{T}\right)h\left(x_{T},T\right)}{\Expect_{p\left(x_{0}\right)}\left[h\left(x_{0},0\right)\right]}$,
then $p^{h}\left(x_{t}\right)=\frac{p\left(x_{t}\right)h\left(x_{t},t\right)}{\Expect_{p\left(x_{0}\right)}\left[h\left(x_{0},0\right)\right]}$
for all $t\in\left[0,T\right]$.\label{prop:Doob_h_bridge_bwd}
\end{prop}
\begin{proof}
The detailed proof is provided in Appdx.~\ref{subsec:Proof-of-Proposition-1}.
\end{proof}
\begin{cor}
$p^{h}\left(x_{0}\right)$ is proportional to $p\left(x_{0}\right)p_{\mathcal{Y}}\left(x_{0}\right)$.\label{corr:target_Doob_h_bwd}
\end{cor}
Corollary~\ref{corr:target_Doob_h_bwd} implies that generated samples
from the bridge not only possess the attributes $\mathcal{Y}$ but
also look real. The realism associated with $p\left(x_{0}\right)$
comes from the base process used to construct the bridge. It can be
suppressed if $h\left(x_{0},0\right)$ is set to $p_{\mathcal{Y}}\left(x_{0}\right)/p\left(x_{0}\right)$,
resulting in $p^{h}\left(x_{0}\right)\propto p_{\mathcal{Y}}\left(x_{0}\right)$.
More generally, we can specify \emph{any} target distribution for
the bridge to converge to by appropriately selecting $h\left(x_{0},0\right)$.
This highlights the generalizability of our framework for editing.

A notable special case of our framework is when $h\left(x_{0},0\right)=p\left(y|x_{0}\right)$
with $y$ being a known attribute (e.g., a class label \cite{Dhariwal2021Diffusion}
or a text prompt \cite{rombach2022high}). In this case, $h\left(x_{t},t\right)=\Expect_{p\left(x_{0}|x_{t}\right)}\left[p\left(y|x_{0}\right)\right]=p\left(y|x_{t}\right)$.
Below, we discuss the continuous-time formulation of the bridge for
the sake of completeness.
\begin{prop}
If the base Markov process is characterized by the reverse-time SDE
$dx_{t}=\left(f\left(x_{t},t\right)-{g\left(t\right)}^{2}\nabla\log p_{t}\left(x_{t}\right)\right)dt+g\left(t\right)d\overline{w}_{t}$
\cite{anderson1982reverse,song2021score}, then the bridge constructed
from it via Doob's $h$-transform has the formula:
\begin{align}
dx_{t}=\  & \left(f\left(x_{t},t\right)-{g\left(t\right)}^{2}\left(\nabla\log p\left(x_{t}\right)+\nabla\log h\left(x_{t},t\right)\right)\right)dt\nonumber \\
 & +g\left(t\right)d\overline{w}_{t}\label{eq:Doob_h_bridge_SDE}
\end{align}
\end{prop}

\subsection{$h$-Edit}

After constructing the bridge, image editing can be carried out through
ancestral sampling from time $T$ to time $0$ along the bridge. However,
for a general function $h$, $p^{h}\left(x_{t-1}|x_{t}\right)$ is
typically non-Gaussian, making direct Monte Carlo sampling from this
distribution impractical. Therefore, we must rely on Markov Chain
Monte Carlo (MCMC) methods, such as Langevin Monte Carlo (LMC) \cite{roberts1996exponential,welling2011bayesian},
for sampling. LMC is particularly well-suited for diffusion models
due to the availability of score functions at every time $t$.

To sample from the (unnormalized) target distribution $p^{h}\left(x_{0}\right)\propto p\left(x_{0}\right)h\left(x_{0},0\right)$,
we perform a sequence of LCM updates, with each update defined as
follows:
\begin{align}
x_{t-1}\approx\  & x_{t}+\eta\nabla_{x_{t}}\log\left(p\left(x_{t}\right)h\left(x_{t},t\right)\right)+\sqrt{2\eta}z\\
=\  & \left(x_{t}+\eta\nabla_{x_{t}}\log p\left(x_{t}\right)+\sqrt{2\eta}z\right)\nonumber \\
 & +\eta\nabla_{x_{t}}\log h\left(x_{t},t\right)\label{eq:LMC_2}\\
=\  & \underbrace{x_{t-1}^{\text{\ensuremath{\base}}}}_{\text{rec.}}+\eta\underbrace{\nabla_{x_{t}}\log h\left(x_{t},t\right)}_{\text{editing}}\label{eq:LMC_3}
\end{align}
where $z\sim\mathcal{N}\left(0,\mathrm{I}\right)$, $\eta>0$ is the
step size, $x_{t}$ and $x_{t-1}$ denote \emph{edited} samples at
time $t$ and $t-1$, respectively. A similar expression to Eq.~\ref{eq:LMC_3}
can be derived by solving the bridge SDE in Eq.~\ref{eq:Doob_h_bridge_SDE}
using the Euler-Maruyama method \cite{kloeden1992numerical}. Intuitively,
$x_{t-1}$ and $x_{t-1}^{\text{\ensuremath{\base}}}$ can be regarded
as samples from $p^{h}\left(x_{t-1}|x_{t}\right)$ and $p\left(x_{t-1}|x_{t}\right)$,
respectively. According to the formula of $p^{h}\left(x_{t-1}|x_{t}\right)$
in Eq.~\ref{eq:doob_transform_bwd}, we can also sample $x_{t-1}$
as follows:
\begin{align}
x_{t-1}\approx\  & x_{t-1}^{\text{init}}+\gamma\nabla_{x_{t-1}}\log p^{h}\left(x_{t-1}|x_{t}\right)+\sqrt{2\gamma}z\\
=\  & \left(x_{t-1}^{\text{init}}+\gamma\nabla_{x_{t-1}}\log p\left(x_{t-1}|x_{t}\right)+\sqrt{2\gamma}z\right)\nonumber \\
 & +\gamma\nabla_{x_{t-1}}\log h\left(x_{t-1},t-1\right)\label{eq:implicit_LMC_2}\\
\approx\  & \underbrace{x_{t-1}^{\base}}_{\text{\text{rec.}}}+\gamma\underbrace{\nabla_{x_{t-1}}\log h\left(x_{t-1}^{\base},t-1\right)}_{\text{editing}}\label{eq:implicit_LMC_3}
\end{align}
Here, $\gamma$ > 0 is the step size. The gradient $\nabla_{x_{t-1}}\log p^{h}\left(x_{t-1}|x_{t}\right)$
does not involve $h\left(x_{t},t\right)$ because it is constant w.r.t.
$x_{t-1}$. Both updates in Eqs.~\ref{eq:LMC_3}, \ref{eq:implicit_LMC_3}
inherently fulfill two key image editing objectives - faithfulness
and effectiveness - through their decomposition into a \emph{``reconstruction''}
term $x_{t-1}^{\base}$ and an \emph{``editing''} term $\nabla_{x_{t}}\log h\left(x_{t},t\right)$
or $\nabla_{x_{t-1}}\log h\left(x_{t-1}^{\base},t-1\right)$, with
$\eta$ or $\gamma$ serving as the trade-off coefficient. Eq.~\ref{eq:LMC_3}
is explicit while Eq.~\ref{eq:implicit_LMC_3} is implicit. Furthermore,
we can view Eq.~\ref{eq:implicit_LMC_3} as a general optimization
problem:
\begin{equation}
x_{t-1}=\argmax{x'_{t-1}}\gamma\log h\left(x'_{t-1},t-1\right)\label{eq:LMC_optim}
\end{equation}
with $x_{t-1}^{\base}$ being the initial value, and perform multiple
gradient ascent updates to improve the editing quality:
\begin{align}
x_{t-1}^{(0)} & =x_{t-1}^{\base}\label{eq:LMC_optim_multi_1}\\
x_{t-1}^{(k+1)} & =x_{t-1}^{(k)}+\gamma\nabla_{x_{t-1}}\log h\left(x_{t-1}^{(k)},t-1\right)\label{eq:LMC_optim_multi_2}
\end{align}
Eq.~\ref{eq:LMC_optim_multi_2} is indeed the $k$-th iterations
of the implicit update formula in Eq.~\ref{eq:implicit_LMC_3}.

We refer to our proposed editing method as \textbf{$h$-Edit} with
Eqs.~\ref{eq:LMC_3} and \ref{eq:implicit_LMC_3} representing the
\emph{explicit} and \emph{implicit} versions of $h$-Edit, respectively.
$h$-Edit is highly \emph{flexible} as it can incorporate \emph{arbitrary}
$\log h$-functions, provided their gradients w.r.t. noisy samples
can be efficiently computed.

For text-guided editing with Stable Diffusion \cite{rombach2022high},
an \emph{explicit} $h$-Edit update is given by:
\begin{align}
x_{t-1}^{\base} & =\tilde{\mu}_{\theta,\omega,t,t-1}\left(x_{t}^{\edit},{\color{red}c^{\orig}}\right)+u_{t}^{\orig}\label{eq:hedit_SD_1}\\
x_{t-1}^{\edit} & =x_{t-1}^{\base}+\left(\sqrt{\sigma_{t-1}^{2}-\omega_{t,t-1}^{2}}-\frac{\sigma_{t}a_{t-1}}{a_{t}}\right)f\left(x_{t}^{\edit},t\right)\label{eq:hedit_SD_2}
\end{align}
where $\tilde{\mu}_{\theta,\omega,t,t-1}\left(\cdot,\cdot\right)$
and $u_{t}^{\orig}$ are defined in Eq.~\ref{eq:DDIM_sampling_SD}
and Eq.~\ref{eq:opt_free_inversion_1}, respectively. $f\left(x_{t},t\right)$
is expressed as follows:
\begin{align}
f\left(x_{t},t\right)=\  & w^{\edit}\epsilon_{\theta}\left(x_{t},t,c^{\edit}\right)-\hat{w}^{\orig}\epsilon_{\theta}\left(x_{t},t,c^{\orig}\right)\nonumber \\
 & +\left(\hat{w}^{\orig}-w^{\edit}\right)\epsilon_{\theta}\left(x_{t},t,\varnothing\right)\label{eq:hedit_SD_3}
\end{align}
Here, $w^{\edit}$, $\hat{w}^{\orig}$ are guidance weights. $\hat{w}^{\orig}$
may differ from $w^{\orig}$ used during inversion. An \emph{one-step
implicit} $h$-Edit update can be derived from Eq.~\ref{eq:hedit_SD_2}
by replacing $f\left(x_{t}^{\edit},t\right)$ with $f\left(x_{t-1}^{\base},t-1\right)$,
which gives:
\begin{equation}
x_{t-1}^{\edit}=x_{t-1}^{\base}+\left(\sqrt{\sigma_{t-1}^{2}-\omega_{t,t-1}^{2}}-\frac{\sigma_{t}a_{t-1}}{a_{t}}\right)f\left(x_{t-1}^{\base},t-1\right)\label{eq:hedit_SD_4}
\end{equation}
A detailed derivation of Eqs.~\ref{eq:hedit_SD_1}-\ref{eq:hedit_SD_4}
is provided in Appdx.~\ref{subsec:Closed-form-expressions-for-explicit-and-implicit}.
An overview of our method in comparison with Edit Friendly \cite{huberman2024edit}
and PnP Inversion \cite{ju2023direct} is shown in Fig.~\ref{fig:model-overview}.

Next, we will delve into the design of $h$ and its score. We will
focus on the implicit form and write $\nabla\log h\left(x_{t-1},t-1\right)$
instead of $\nabla_{x_{t-1}}\log h\left(x_{t-1},t-1\right)$ for simplicity.

\subsection{Designing $h$-Functions\label{subsec:Designing-h-Functions}}

\subsubsection{$h$-functions for conditional diffusion models}

In most conditional diffusion models, $h\left(x_{t-1},t-1\right)=p\left(y|x_{t-1}\right)$
where $y$ is a \emph{predefined} condition. This means:
\begin{align}
 & \nabla\log h\left(x_{t-1},t-1\right)\nonumber \\
 & \ \ \ =\nabla\log p\left(y|x_{t-1}\right)\label{eq:h_tm1_score_1}\\
 & \ \ \ =\nabla\log p\left(x_{t-1}|y\right)-\nabla\log p\left(x_{t-1}\right)\label{eq:h_tm1_score_2}
\end{align}
Eqs.~\ref{eq:h_tm1_score_1} and \ref{eq:h_tm1_score_2} correspond
to the classifier-based guidance and classifier-free guidance cases,
respectively. For text-guided editing with SD, $\nabla\log p\left(x_{t-1}|y\right)$
and $\nabla\log p\left(x_{t-1}\right)$ are modeled as $\frac{-\tilde{\epsilon}_{\theta}\left(x_{t-1},t-1,c^{\edit}\right)}{\sigma_{t-1}}$
and $\frac{-\tilde{\epsilon}_{\theta}\left(x_{t-1},t-1,c^{\orig}\right)}{\sigma_{t-1}}$,
respectively.

\subsubsection{External reward models $h\left(x_{0},0\right)$\label{subsec:h-function-External-reward-models}}

In many practical editing scenarios, only external reward models on
clean data $h\left(x_{0},0\right)$ are available. This means $h\left(x_{t},t\right)$
cannot take $x_{t}$ as the direct input but must be computed through
$h\left(x_{0},0\right)$ as $\Expect_{p\left(x_{0}|x_{t}\right)}\left[h\left(x_{0},0\right)\right]$.
Since directly sampling from $p\left(x_{0}|x_{t}\right)$ is difficult,
existing works \cite{chung2023diffusion,yu2023freedom,bansal2024universal}
usually approximate $h\left(x_{t},t\right)=\Expect_{p\left(x_{0}|x_{t}\right)}\left[h\left(x_{0},0\right)\right]$
by $h\left(x_{0|t},0\right)$ where $x_{0|t}:=\Expect_{p\left(x_{0}|x_{t}\right)}\left[x_{0}\right]$
denotes the posterior estimation of $x_{0}$ given $x_{t}$. In SD,
$x_{0|t}$ can be derived from $x_{t}$ and $\tilde{\epsilon}_{\theta}\left(x_{t},t,c^{\orig}\right)$
as $\frac{x_{t}-\sigma_{t}\tilde{\epsilon}_{\theta}\left(x_{t},t,c^{\orig}\right)}{a_{t}}$
based on Tweedie's formula \cite{efron2011tweedie}.

\subsubsection{$h$-functions for reconstruction}

In addition to using $h$ as an editing function, we can design an
$h$-function specifically for reconstruction, defined as:

\begin{equation}
h_{\text{rec}}\left(x_{t-1},t-1\right):=\exp\left(-\lambda_{t-1}\left\Vert x_{t-1}-x_{t-1}^{\base}\right\Vert _{2}^{2}\right)\label{eq:h_rec}
\end{equation}
When this $h$-function is integrated into our optimization framework
in Eq.~\ref{eq:LMC_optim}, it enables simultaenous optimization-free
and optimization-based reconstruction (via $x_{t-1}^{\base}$ and
$\nabla\log h_{\text{rec}}\left(x_{t-1},t-1\right)$, respectively),
exclusive to $h$-Edit.

\subsubsection{Product of $h$-Experts}

Since $\log h$ can be interpreted as a \emph{negative energy function},
we can combine multiple $h$-functions to create a ``product of $h$-experts''
as follows:
\begin{equation}
h=h_{1}*h_{2}*...*h_{m}\label{eq:prod_h_experts}
\end{equation}
where $m$ denotes the number of $h$-functions. The combined $h$-function
in Eq.~\ref{eq:prod_h_experts} can be easily integrated into our
framework by summing the score for each component:
\begin{equation}
\nabla\log h\left(x_{t-1},t-1\right)=\sum_{i=1}^{m}\nabla\log h_{i}\left(x_{t-1},t-1\right)\label{eq:score_prod_h_experts}
\end{equation}

%% file: related.tex
\begin{table*}
\begin{centering}
{\resizebox{0.85\textwidth}{!}{%
\par\end{centering}
\begin{centering}
\begin{tabular}{cccc>{\centering}p{0.12\textwidth}cccc}
\hline 
\multirow{1}{*}{\textbf{Inv.}} & \multirow{1}{*}{\textbf{Attn.}} & \multirow{1}{*}{\textbf{Method}} & \textbf{\textcolor{black}{CLIP Sim.}}$\uparrow$ & \multirow{1}{0.12\textwidth}{\textbf{\textcolor{black}{Local CLIP}}\textcolor{black}{$\uparrow$}} & \multirow{1}{*}{\textbf{\textcolor{black}{DINO Dist.}}\textcolor{black}{$_{\times10^{2}}$$\downarrow$}} & \multirow{1}{*}{\textbf{LPIPS}$_{\times10^{2}}$$\downarrow$} & \multirow{1}{*}{\textbf{SSIM}$_{\times10}$$\uparrow$} & \multirow{1}{*}{\textbf{PSNR}$\uparrow$}\tabularnewline
\hline 
\hline 
\multirow{6}{*}{Deter.} & \multirow{6}{*}{P2P} & NP & \textcolor{black}{0.246} & \uline{0.140} & \textcolor{black}{1.62} & 6.90 & 8.34 & 26.21\tabularnewline
 &  & NT & \textcolor{black}{0.248} & 0.130 & \textcolor{black}{1.34} & 6.07 & 8.41 & 27.03\tabularnewline
 &  & StyleD & \textcolor{black}{0.248} & 0.085 & \textbf{\textcolor{black}{1.17}} & 6.61 & 8.34 & 26.05\tabularnewline
 &  & NMG & \textcolor{black}{0.249} & 0.087 & \textcolor{black}{\uline{1.32}} & 5.59 & 8.47 & 27.05\tabularnewline
 &  & PnP Inv & \textcolor{black}{\uline{0.250}} & 0.095 & \textbf{\textcolor{black}{1.17}} & \uline{5.46} & \uline{8.48} & \uline{27.22}\tabularnewline
 &  & $h$-Edit-D & \textbf{\textcolor{black}{0.253}} & \textbf{0.147} & \textbf{1.17} & \textbf{4.85} & \textbf{8.54} & \textbf{27.87}\tabularnewline
\hline 
\hline 
\multirow{5}{*}{Random} & \multirow{3}{*}{None} & EF & \textcolor{black}{\uline{0.254}} & \uline{0.122} & \textcolor{black}{\uline{1.29}} & \uline{6.09} & \uline{8.37} & \uline{25.87}\tabularnewline
 &  & LEDITS++ & \textcolor{black}{0.254} & 0.113 & \textcolor{black}{2.34} & 8.88 & 8.11 & 23.36\tabularnewline
 &  & $h$-Edit-R & \textbf{0.255} & \textbf{0.148} & \textbf{1.28} & \textbf{5.55} & \textbf{8.46} & \textbf{26.43}\tabularnewline
\cline{2-9} \cline{3-9} \cline{4-9} \cline{5-9} \cline{6-9} \cline{7-9} \cline{8-9} \cline{9-9} 
 & \multirow{2}{*}{P2P} & EF & \textcolor{black}{0.255} & 0.126 & \textcolor{black}{1.51} & 5.70 & 8.40 & 26.30\tabularnewline
 &  & $h$-Edit-R & \textbf{0.256} & \textbf{0.159} & \textbf{1.45} & \textbf{5.08} & \textbf{8.50} & \textbf{26.97}\tabularnewline
\hline 
\end{tabular}{\small{}}}}{\small\par}
\par\end{centering}
\caption{Text-guided image editing results of $h$-Edit and other baselines.
The best and second best results for each metric and inversion type
are highlighted in bold and underscored, respectively. \label{tab:Text-Guided-Editing-Results}}
\end{table*}

Due to space constraints, this section only covers related work in
training-free editing. For details on conditional generation and diffusion
bridges, please refer to Appdx.~\ref{sec:Additional-Related-Work}.

The advent of conditional diffusion models, particularly text-guided
latent diffusion models like Stable Diffusion \cite{rombach2022high},
has greatly advanced the development of various diffusion-based text-guided
image editing techniques. These methods can be broadly categorized
into training-based \cite{kim2022diffusionclip,kawar2023imagic,kwon2023diffusion,zhang2023sine}
and training-free methods \cite{meng2022sdedit,wu2023uncovering,li2023stylediffusion,mokady2023null,xu2024inversion}.
Unlike training-based methods, which finetune the noise network \cite{kim2022diffusionclip}
or employ an auxiliary model \cite{kwon2023diffusion} through additional
training, training-free methods modify the attention or feature maps
in Stable Diffusion (SD) \cite{hertz2023prompttoprompt,tumanyan2023plug,cao2023masactrl,parmar2023zero}
or adjust the generation process of SD \cite{mokady2023null} to ensure
editing fidelity. Null-text inversion (NTI) \cite{mokady2023null}
optimizes the null-text embedding during generation to minimize discrepancies
between this process and the forward process. Prompt Tuning inversion
(PTI) \cite{dong2023prompt} interpolates between the target text
embedding and the null-text embedding optimized by NTI to create a
more suitable embedding for editing. EDICT \cite{wallace2023edict}
draws inspiration from affine coupling layers in normalizing flows
to design a more faithful reconstruction process compared to DDIM
sampling. Negative Prompt inversion (NPI) \cite{miyake2023negative}
bypasses the costly optimization of NTI by using the original text
embedding instead of the null-text embedding, while ProxNPI \cite{han2024proxedit}
adds an auxiliary regularization term to enhance NPI's reconstruction
capabilities. Noise Map Guidance (NMG) \cite{cho2024noise} leverages
energy-based guidance \cite{zhao2022egsde} and information from the
inversion process to denoise samples in a way that improve reconstruction.
PnP Inversion \cite{ju2023direct} avoids optimization by incorporating
the difference between inversion and reconstruction samples directly
into the editing update. AIDI \cite{pan2023effective} views exact
reconstruction as a fixed-point iteration problem and use Anderson
acceleration to find the solution. Unlike these deterministic-inversion-based
methods, Edit Friendly (EF) \cite{huberman2024edit} employs random
inversion with independent sampling of intermediate noisy samples,
achieving good reconstruction without the need for attention map adjustments
like P2P. LEDITS++ \cite{brack2024ledits} introduces several enhancements
to EF, improving both efficiency and versatility in editing. Generally,
most training-free methods are limited to text-guided editing, while
our approach allows for the seamless combination of multiple editing
types due to the clear separation of the reconstruction and editing
terms.

%% file: exp.tex
Due to space limit, we only provide main results in this section and
refer readers to Appdx.~\ref{sec:Additional-Ablation-Studies} for
our ablation studies on $w^{\edit}$, $\hat{w}^{\orig}$, the number
of optimization steps, as well as other additional results. 

\subsection{Text-guided Editing\label{subsec:Text-guided-Image-Editing}}

\subsubsection{Experiment Setup}

\begin{figure*}[t]
\centering{}%
\begin{minipage}[b][1\totalheight][t]{0.5\textwidth}%
\begin{center}
\includegraphics[width=1\textwidth]{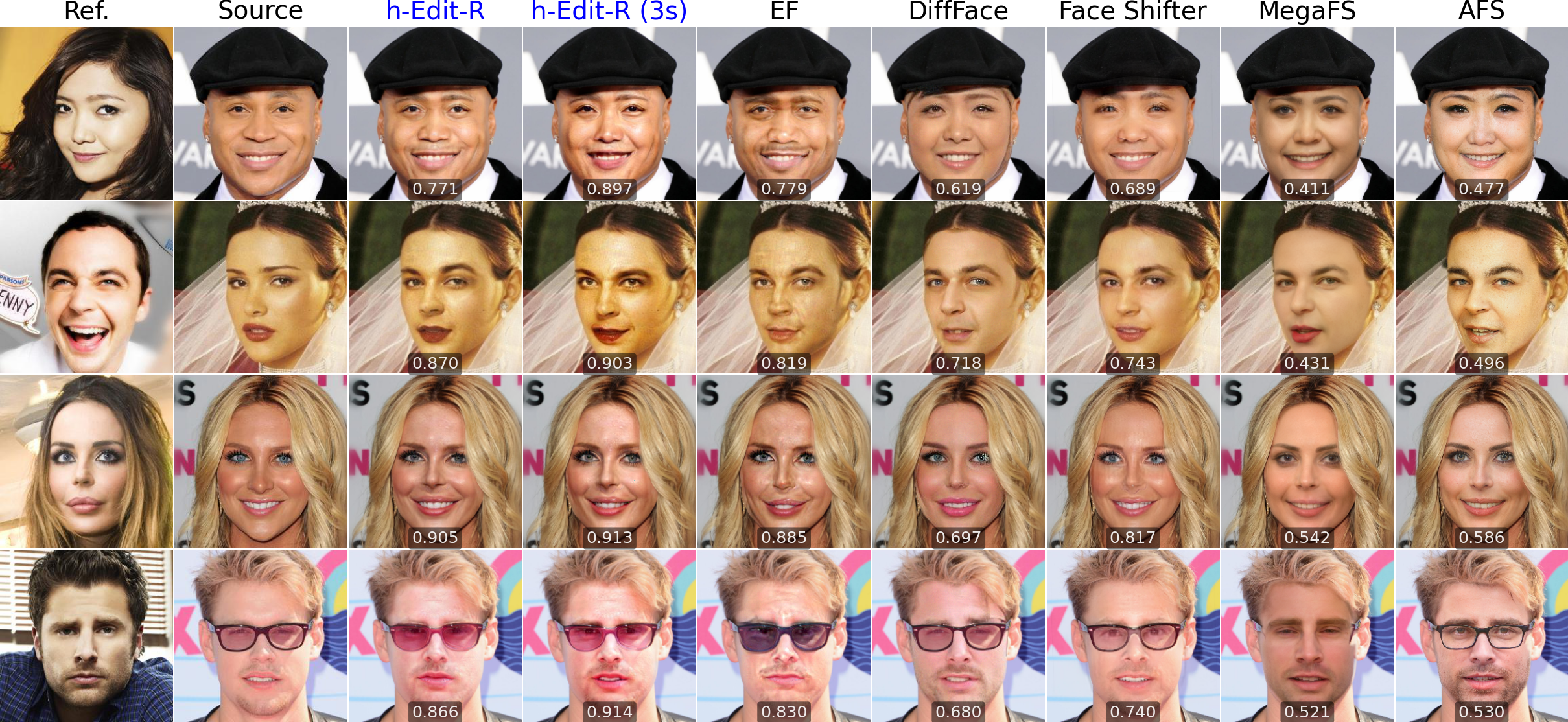}
\par\end{center}%
\end{minipage}\hspace*{0.03\textwidth}%
\begin{minipage}[b][1\totalheight][t]{0.45\textwidth}%
\begin{center}
{\resizebox{\textwidth}{!}{
\begin{tabular}{cccccc}
\toprule 
\textbf{Method} & \textbf{ID}$\uparrow$ & \textbf{Expr}.$\downarrow$ & \textbf{Pose}$\downarrow$ & \textbf{LPIPS}$\downarrow$ & \textbf{FID}$\downarrow$\tabularnewline
\midrule
\midrule 
FaceShifter & 0.70 & \textbf{2.39} & \textbf{2.81} & 0.08 & \textbf{10.16}\tabularnewline
MegaFS & 0.34 & 2.88$^{\dagger}$ & 7.71 & 0.15 & 27.07\tabularnewline
AFS & 0.47 & 2.92 & 4.68 & 0.13 & 17.55\tabularnewline
DiffFace & 0.61 & 3.04 & 4.35 & 0.10 & \uline{11.89}\tabularnewline
EF & 0.74 & 3.10 & 4.12 & 0.06 & 20.78\tabularnewline
\midrule
\midrule 
$h$-edit-R & \uline{0.80} & \uline{2.76} & \uline{3.78} & \textbf{0.04} & 17.68\tabularnewline
$h$-edit-R (3s) & \textbf{0.84} & 3.10 & 4.29 & \uline{0.05} & 19.12\tabularnewline
\bottomrule
\end{tabular}{\small{}}}}
\par\end{center}%
\end{minipage}\caption{\textbf{Left}: Visualization of swapped faces produced by implicit
$h$-Edit-R and baselines. (3s) denotes $\protect\Model$-R with 3
optimization steps. Identity similarity scores (higher is better)
are displayed below each output. \textbf{Right}: Face swapping results
of implicit $h$-Edit-R and other baselines. $\dagger$: The expression
error for MegaFS was calculated on images with detectable faces, as
required by the evaluation metric.\label{fig:Qualitative-examples_Quantitative-results_face_swapping}}
\end{figure*}

We evaluate our method on text-guided image editing using the PIE-Bench
dataset~\cite{ju2023direct}, which includes 700 diverse images of
humans, animals, and objects across various environments. Each image
comes with an original and edited text descriptions and an annotated
mask indicating the editing region. PIE-Bench features 10 distinct
editing categories, including adding, removing, or modifying objects,
styles, and backgrounds.

For evaluation, we follow~\cite{ju2023direct} to use CLIP similarity~\cite{radford2021learning}
between the edited image and text to measure editing effectiveness.
To assess editing faithfulness, we compute PSNR, LPIPS~\cite{zhang2018unreasonable},
and SSIM~\cite{wang2004image} on non-edited regions, as defined
by the editing masks, and DINO feature distance~\cite{tumanyan2022splicing}
on the entire image. Additionally, we include local directional CLIP
similarity \cite{kim2022diffusionclip} to enhance evaluation of editing
effectiveness, as standard CLIP similarity may be insufficient when
the edited attribute represents only a small part of the target text.
While these metrics offer insights, they are imperfect, as analyzed
in Appdx.~\ref{sec:Metrics}. Visual assessments remain essential
for evaluating editing quality.

We compare $h$-Edit with state-of-the-art diffusion-based text-guided
editing baselines that use either deterministic or random inversion,
including NT~\cite{mokady2023null}, NP~\cite{miyake2023negative},
StyleD~\cite{li2023stylediffusion}, NMG~\cite{cho2024noise}, PnP
Inv~\cite{ju2023direct}, EF~\cite{huberman2024edit}, and LEDITS++~\cite{brack2024ledits}.
We refer to $h$-Edit with deterministic inversion as $h$-Edit-D,
and with random inversion as $h$-Edit-R. For a fair comparison, we
adhere to the default settings in~\cite{ju2023direct,huberman2024edit},
using Stable Diffusion v1.4~\cite{rombach2022high} and 50 sampling
steps for editing. Following~\cite{ju2023direct}, we apply Prompt-to-Prompt
(P2P) \cite{hertz2023prompttoprompt} to all deterministic-inversion-based
methods to ensure faithful reconstruction. For random-inversion-based
methods, we report results both with and without P2P. Unless otherwise
specified, we use the implicit form with a single optimization step
(Eq.~\ref{eq:implicit_LMC_3}) for both $h$-Edit-D and $h$-Edit-R.
The hyperparameters $w^{\orig}$, $w^{\edit}$, $\hat{w}^{\orig}$
are set to $1.0$, $10.0$, $9.0$ for $h$-Edit-D, and $1.0$, $7.5$,
$5.0$ for $h$-Edit-R, respectively, as these values yield strong
quantitative and qualitative results. Detailed ablation studies on
these hyperparameters are provided in Appdx.~\ref{sec:Additional-Ablation-Studies}.

\subsubsection{Results\label{subsec:Results-Setting-1}}

As shown in Table~\ref{tab:Text-Guided-Editing-Results}, $h$-Edit-D
+ P2P significantly outperforms all deterministic-inversion-based
baselines with P2P in both editing effectiveness and faithfulness.
For example, our method improves over NT, a strong baseline, by 1.22$\times$$10^{-2}$
in LPIPS and 0.017 in local CLIP similarity. We observed that PnP
Inv and NMG often reconstruct the original image in challenging editing
scenarios, achieving high faithfulness despite not actually making
meaningful changes. In contrast, $h$-Edit-D + P2P consistently performs
successful edits while maintaining superior faithfulness. This validates
the theoretical soundness of $h$-Edit compared to other methods.

Similarly, $h$-Edit-R outperfoms both EF and LEDITS++ across all
metrics, with or without P2P. This improvement is largely due to the
implicit form and the carefully selected value of $\hat{w}^{\orig}$
- features unique to $h$-Edit. Additionally, we observed that LEDITS++
occasionally produces unfaithful or erroneous images, even after hyperparameter
tuning. Notably, random-inversion methods (including $h$-Edit-R)
without P2P often fall behind their P2P-enabled counterparts in changing
color and texture but excel in adding and removing objects, suggesting
that the choice to combine with P2P depends on the specific editing
scenario.

In Fig.~\ref{fig:qualitative-text-guided} and Appdx.~\ref{subsec:Additional_results_text_guided},
we provide a \emph{non-exhaustive} list of edited images by our method
and baselines, showcasing our superior performance.

\begin{figure*}[t]
\begin{centering}
\includegraphics[width=0.95\textwidth]{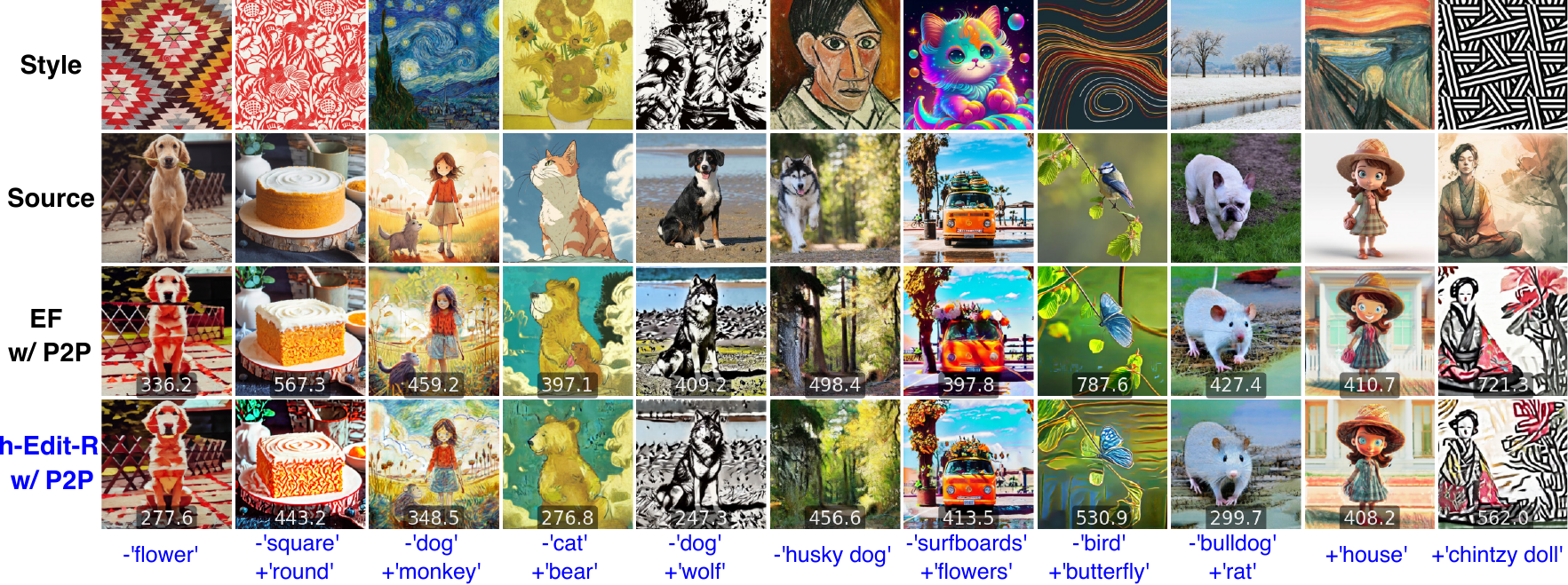}
\par\end{centering}
\caption{Qualitative comparison of $h$-Edit-R + P2P and EF + P2P in the combined
editing task. Style losses (lower is better) are shown below each
output image. h-Edit-R + P2P achieves superior results in both style
transfer and text-guided editing.\label{fig:Combined-Style-Editing-Qualititative}}
\end{figure*}

\subsection{Face Swapping\label{subsec:Face-Attribute-Editing}}

\subsubsection{Experimental Settings}

We consider face swapping as a benchmark to verify the capabilities
of $h$-Edit in reward-model-based editing. Given a diffusion model
trained on 256$\times$256 CelebA-HQ facial images~\cite{meng2022sdedit,karras2017progressive},
and a pretrained ArcFace model~\cite{deng2019arcface}, our goal
is to transfer the identity from a reference face $x_{0}^{\refer}$
to an original face $x_{0}^{\orig}$ while preserving other attributes
of $x_{0}^{\orig}$ such as hair style, pose, facial expression, and
background. For this experiment, we use 5,000 pairs $\left(x_{0}^{\orig},x_{0}^{\refer}\right)$
sampled randomly from CelebA-HQ.

We use implicit $h$-Edit-R with either 1 or 3 optimization steps.
Since P2P is inapplicable to unconditional diffusion models, our method
operates without P2P. The cosine similarity between the edited image
$x_{0}^{\edit}$ and $x_{0}^{\text{\ensuremath{\refer}}}$ is employed
as the reward, and the score $\nabla\log h\left(x_{t-1},t-1\right)$
is approximated based on the technique discussed in Section~\ref{subsec:h-function-External-reward-models}.
We compare $\Model$-R to well-known face-swapping methods, including
GAN-based (FaceShifter~\cite{li2020advancing}), Style-GAN-based
(MegaFS~\cite{zhu2021one} and AFS~\cite{vu2022face}), and diffusion-based
(DiffFace~\cite{kim2022diffface}). Unlike DiffFace which is a training
based method, our method is training-free. We also include EF as a
training-free baseline by adding the score to its editing term as
described in Algo.~\ref{subsec:algorithm_ef}. This extension of
EF has never been considered in the literature. We use 100 sampling
steps for all diffusion-based methods, including DiffFace. Facial
images generated by all methods are masked before evaluation, with
unmasked results provided in Appdx.~\ref{subsec:Face-swapping-without-masks}.
Following \cite{vu2022face,li2020advancing}, we assess editing effectiveness
via cosine similarity using ArcFace, faithfulness via expression/pose
error and LPIPS, and visual quality via FID~\cite{heusel2017gans}.

\subsubsection{Results}

As shown in Fig.~\ref{fig:Qualitative-examples_Quantitative-results_face_swapping}
(right), both versions of $h$-Edit-R achieve the highest face-swapping
accuracies. $h$-Edit-R also ranks second-best in preserving expressions
and poses, outperforming DiffFace and EF by large margins. However,
in terms of FID, our method falls short of FaceShifter and DiffFace,
likely because these methods are specifically tailored for face swapping
and trained on larger face datasets (FFHQ~\cite{karras2019style}
for DiffFace and FFHQ + CelebA-HQ for FaceShifter). Using three optimization
steps improves the identity transfer accuracy compared to using one
both quantitatively and qualitatively (Fig.~\ref{fig:Qualitative-examples_Quantitative-results_face_swapping}
(left)), showcasing the advantage of our implicit form. However, this
improvement may slightly reduce faithfulness, especially when the
source and reference faces differ significantly. Additional visualizations
are provided in Appdx.~\ref{subsec:Additional_results_face_swapping}.

\subsection{Combined Text-guided and Style Editing\label{subsec:Style-Editing-and-Combined-Editing}}

\subsubsection{Experimental Settings}

This task is similar to text-guided editing in Section~\ref{subsec:Text-guided-Image-Editing}
but with an additional requirement: the edited image $x_{0}^{\edit}$
should have similar style as a reference image $x_{0}^{\style}$.
Following \cite{yu2023freedom}, we use the negative L2 distance between
the Gram matrices~\cite{johnson2016perceptual} from the third feature
layer of the CLIP image encoder w.r.t. $x_{0}^{\edit}$ and $x_{0}^{\style}$
as a style reward. The norm of the style reward score is scaled to
match the norm of the editing function $f\left(\cdot\right)$ in Eq.~\ref{eq:hedit_SD_3}
at each time $t$, inspired by~\cite{yu2023freedom}. In this experiment,
each original image $x_{0}^{\orig}$ from the PIE-Bench dataset is
paired with a style image randomly selected from a set of 11 styles
shown in Fig.~\ref{fig:Combined-Style-Editing-Qualititative}. We
employ implicit $h$-Edit-R + P2P and compare it with EF + P2P. We
keep $\left(w^{\edit},\hat{w}^{\orig}\right)$ for our method and
$w^{\edit}$ for EF the same as in Section~\ref{subsec:Text-guided-Image-Editing},
tuning only the style editing coefficient $\rho^{\style}$. Given
the limitations of existing metrics in evaluating stylized edited
images, our choice of $\rho^{\style}$ is based primarily on visual
quality. We found that $\rho^{\style}$ equal 0.6 and 1.5 provide
the best results for our method and EF, respectively. Additional justification
for this selection is provided in Appdx.~\ref{subsec:Additional_results_style_text_combined}.
All other settings remain consistent with those used in the text-guided
editing experiment.

\subsubsection{Results}

It can be seen from Fig.~\ref{fig:Combined-Style-Editing-Qualititative}
and the visualizations in Appdx.~\ref{subsec:Additional_results_style_text_combined}
that $h$-Edit-R + P2P achieves more effective text-guided and style
edits while better preserving non-edited content compared to EF +
P2P. EF + P2P seems to struggle with combined editing task, sometimes
introducing artifacts (e.g., a baby bear in the fourth column in Fig.~\ref{fig:Combined-Style-Editing-Qualititative})
or altering non-edited content (e.g., a different girl in the third
column). Additionally, EF + P2P is more sensitive to the change of
$\rho^{\style}$ as slightly increasing $\rho^{\style}$ can improve
style editing but also exacerbate the unfaithfulness problem (Appdx.~\ref{subsec:Additional_results_style_text_combined}).

%% file: conclusion.tex
We introduced the reverse-time bridge modeling framework for effective
diffusion-based image editing, and proposed $h$-Edit - a novel training-free
editing method - as an instance of our framework. $h$-Edit leverages
Doob's $h$-transform and Langevin Monte Carlo to create an effective
editing update, composed of the ``reconstruction'' and ``editing''
terms, which capture the editing faithfulness and effectiveness, respectively.
This design grants our method great flexibility, allowing for seamless
integration of various $h$-functions to support different editing
objectives. Extensive experiments across diverse editing tasks demonstrated
that $h$-Edit achieves state-of-the-art editing performance, as evidenced
by quantitatively and qualitatively metrics. These results validate
both the theoretical soundness and practical strength of our method,
which we hope will inspire future research to address more complex
real-world editing challenges while maintaining theoretical guarantees. 

Despite these advantages, our method faces challenges in some difficult
editing cases. Although these issues could be partially mitigated
by using the implicit version with multiple optimization loops (Appdx.~\ref{subsec:Impact-of-multiple-loops})
or by manually increasing $w^{\edit}$ and $\hat{w}^{\orig}$ (Appdx.~\ref{subsec:Impact-of-hat_w_orig}),
an automated solution for handling them would be highly beneficial.
Another promising direction is to modify $x_{t-1}^{\base}$ to focus
on preserving only the non-edited regions, enhancing editing effectiveness.

%% file: appdx.tex
\section{Theoretical Results}

\subsection{Derivation of the formula of $h\left(x_{t},t\right)$\label{subsec:Derivation-of-the-h-func}}

Below, we prove that $h\left(x_{t},t\right)$ satisfying Eqs.~\ref{eq:h_cond_bwd_1},
\ref{eq:h_cond_bwd_2} can be expressed as follows:
\begin{align}
h\left(x_{t},t\right) & =\Expect_{p\left(x_{0}|x_{t}\right)}\left[h\left(x_{0},0\right)\right]\label{eq:h_analytical_form_1}\\
 & =\Expect_{p\left(x_{0}|x_{t}\right)}\left[p_{\mathcal{Y}}\left(x_{0}\right)\right]\label{eq:h_analytical_form_2}
\end{align}
where $p\left(x_{0}|x_{t}\right)$ is the transition distribution
of the base backward Markov process.

We can quickly verify that Eq.~\ref{eq:h_analytical_form_1} is correct
for $t=1$ since $h\left(x_{1},1\right)=\int p\left(x_{0}|x_{1}\right)h\left(x_{0},0\right)dx_{0}=\Expect_{p\left(x_{0}|x_{1}\right)}\left[h\left(x_{0},0\right)\right]$
directly from Eqs.~\ref{eq:h_cond_bwd_1}, \ref{eq:h_cond_bwd_2}.
Assuming that Eq.~\ref{eq:h_analytical_form_1} has been correct
for $t-1$ ($t\geq2$), we will prove that it is correct for $t$.
The RHS of Eq.~\ref{eq:h_cond_bwd_1} can be transformed as follows:
\begin{align}
h\left(x_{t},t\right) & =\int p\left(x_{t-1}|x_{t}\right)h\left(x_{t-1},t-1\right)dx_{t-1}\\
 & =\int p\left(x_{t-1}|x_{t}\right)\Expect_{p\left(x_{0}|x_{t-1}\right)}\left[h\left(x_{0},0\right)\right]dx_{t-1}\\
 & =\int p\left(x_{t-1}|x_{t}\right)\left(\int p\left(x_{0}|x_{t-1}\right)h\left(x_{0},0\right)dx_{0}\right)dx_{t-1}\\
 & =\int\left(\int p\left(x_{0}|x_{t-1}\right)p\left(x_{t-1}|x_{t}\right)dx_{t-1}\right)h\left(x_{0},0\right)dx_{0}\\
 & =\int p\left(x_{0}|x_{t}\right)p_{\mathcal{Y}}\left(x_{0}\right)dx_{0}\label{eq:h_formula_5}\\
 & =\Expect_{p\left(x_{0}|x_{t}\right)}\left[h\left(x_{0},0\right)\right]\label{eq:h_formula_6}
\end{align}
In Eq.~\ref{eq:h_formula_5}, $p\left(x_{0}|x_{t}\right)$ equals
$\int p\left(x_{0}|x_{t-1}\right)p\left(x_{t-1}|x_{t}\right)dx_{t+1}$
because this is the Chapman-Kolmogorov equation \cite{karush1961chapman,Kieu2025}
for the base backward process. Eq.~\ref{eq:h_formula_6} completes
our proof.

\subsection{Proof of Proposition~\ref{prop:Doob_h_bridge_bwd}\label{subsec:Proof-of-Proposition-1}}

First, it can be seen that $p^{h}\left(x_{t-1}|x_{t}\right)$ is well
normalized since according to Eqs.~\ref{eq:doob_transform_bwd},
\ref{eq:h_cond_bwd_1}, we have:
\begin{align}
\int p^{h}\left(x_{t-1}|x_{t}\right)dx_{t-1} & =\frac{\int p\left(x_{t-1}|x_{t}\right)h\left(x_{t-1},t-1\right)dx_{t-1}}{h\left(x_{t},t\right)}\\
 & =\frac{h\left(x_{t},t\right)}{h\left(x_{t},t\right)}\\
 & =1
\end{align}
Thus, $p^{h}\left(x_{t-1}|x_{t}\right)$ can be viewed as the transition
distribution of our bridge. Besides, since $x_{t-1}$ in $p^{h}\left(x_{t-1}|x_{t}\right)$
only depends on $x_{t}$, this bridge is a reverse-time Markov process.

Next, we prove that $p^{h}\left(x_{t}\right)=\frac{p\left(x_{t}\right)h\left(x_{t},t\right)}{\Expect_{p\left(x_{0}\right)}\left[h\left(x_{0},0\right)\right]}$
for all $t\in\left[0,T\right]$. This equation holds for $t=T$ due
to our assumption $p^{h}\left(x_{T}\right)=\frac{p\left(x_{T}\right)h\left(x_{T},T\right)}{\Expect_{p\left(x_{0}\right)}\left[h\left(x_{0},0\right)\right]}$.
Assuming that this equation holds for time $t$, we will prove that
it holds for time $t-1$. Since the bridge is a reverse-time Markov
process, we can compute $p^{h}\left(x_{t-1}\right)$ as follows:
\begin{align}
p^{h}\left(x_{t-1}\right) & =\int p^{h}\left(x_{t-1}|x_{t}\right)p^{h}\left(x_{t}\right)dx_{t}\label{eq:p_h_x_tm1_1}\\
 & =\int p\left(x_{t-1}|x_{t}\right)\frac{h\left(x_{t-1},t-1\right)}{\cancel{h\left(x_{t},t\right)}}\frac{p\left(x_{t}\right)\cancel{h\left(x_{t},t\right)}}{\Expect_{p\left(x_{0}\right)}\left[h\left(x_{0},0\right)\right]}dx_{t}\label{eq:p_h_x_tm1_2}\\
 & =\frac{h\left(x_{t-1},t-1\right)\int p\left(x_{t-1}|x_{t}\right)p\left(x_{t}\right)dx_{t}}{\Expect_{p\left(x_{0}\right)}\left[h\left(x_{0},0\right)\right]}\label{eq:p_h_x_tm1_3}\\
 & =\frac{p\left(x_{t-1}\right)h\left(x_{t-1},t-1\right)}{\Expect_{p\left(x_{0}\right)}\left[h\left(x_{0},0\right)\right]}\label{eq:p_h_x_tm1_4}
\end{align}
where Eq.~\ref{eq:p_h_x_tm1_2} leverages Eq.~\ref{eq:doob_transform_bwd}
and the inductive assumption. Eq.~\ref{eq:p_h_x_tm1_4} completes
our proof.

Finally, we prove that $p^{h}\left(x_{t}\right)$ is a well normalized
distribution as follows:
\begin{align}
\int p^{h}\left(x_{t}\right)dx_{t} & =\frac{\int p\left(x_{t}\right)h\left(x_{t},t\right)dx_{t}}{\Expect_{p\left(x_{0}\right)}\left[h\left(x_{0},0\right)\right]}\\
 & =\frac{\int p\left(x_{t}\right)\Expect_{p\left(x_{0}|x_{t}\right)}\left[h\left(x_{0},0\right)\right]dx_{t}}{\Expect_{p\left(x_{0}\right)}\left[h\left(x_{0},0\right)\right]}\label{eq:int_p_h_xt_2}\\
 & =\frac{\int p\left(x_{t}\right)\left(\int p\left(x_{0}|x_{t}\right)h\left(x_{0},0\right)dx_{0}\right)dx_{t}}{\Expect_{p\left(x_{0}\right)}\left[h\left(x_{0},0\right)\right]}\\
 & =\frac{\int\left(\int p\left(x_{t}\right)p\left(x_{0}|x_{t}\right)dx_{t}\right)h\left(x_{0},0\right)dx_{0}}{\Expect_{p\left(x_{0}\right)}\left[h\left(x_{0},0\right)\right]}\\
 & =\frac{\int p\left(x_{0}\right)h\left(x_{0},0\right)dx_{0}}{\Expect_{p\left(x_{0}\right)}\left[h\left(x_{0},0\right)\right]}\\
 & =1
\end{align}
The fact that $h\left(x_{t},t\right)=\Expect_{p\left(x_{0}|x_{t}\right)}\left[h\left(x_{0},0\right)\right]$
in Eq.~\ref{eq:int_p_h_xt_2} was proven in Section~\ref{subsec:Derivation-of-the-h-func}.

\subsection{Closed-form expressions for the explicit and implicit $h$-Edit updates
for Stable Diffusion\label{subsec:Closed-form-expressions-for-explicit-and-implicit}}

In this section, we derive closed-form expressions for the explicit
and implicit $h$-Edit updates corresponding to Eq.~\ref{eq:LMC_3}
and Eq.~\ref{eq:implicit_LMC_3}, respectively, for Stable Diffusion
(SD). First, we can express $\nabla_{x_{t}}\log h\left(x_{t},t\right)$
as follows:
\begin{align}
\nabla_{x_{t}}\log h\left(x_{t},t\right)=\  & \nabla_{x_{t}}\log p^{h}\left(x_{t}\right)-\nabla_{x_{t}}\log p\left(x_{t}\right)\\
=\  & \frac{-\tilde{\epsilon}_{\theta}\left(x_{t},t,c^{\edit}\right)}{\sigma_{t}}-\frac{-\tilde{\epsilon}_{\theta}\left(x_{t},t,c^{\orig}\right)}{\sigma_{t}}\\
=\  & \frac{-1}{\sigma_{t}}\left(\tilde{\epsilon}_{\theta}\left(x_{t},t,c^{\edit}\right)-\tilde{\epsilon}_{\theta}\left(x_{t},t,c^{\orig}\right)\right)\\
=\  & \frac{-1}{\sigma_{t}}\Big(w^{\edit}\epsilon_{\theta}\left(x_{t},t,c^{\edit}\right)+\left(1-w^{\edit}\right)\epsilon_{\theta}\left(x_{t},t,\varnothing\right)\nonumber \\
 & \ \ \ \ \ \ \ -\left(w^{\orig}\epsilon_{\theta}\left(x_{t},t,c^{\orig}\right)+\left(1-w^{\orig}\right)\epsilon_{\theta}\left(x_{t},t,\varnothing\right)\right)\Big)\\
=\  & \frac{-1}{\sigma_{t}}\Big(w^{\edit}\epsilon_{\theta}\left(x_{t},t,c^{\edit}\right)-w^{\orig}\epsilon_{\theta}\left(x_{t},t,c^{\orig}\right)+\left(w^{\orig}-w^{\edit}\right)\epsilon_{\theta}\left(x_{t},t,\varnothing\right)\Big)\label{eq:closed_form_explicit_SD_h_score_5}\\
=\  & \frac{-1}{\sigma_{t}}f\left(x_{t},t\right)
\end{align}
Finding the formula of $\eta$ in Eq.~\ref{eq:LMC_3} can be somewhat
tricky. The key is to examine the equation $x_{t-1}^{\base}=x_{t}+\eta\nabla_{x_{t}}\log p\left(x_{t}\right)+\sqrt{2\eta}z$
in Eq.~\ref{eq:LMC_2}, which can be interpreted as sampling $x_{t-1}^{\base}$
from the Gaussian backward transition distribution $p_{\theta}\left(x_{t-1}|x_{t}\right)$.
This implies that if we omit the random term $\sqrt{2\eta}z$, the
simplified equation $x_{t-1}^{\base}=x_{t}+\eta\nabla_{x_{t}}\log p\left(x_{t}\right)$
corresponds to the mean of $p_{\theta}\left(x_{t-1}|x_{t}\right)$,
as provided in Eq.~\ref{eq:DDIM_sampling_SD}, and rewritten as follows:
\begin{align}
x_{t-1}^{\base} & =\underbrace{\frac{a_{t-1}}{a_{t}}x_{t}+\left(\sqrt{\sigma_{t-1}^{2}-\omega_{t,t-1}^{2}}-\frac{\sigma_{t}a_{t-1}}{a_{t}}\right)\tilde{\epsilon}_{\theta}\left(x_{t},t,c^{\orig}\right)}_{\tilde{\mu}_{\theta,\omega,t,t-1}\left(x_{t},c^{\orig}\right)}\label{eq:closed_form_explicit_SD_xbase_1}\\
 & =\frac{a_{t-1}}{a_{t}}x_{t}+\left(\sqrt{\sigma_{t-1}^{2}-\omega_{t,t-1}^{2}}-\frac{\sigma_{t}a_{t-1}}{a_{t}}\right)\left(w^{\orig}\epsilon_{\theta}\left(x_{t},t,c^{\orig}\right)+\left(1-w^{\orig}\right)\epsilon_{\theta}\left(x_{t},t,\varnothing\right)\right)\label{eq:closed_form_explicit_SD_xbase_2}\\
 & =\frac{a_{t-1}}{a_{t}}x_{t}-\left(\sqrt{\sigma_{t-1}^{2}-\omega_{t,t-1}^{2}}-\frac{\sigma_{t}a_{t-1}}{a_{t}}\right)\sigma_{t}\nabla_{x_{t}}\log p\left(x_{t}\right)\label{eq:closed_form_explicit_SD_xbase_3}
\end{align}
Eq.~\ref{eq:closed_form_explicit_SD_xbase_3} suggests that $\eta=-\left(\sqrt{\sigma_{t-1}^{2}-\omega_{t,t-1}^{2}}-\frac{\sigma_{t}a_{t-1}}{a_{t}}\right)\sigma_{t}$.
One can easily verify that $\eta>0$. It is worth noting that there
is a little mismatch between the coefficients of $x_{t}$ in Eq.~\ref{eq:closed_form_explicit_SD_xbase_1}
and in $x_{t-1}^{\base}=x_{t}+\eta\nabla_{x_{t}}\log p\left(x_{t}\right)$.
This is expected because the standard LMC update assumes a forward
diffusion process governed by the SDE $dx_{t}=\sqrt{2}dw_{t}$, which
lacks a drift term. In contrast, the continuous-time forward process
of Stable Diffusion follows the SDE $dx_{t}=\frac{-\beta_{t}}{2}x_{t}dt+\sqrt{\beta_{t}}dw_{t}$,
which has the drift term $\frac{-\beta_{t}}{2}x_{t}$.

It can be inferred that $u_{t}^{\orig}$ mimics the random term $\sqrt{2\eta}z$,
with the key difference being that it is precomputed during the forward
pass rather than randomly sampled during the backward pass.

According to the above analysis, the explicit $h$-Edit update for
Stable Diffusion is given by:
\begin{align}
x_{t-1}^{\base} & =\underbrace{\tilde{\mu}_{\theta,\omega,t,t-1}\left(x_{t}^{\edit},c^{\orig}\right)}_{x_{t}+\eta\nabla\log p\left(x_{t}\right)}+\underbrace{u_{t}^{\orig}}_{\sqrt{2\eta}z}\label{eq:closed_form_explicit_SD_xbase_4}\\
x_{t-1}^{\edit} & =x_{t-1}^{\base}+\underbrace{\left(-\left(\sqrt{\sigma_{t-1}^{2}-\omega_{t,t-1}^{2}}-\frac{\sigma_{t}a_{t-1}}{a_{t}}\right)\sigma_{t}\right)}_{\eta}\underbrace{\frac{-1}{\sigma_{t}}f\left(x_{t}^{\edit},t\right)}_{\nabla\log h\left(x_{t},t\right)}\\
 & =x_{t-1}^{\base}+\left(\sqrt{\sigma_{t-1}^{2}-\omega_{t,t-1}^{2}}-\frac{\sigma_{t}a_{t-1}}{a_{t}}\right)f\left(x_{t}^{\edit},t\right)
\end{align}

To derive the implicit $h$-Edit update, we first write Eq.~\ref{eq:closed_form_explicit_SD_xbase_1}
in the implicit form $x_{t-1}=\frac{a_{t-1}}{a_{t}}x_{t}+\left(\sqrt{\sigma_{t-1}^{2}-\omega_{t,t-1}^{2}}-\frac{\sigma_{t}a_{t-1}}{a_{t}}\right)\tilde{\epsilon}_{\theta}\left(x_{t-1},t-1,c^{\orig}\right)$,
which reveals that $\gamma=-\left(\sqrt{\sigma_{t-1}^{2}-\omega_{t,t-1}^{2}}-\frac{\sigma_{t}a_{t-1}}{a_{t}}\right)\sigma_{t-1}$.
Using this, we compute $x_{t-1}^{\edit}$ based on the formula in
Eq.~\ref{eq:implicit_LMC_3} as follows:
\begin{align}
x_{t-1}^{\edit} & =x_{t-1}^{\base}+\gamma\nabla_{x_{t-1}}h\left(x_{t-1}^{\base},t-1\right)\\
 & =x_{t-1}^{\base}+\left(-\left(\sqrt{\sigma_{t-1}^{2}-\omega_{t,t-1}^{2}}-\frac{\sigma_{t}a_{t-1}}{a_{t}}\right)\sigma_{t-1}\right)\frac{1}{\sigma_{t-1}}f\left(x_{t-1}^{\base},t-1\right)\\
 & =x_{t-1}^{\base}+\left(\sqrt{\sigma_{t-1}^{2}-\omega_{t,t-1}^{2}}-\frac{\sigma_{t}a_{t-1}}{a_{t}}\right)f\left(x_{t-1}^{\base},t-1\right)
\end{align}
where $x_{t-1}^{\base}$ is given in Eq.~\ref{eq:closed_form_explicit_SD_xbase_4}.

One advantage of the natural disentanglement in the $h$-Edit update
is that the guidance scales $w^{\orig}$ for computing $x_{t-1}^{\base}$
in Eq.~\ref{eq:closed_form_explicit_SD_xbase_2} and $w^{\orig}$
for computing $\nabla\log h\left(x_{t},t\right)$ in Eq.~\ref{eq:closed_form_explicit_SD_h_score_5}
do \emph{not} need to be the same. This allows $w^{\orig}$ in Eq.~\ref{eq:closed_form_explicit_SD_xbase_2}
to follow the guidance scale used in the forward pass, while $w^{\orig}$
in Eq.~\ref{eq:closed_form_explicit_SD_h_score_5} can be chosen
arbitrarily. To emphasize this distinction, we denote $w^{\orig}$
in Eq.~\ref{eq:closed_form_explicit_SD_h_score_5} as $\hat{w}^{\orig}$,
indicating that it may differ from $w^{\orig}$ in Eq.~\ref{eq:closed_form_explicit_SD_xbase_2}.
This $\hat{w}^{\orig}$ can be interpreted as a hyperparameter controlling
how much of the original image's information is excluded from the
editing process. During our experiments, we observed that $w^{\orig}$,
$\hat{w}^{\orig}$, and $w^{\edit}$ should be chosen such that $0<w^{\orig}\leq\hat{w}^{\orig}<w^{\edit}$.

\section{Algorithms\label{sec:Algorithms}}

\subsection{$h$-Edit for Combined Editing\label{subsec:algorithm_h_edit}}

In Algorithms~\ref{alg:explicit_h_edit} and \ref{alg:implicit_h_edit},
we provide pseudo-codes for the explicit and implicit versions of
$h$-Edit for combined text-guided and reward-model-based editing.

\begin{algorithm}[H]
\begin{algorithmic}[1] 

\Require Original image $x_{0}^{\orig}$, reference image $x_{0}^{\refer}$,
original text $c^{\orig}$, edited text $c^{\edit}$, guidance weights
$w^{\orig}$, $w^{\edit}$, $\hat{w}^{\orig}$, external encoder $G$,
external distance loss $\mathcal{L}$, external guidance weight $\rho_{t}$.

\State $\left\{ x_{t}^{\orig}\right\} _{t=1}^{T}$, $\left\{ u_{t}^{\orig}\right\} _{t=1}^{T}$
= $\text{Inversion}\left(x_{0}^{\orig},c^{\orig}\right)$

\State $x_{T}^{\edit}=x_{T}^{\orig}$

\For {$t=T,\dots,1$}

	\State $x_{t}=x_{t}^{\edit}$

	\State $\tilde{\epsilon}_{\theta}\left(x_{t},t,c^{\orig}\right)=w^{\orig}\epsilon_{\theta}\left(x_{t},t,c^{\orig}\right)+\left(1-w^{\orig}\right)\epsilon_{\theta}\left(x_{t},t,\varnothing\right)$

	\State Compute $\tilde{\mu}_{\theta,\omega,t,t-1}\left(x_{t},c^{\orig}\right)$
from $\tilde{\epsilon}_{\theta}\left(x_{t},t,c^{\orig}\right)$ via
Eq.~\ref{eq:backward_trans_prob}

	\State $\ensuremath{x_{t-1}^{\base}=\tilde{\mu}_{\theta,\omega,t,t-1}\left(x_{t},c^{\orig}\right)+u_{t}^{\orig}}$

	\If {text-guided editing}

		\If {combined with P2P}

			\State Get the attention map $M_{t}^{\edit}$ from $\epsilon_{\theta}\left(x_{t},t,c^{\edit}\right)$

			\State Get the attention map $M_{t}^{\orig}$ from $\epsilon_{\theta}\left(x_{t}^{\orig},t,c^{\orig}\right)$

			\State $\hat{M}_{t}^{\edit}=\text{P2P}\left(M_{t}^{\edit},M_{t}^{\orig},t\right)$

			\State Apply the new attention map $\hat{M}_{t}^{\edit}$ to $\epsilon_{\theta}\left(x_{t},t,c^{\edit}\right)$

		\EndIf 

		\State $f\left(x_{t},t\right)=w^{\edit}\epsilon_{\theta}\left(x_{t},t,c^{\edit}\right)-\hat{w}^{\orig}\epsilon_{\theta}\left(x_{t},t,c^{\orig}\right)+\left(\hat{w}^{\orig}-w^{\edit}\right)\epsilon_{\theta}\left(x_{t},t,\varnothing\right)$

		\State $\hat{x}_{t-1}=x_{t-1}^{\base}+\left(\sqrt{\sigma_{t-1}^{2}-\omega_{t,t-1}^{2}}-\frac{\sigma_{t}a_{t-1}}{a_{t}}\right)f\left(x_{t},t\right)$

		\State $\hat{\epsilon}_{t}=\text{stop\_grad}\left(w^{\edit}\epsilon_{\theta}\left(x_{t},t,c^{\edit}\right)+\left(1-w^{\edit}\right)\epsilon_{\theta}\left(x_{t},t,\varnothing\right)\right)$

	\Else

		\State $\hat{x}_{t-1}=x_{t-1}^{\base}$

		\State $\hat{\epsilon}_{t}=\text{stop\_grad}\left(w^{\orig}\epsilon_{\theta}\left(x_{t},t,c^{\orig}\right)+\left(1-w^{\orig}\right)\epsilon_{\theta}\left(x_{t},t,\varnothing\right)\right)$

	\EndIf 

	\State $x_{0|t}=\dfrac{x_{t}-\sigma_{t}\hat{\epsilon}_{t}}{a_{t}}$

	\State  $g_{t}=-\nabla_{x_{t}}\mathcal{L}\left(G\left(x_{0|t}\right),G\left(x_{0}^{\refer}\right)\right)$

	\State  $x_{t-1}^{\edit}=\hat{x}_{t-1}+\rho_{t}g_{t}$

	\If {text-guided editing \textbf{and} combined with P2P \textbf{and}
local blending}

		\State $x_{t-1}^{\edit}=\text{local\_blend}\left(x_{t-1}^{\edit},x_{t-1}^{\orig}\right)$

	\EndIf 

\EndFor 

\end{algorithmic}\caption{Explicit $h$-Edit for combined editing, compatible with both deterministic
and random inversion, and supporting integration with the P2P \cite{hertz2023prompttoprompt}.\label{alg:explicit_h_edit}}
\end{algorithm}

\begin{algorithm}[H]
\begin{algorithmic}[1] 

\Require Original image $x_{0}^{\orig}$, reference image $x_{0}^{\refer}$,
original text $c^{\orig}$, edited text $c^{\edit}$, guidance weights
$w^{\orig}$, $w^{\edit}$, $\hat{w}^{\orig}$, reconstruction weight
$\lambda_{t}$, external encoder $G$, external distance loss $\mathcal{L}$,
external guidance weight $\rho_{t}$, number of implicit loops $K$.

\State $\left\{ x_{t}^{\orig}\right\} _{t=1}^{T}$, $\left\{ u_{t}^{\orig}\right\} _{t=1}^{T}$
= $\text{Inversion}\left(x_{0}^{\orig},c^{\orig}\right)$

\State $x_{T}^{\edit}=x_{T}^{\orig}$

\For {$t=T,\dots,1$}

	\State $x_{t}=x_{t}^{\edit}$

	\State $\tilde{\epsilon}_{\theta}\left(x_{t},t,c^{\orig}\right)=w^{\orig}\epsilon_{\theta}\left(x_{t},t,c^{\orig}\right)+\left(1-w^{\orig}\right)\epsilon_{\theta}\left(x_{t},t,\varnothing\right)$

	\State Compute $\tilde{\mu}_{\theta,\omega,t,t-1}\left(x_{t},c^{\orig}\right)$
from $\tilde{\epsilon}_{\theta}\left(x_{t},t,c^{\orig}\right)$ via
Eq.~\ref{eq:backward_trans_prob}

	\State $\ensuremath{x_{t-1}^{\base}=\tilde{\mu}_{\theta,\omega,t,t-1}\left(x_{t},c^{\orig}\right)+u_{t}^{\orig}}$

	\State $x_{t-1}^{(0)}=x_{t-1}^{\base}$

	\For {$k=0,\dots,K-1$}

		\If {improving reconstruction}

			\State $r_{t-1}=x_{t-1}^{(k)}-x_{t-1}^{\base}$

			\State $x_{t-1}^{(k)}=x_{t-1}^{(k)}-\lambda_{t-1}r_{t-1}$

		\EndIf 

		\If {text-guided editing}

			\If {combined with P2P}

				\State Get the attention map $M_{t-1}^{\edit}$ from $\epsilon_{\theta}\left(x_{t-1}^{(k)},t-1,c^{\edit}\right)$

				\State Get the attention map $M_{t-1}^{\orig}$ from $\epsilon_{\theta}\left(x_{t-1}^{\orig},t-1,c^{\orig}\right)$

				\State $\hat{M}_{t-1}^{\edit}=\text{P2P}\left(M_{t-1}^{\edit},M_{t-1}^{\orig},t-1\right)$

				\State Apply the new attention map $\hat{M}_{t-1}^{\edit}$ to
$\epsilon_{\theta}\left(x_{t-1}^{(k)},t-1,c^{\edit}\right)$

			\EndIf 

			\State $f\left(x_{t-1}^{(k)},t-1\right)=w^{\edit}\epsilon_{\theta}\left(x_{t-1}^{(k)},t-1,c^{\edit}\right)-\hat{w}^{\orig}\epsilon_{\theta}\left(x_{t-1}^{(k)},t-1,c^{\orig}\right)+\left(\hat{w}^{\orig}-w^{\edit}\right)\epsilon_{\theta}\left(x_{t-1}^{(k)},t-1,\varnothing\right)$

			\State $\hat{x}_{t-1}=x_{t-1}^{(k)}+\left(\sqrt{\sigma_{t-1}^{2}-\omega_{t,t-1}^{2}}-\frac{\sigma_{t}a_{t-1}}{a_{t}}\right)f\left(x_{t-1}^{(k)},t-1\right)$

			\State $\hat{\epsilon}_{t-1}=\text{stop\_grad}\left(w^{\edit}\epsilon_{\theta}\left(x_{t-1}^{(k)},t-1,c^{\edit}\right)+\left(1-w^{\edit}\right)\epsilon_{\theta}\left(x_{t-1}^{(k)},t-1,\varnothing\right)\right)$

		\Else

			\State $\hat{x}_{t-1}=x_{t-1}^{(k)}$

			\State $\hat{\epsilon}_{t-1}=\text{stop\_grad}\left(w^{\orig}\epsilon_{\theta}\left(x_{t-1}^{(k)},t-1,c^{\orig}\right)+\left(1-w^{\orig}\right)\epsilon_{\theta}\left(x_{t-1}^{(k)},t-1,\varnothing\right)\right)$

		\EndIf 

		\State $x_{0|t-1}=\dfrac{\hat{x}_{t-1}-\sigma_{t-1}\hat{\epsilon}_{t-1}}{a_{t-1}}$

		\State  $g_{t-1}=-\nabla_{\hat{x}_{t-1}}\mathcal{L}\left(G\left(x_{0|t-1}\right),G\left(x_{0}^{\refer}\right)\right)$

		\State  $x_{t-1}^{\left(k+1\right)}=\hat{x}_{t-1}+\rho_{t-1}g_{t-1}$

	\EndFor 

	\State $x_{t-1}^{\text{\ensuremath{\edit}}}=x_{t-1}^{(K)}$

	\If {text-guided editing \textbf{and} combined with P2P \textbf{and}
local blending}

		\State $x_{t-1}^{\edit}=\text{local\_blend}\left(x_{t-1}^{\edit},x_{t-1}^{\orig}\right)$

	\EndIf 

\EndFor 

\end{algorithmic}\caption{Implicit $h$-Edit for combined editing, compatible with both deterministic
and random inversions, and supporting integration with the P2P \cite{hertz2023prompttoprompt}.\label{alg:implicit_h_edit}}
\end{algorithm}

\subsection{Edit Friendly for Combined Editing\label{subsec:algorithm_ef}}

In this work, we extend Edit Friendly \cite{huberman2024edit} to
combined text-guided and reward-model-based editing tasks by combining
it with the technique in \cite{yu2023freedom}. The pseudo-code for
this extension is provided in Algorithm~\ref{alg:EF_for_combined_editing}.
This extension serves as a baseline for our method in the combined
editing setting.

\begin{algorithm}
\begin{algorithmic}[1] 

\Require Original image $x_{0}^{\orig}$, reference image $x_{0}^{\refer}$,
original text $c^{\orig}$, edited text $c^{\edit}$, guidance weights
$w^{\orig}$, $w^{\edit}$, external encoder $G$, external distance
loss $\mathcal{L}$, external guidance weight $\rho_{t}$.

\State $x_{T}^{\orig},\left\{ u_{t}^{\orig}\right\} _{t=1}^{T}=\text{RandomInversion}(x_{0}^{\orig},c^{\orig})$

\State $x_{T}^{\edit}=x_{T}^{\orig}$

\For {$t=T,\dots,1$}

	\State $x_{t}=x_{t}^{\edit}$

		\If {text-guided editing}

			\If {combined with P2P}

				\State Get the attention map $M_{t}^{\edit}$ from $\epsilon_{\theta}\left(x_{t},t,c^{\edit}\right)$

				\State Get the attention map $M_{t}^{\orig}$ from $\epsilon_{\theta}\left(x_{t}^{\orig},t,c^{\orig}\right)$

				\State $\hat{M}_{t}^{\edit}=\text{P2P}\left(M_{t}^{\edit},M_{t}^{\orig},t\right)$

				\State Apply the new attention map $\hat{M}_{t}^{\edit}$ to
$\epsilon_{\theta}\left(x_{t},t,c^{\edit}\right)$

			\EndIf 

			\State $\tilde{\epsilon}_{\theta}\left(x_{t},t\right)=w^{\edit}\epsilon_{\theta}\left(x_{t},t,c^{\edit}\right)+\left(1-w^{\edit}\right)\epsilon_{\theta}\left(x_{t},t,\varnothing\right)$

		\Else

			\State $\tilde{\epsilon}_{\theta}\left(x_{t},t\right)=w^{\orig}\epsilon_{\theta}\left(x_{t},t,c^{\orig}\right)+\left(1-w^{\orig}\right)\epsilon_{\theta}\left(x_{t},t,\varnothing\right)$

		\EndIf 

	\State Compute $\tilde{\mu}_{\theta,\omega,t,t-1}\left(x_{t},t\right)$
from $\tilde{\epsilon}_{\theta}\left(x_{t},t\right)$ via Eq.~\ref{eq:backward_trans_prob}

	\State $x_{0|t}=\dfrac{x_{t}-\sigma_{t}\tilde{\epsilon}_{\theta}\left(x_{t},t\right)}{a_{t}}$
where $a_{t}=\sqrt{\bar{\alpha}_{t}}$ and $\sigma_{t}=\sqrt{1-\bar{\alpha}_{t}}$

	\State  $g_{t}=-\nabla_{x_{t}}\mathcal{L}\left(G\left(x_{0|t}\right),G\left(x_{0}^{\refer}\right)\right)$

	\State $x_{t-1}^{\edit}=\tilde{\mu}_{\theta,\omega,t,t-1}\left(x_{t},t\right)+\rho_{t}g_{t}+u_{t}^{\orig}$

	\If {text-guided editing \textbf{and} combined with P2P \textbf{and}
local blending}

		\State $x_{t-1}^{\edit}=\text{local\_blend}\left(x_{t-1}^{\edit},x_{t-1}^{\orig}\right)$

	\EndIf 

\EndFor 

\State \textbf{return} $x_{0}^{\edit}$

\end{algorithmic}

\caption{Edit Friendly for combined editing, supporting integration with the
P2P \cite{hertz2023prompttoprompt}.\label{alg:EF_for_combined_editing}}
\end{algorithm}

\section{Additional Discussion on Related Work\label{sec:Additional-Related-Work}}

\subsection{Training-based Editing}

Training-based approaches, such as DiffusionCLIP \cite{kim2022diffusionclip}
and Asyrp \cite{kwon2023diffusion}, modify the noise network of a
pretrained diffusion model through fine-tuning or by incorporating
an auxiliary network, resulting in a new noise network that supports
generating images with the desired editing attributes. The local directional
CLIP loss \cite{gal2022stylegan} is commonly used as the training
objective. However, these methods require training a new network for
each specific editing target, limiting their adaptability to diverse
editing scenarios in practice. In contrast, InstructPix2Pix \cite{brooks2023instructpix2pix}
trains an entirely new diffusion model that generates images based
on editing instructions. The instruction texts and target edited images
for training are generated by GPT-3 \cite{brown2020language} and
P2P \cite{hertz2023prompttoprompt}, respectively, meaning that the
quality of the edits is inherently tied to P2P\textquoteright s performance.
Additionally, the high training cost remains a significant drawback
of this method. 

\subsection{Conditional Generation with Diffusion Models}

The goal of conditional generation is to sample data from the joint
distribution $p\left(x_{0}\right)p\left(y|x_{0}\right)$, which can
be achieved by learning the score $\nabla\log p\left(x_{t},y\right)$
of the joint distribution $p\left(x_{t},y\right)$ via the score matching
framework \cite{hyvarinen2005estimation,song2019generative}. Class-guided
diffusion model \cite{Dhariwal2021Diffusion} learns a noisy classifier
$p\left(y|x_{t}\right)$ and combines its gradient with the score
$\nabla\log p\left(x_{t}\right)$ learned by another unconditional
diffusion model (e.g., DDPM \cite{ho2020denoising}) to obtain $\nabla\log p\left(x_{t},y\right)$.
Meanwhile, classifier-free guidance \cite{Ho2022Classifier} simultaneously
learn both $\nabla\log p\left(x_{t}\right)$ and $\nabla\log p\left(x_{t}|y\right)$
using a single noise network. Energy-guided SDE (EGSDE) \cite{zhao2022egsde}
extends class-guided diffusion models to solve the image-to-image
translation problem. It utilizes a noisy classifier pretrained on
both the source and target domains to define a similarity score between
noisy samples from the two domains. This score acts as a negative
energy guiding the generation of target domain samples toward preserving
some properties of the corresponding source domain samples. The energy-based
perspective have also been considered in works on generating compositional
concepts with diffusion models \cite{liu2022compositional}. FreeDom
\cite{yu2023freedom} approximates the time-dependent energy function
in EGSDE using Tweedie's formula: $\mathcal{E}\left(c,x_{t},t\right)=\Expect_{p\left(x_{0}|x_{t}\right)}\left[\mathcal{E}\left(c,x_{0},t\right)\right]\approx\mathcal{E}\left(c,x_{0|t},t\right)$
\cite{efron2011tweedie,chung2023diffusion}. This eliminates the reliance
on a noisy classifier which is often difficult to obtain in practice
and allows FreeDom to leverage any available pretrained model on clean
samples $x_{0}$ to define the energy function. As a result, FreeDom
supports conditional information from segmentation maps, style images,
and face IDs. Similarly, UGD \cite{bansal2024universal} utilizes
Tweedie's formula but employs a different reparameterization for guidance
using external networks. 

The EGSDE framework can be considered as a special case of our reverse-time
bridge modeling framework, as ours applies to more general Markov
processes rather than just diffusion SDEs. Our framework also provides
a formula for the bridge's transition distribution, enabling ancestral
sampling in a discrete-time setting. Meanwhile, EGSDE usually relies
on the Euler-Maruyama method for approximate sampling because it only
has access to the instantaneous velocity at time $t$.

\subsection{Diffusion Bridges and Doob's $h$-Transform}

Most diffusion bridge methods \cite{de2021diffusion,liu2023learning,liu2023i2sb,somnath2023aligned,zhou2024denoising}
focus on the image-to-image translation problem which involves matching
two explicit distributions of two domains A, B. They typically assume
a diffusion model that generates domain A from Gaussian noise is given,
and apply Doob's $h$-transform \cite{doob1984classical} to the forward
process of this diffusion model to map samples of domain A to those
of domain B rather than Gaussian noise. Some approaches like \cite{somnath2023aligned,liu2023learning}
directly learn the $h$-function, while others \cite{zhou2024denoising}
utilize an analytical form of the $h$-function and learn the score
of the reverse bridge. Our method, in contrast, applies Doob's $h$-transform
to the backward process to map Gaussian noise to samples with the
desired target attributes.

\section{Further Details on Experimental Settings\label{sec:Additional-Experimental-Setup}}

\subsection{Text-guided Editing\label{subsec:Exp_setting_text_guided_editing}}

The P2P hyperparameters for deterministic-inversion-based methods
with P2P (including $h$-Edit-D + P2P) were configured based on the
setup in \cite{ju2023direct}. Specifically, the sampling step proportions
for self-attention and cross-attention controls were set to 0.6 and
0.4, respectively. For $h$-Edit-R and EF with P2P, the proportion
of sampling steps for self-attention control was adjusted to 0.35,
as 0.6 was found to be excessive for effective editing with these
methods. For $h$-Edit-R and EF without P2P, the first 15 steps were
skipped to ensure faithful reconstruction, as recommended in \cite{huberman2024edit}.
This skipping was not required for their P2P counterparts. For LEDITS++~\cite{brack2024ledits},
we adhered to the hyperparameters specified in the original paper.

\subsection{Face Swapping\label{subsec:Exp_setting_face_swapping}}

We utilized the official pretrained models for MegaFS, AFS, and DiffFace,
available at \href{https://github.com/zyainfal/One-Shot-Face-Swapping-on-Megapixels}{MegaFS},
\href{https://github.com/truongvu2000nd/AFS}{AFS}, and \href{https://github.com/hxngiee/DiffFace}{DiffFace},
respectively. Since the official pretrained model for FaceShifter
is unavailable, we used an unofficial pretrained model from \href{https://github.com/richarduuz/Research_Project/tree/master/ModelC}{this repository}.
For evaluation, we employed a pretrained ArcFace model with the IR-SE-50
backbone (\cite{tov2021designing,yu2023freedom}), available through
the \href{https://github.com/TreB1eN/InsightFace_Pytorch}{InsightFace}
library for evaluation. This model was also used in $h$-Edit-R, EF,
and FaceShifter\footnote{FaceShifter uses the ArcFace model with the IR-SE-50 backbone to extract
face identity embeddings during both training and generating swapped
faces.} for generating swapped faces. For DiffFace, the ArcFace model with
the ResNet101 backbone from its official code was used for face swapping.
MegaFS and AFS relied on the ArcFace model with the IR-SE-50 backbone
during training but not during face swapping. Additional evaluations
using other face identity representation models are provided in Appendix~\ref{subsec:Additional_results_face_swapping}.
CelebA-HQ images were resized to 256$\times$256 and cropped as $\texttt{x=x[:, :, 35:223, 32:220]}$
to prepare them for input into the ArcFace model. Following \cite{yu2023freedom},
we defined the coefficient $\rho_{t}$ for the identity similarity
reward gradient (Algorithms~\ref{alg:implicit_h_edit}, \ref{alg:explicit_h_edit},
\ref{alg:EF_for_combined_editing}) as $\rho^{\face}\times\sqrt{\bar{\alpha}_{t}}$,
where $\bar{\alpha}_{t}$ is the Stable Diffusion scheduler coefficient
at time step $t$. For $h$-Edit-R and EF, $\rho^{\face}$ was set
to 100.0. For h-Edit-R (3s),$\rho^{\face}$ was reduced to 50.0, which
provided a better balance between editing effectiveness and faithfulness
when using three optimization steps. To further enhance faithfulness
to the original image, we incorporated the negative LPIPS score as
an additional reward alongside identity similarity. The LPIPS score,
computed using a pretrained VGG network, measures the perceptual similarity
between $x_{0}^{\edit}$ and $x_{0}^{\text{\ensuremath{\orig}}}$.
The coefficient for this reward is similar to that of the identity
similarity reward. For post-processing, we applied a mask generated
by the face parsing model in \cite{yu2018bisenet} to preserve the
original background while applying edits to the face. This procedure
was consistent across all baselines. The face swapping results without
using masks are provided in Appdx.~\ref{subsec:Face-swapping-without-masks}.

\subsection{Combined Text-guided and Style Editing\label{subsec:Exp_setting_style_n_text_guided}}

In combined text-guided and style editing, we disabled local blending
in P2P as our experiments indicated that it negatively impacts style
editing performance. For EF + P2P, following~\cite{yu2023freedom},
we scaled the gradient norm of the style loss reward at each time
$t$ by the norm of $\left[\epsilon\left(x_{t},t,c^{\edit}\right)-\epsilon\left(x_{t},t,\varnothing\right)\right]$.
This corresponds to defining the coefficient $\rho_{t}$ for style
editing in EF + P2P as:
\begin{equation}
\rho_{t}:=\rho^{\style}*\frac{\left\Vert \left(\epsilon\left(x_{t},t,c^{\edit}\right)-\epsilon\left(x_{t},t,\varnothing\right)\right)\right\Vert _{2}}{\left\Vert g_{t}\right\Vert _{2}}
\end{equation}
For $h$-Edit-R + P2P, we scaled the gradient norm of the style reward
to match the norm of the text-guided editing function $f\left(\cdot\right)$
in Eq.~\ref{eq:hedit_SD_3}. This approach leverages the disentangled
update mechanism unique to our method (Sections~\ref{sec:Method}
and \ref{subsec:Closed-form-expressions-for-explicit-and-implicit}).
Accordingly, the coefficient $\rho_{t}$ for the style editing term
in $h$-Edit-R + P2P is defined as:
\begin{equation}
\rho_{t}:=\rho^{\style}*\frac{\left\Vert f\left(x_{t},t\right)\right\Vert _{2}}{\left\Vert g_{t}\right\Vert _{2}}
\end{equation}

\begin{figure*}
\begin{centering}
\begin{tabular}{cc}
\includegraphics[width=0.49\textwidth]{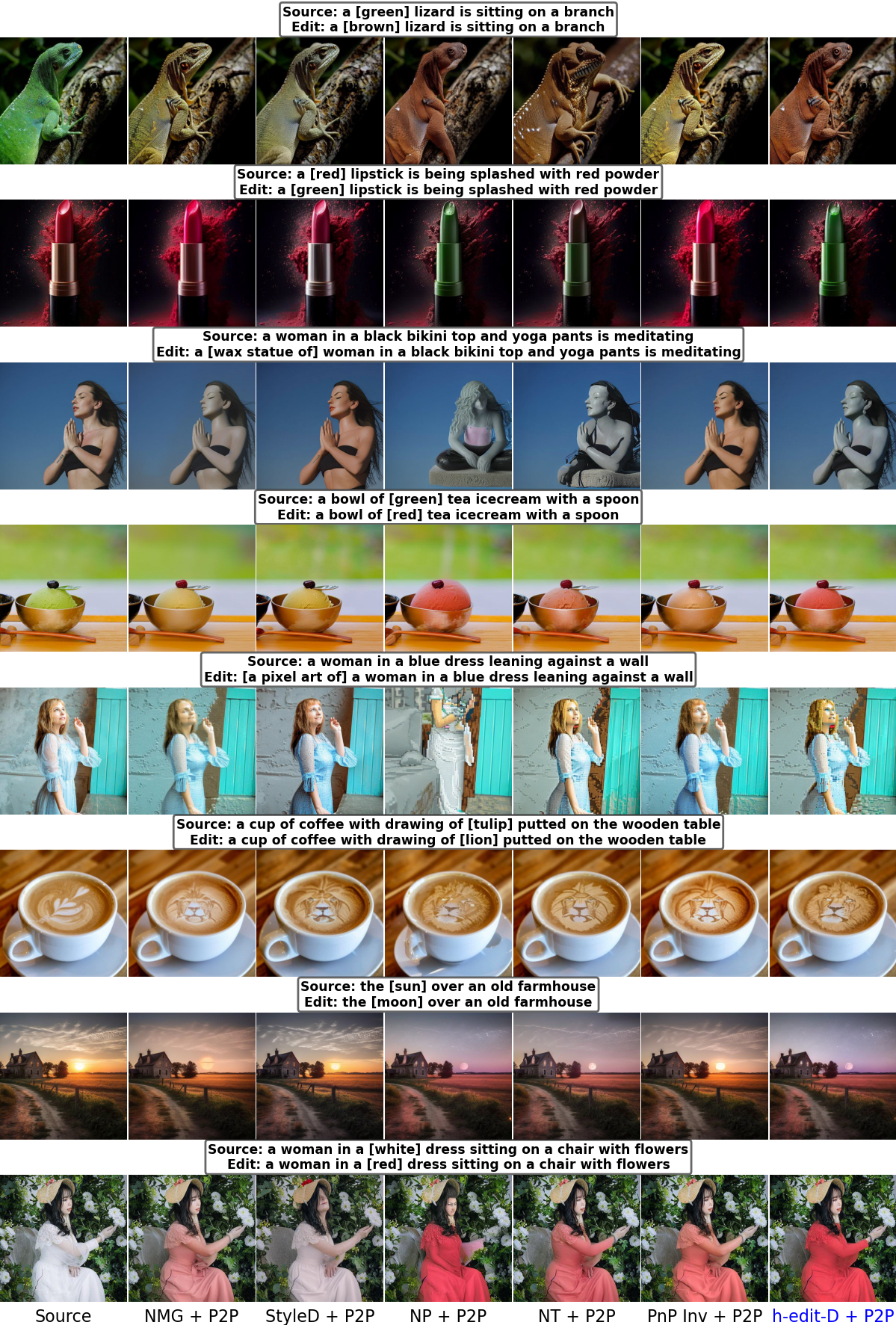} & \includegraphics[width=0.49\textwidth]{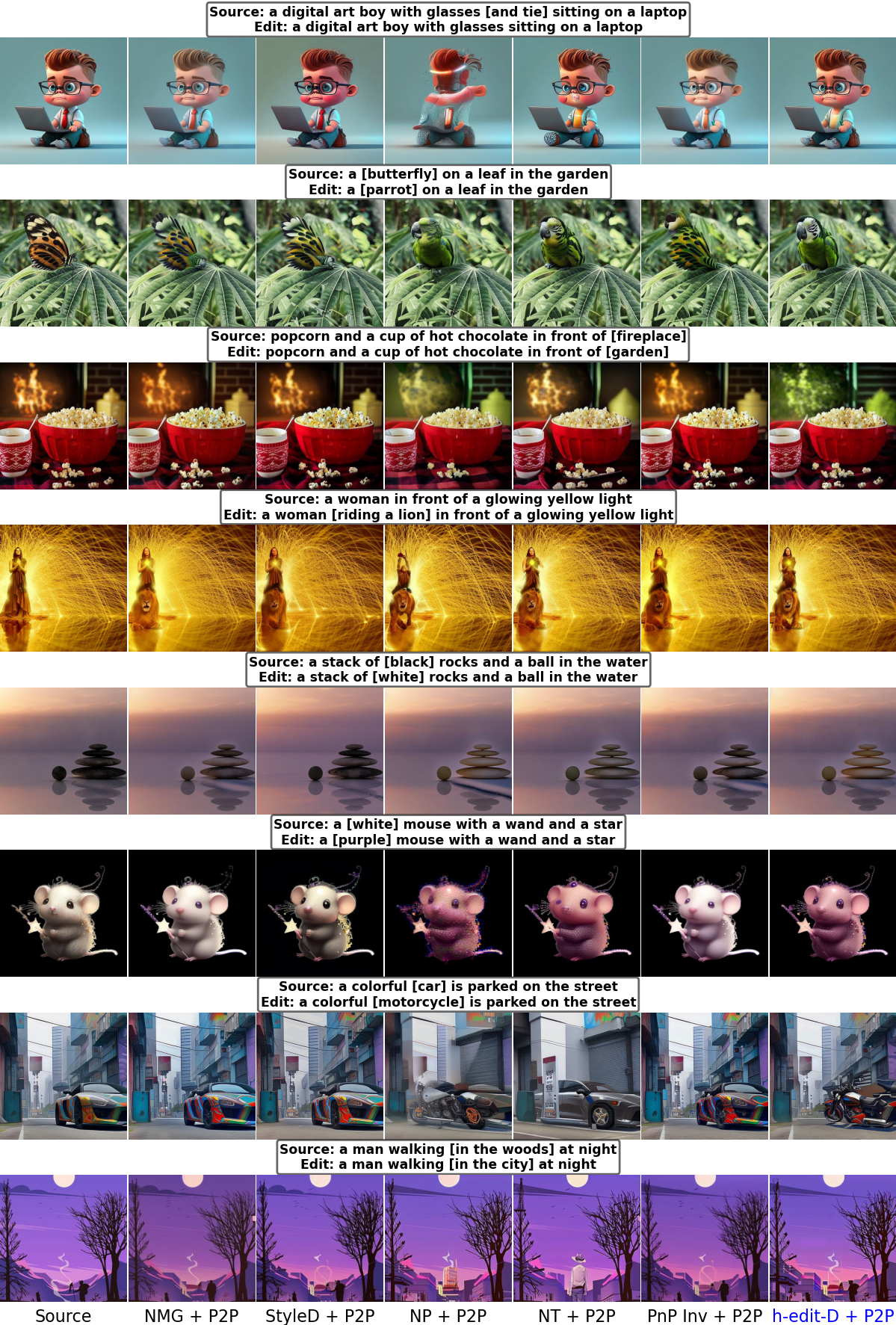}\tabularnewline
\end{tabular}
\par\end{centering}
\caption{Additional visual comparisons between $\protect\Model$-D + P2P and
other deterministic-inversion-based methods with P2P.\label{fig:Additional-visualization-h-Edit-D-P2P}}
\end{figure*}

\section{Additional Experimental Results\label{sec:Additional-Experimental-Results}}

\subsection{Text-guided Editing\label{subsec:Additional_results_text_guided}}

\subsubsection{Deterministic-inversion-based methods}

Fig.~\ref{fig:Additional-visualization-h-Edit-D-P2P} shows additional
edited images produced by $h$-Edit-D + P2P alongside other deterministic-inversion-based
editing methods with P2P \cite{cho2024noise,ju2023direct,li2023stylediffusion,mokady2023null,miyake2023negative}.
$h$-Edit-D + P2P consistently outperforms the baselines in handling
difficult edits, while maintaining faithful reconstruction, as reflected
in the quantitative results in Table~\ref{tab:Text-Guided-Editing-Results}.
For instance, our method successfully removes the boy's tie (first
row, right) and transforms the car into a motorcycle (seventh row,
right), tasks where most other methods struggle. Although NP + P2P
and NT + P2P demonstrate strong editing capabilities, as evidenced
by their high local CLIP similarity scores in Table~\ref{tab:Text-Guided-Editing-Results},
they are not good at preserving non-edited content compared to other
methods. Conversely, NMG + P2P, StyleD + P2P, and PnP Inv + P2P achieve
high fidelity to the original image, but fail to deliver effective
edits in many cases.

\begin{figure*}
\begin{centering}
\begin{tabular}{cc}
\includegraphics[width=0.36\textwidth]{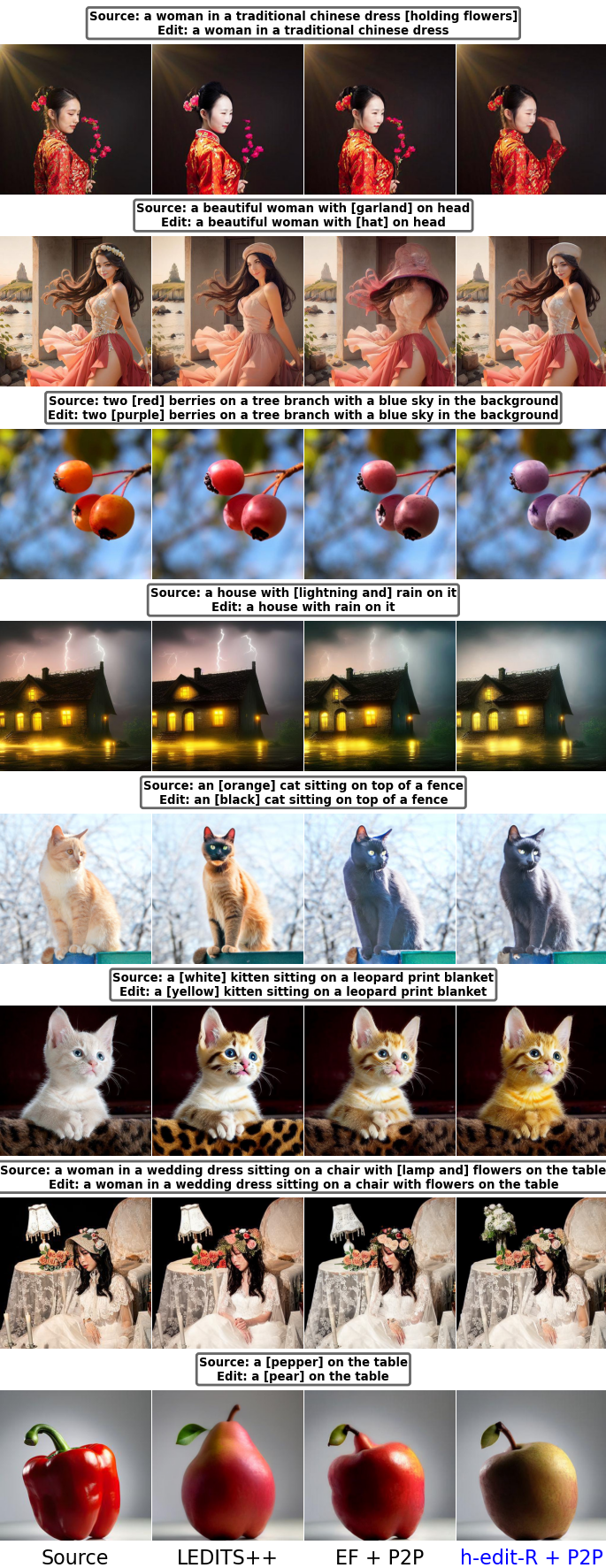} & \includegraphics[width=0.36\textwidth]{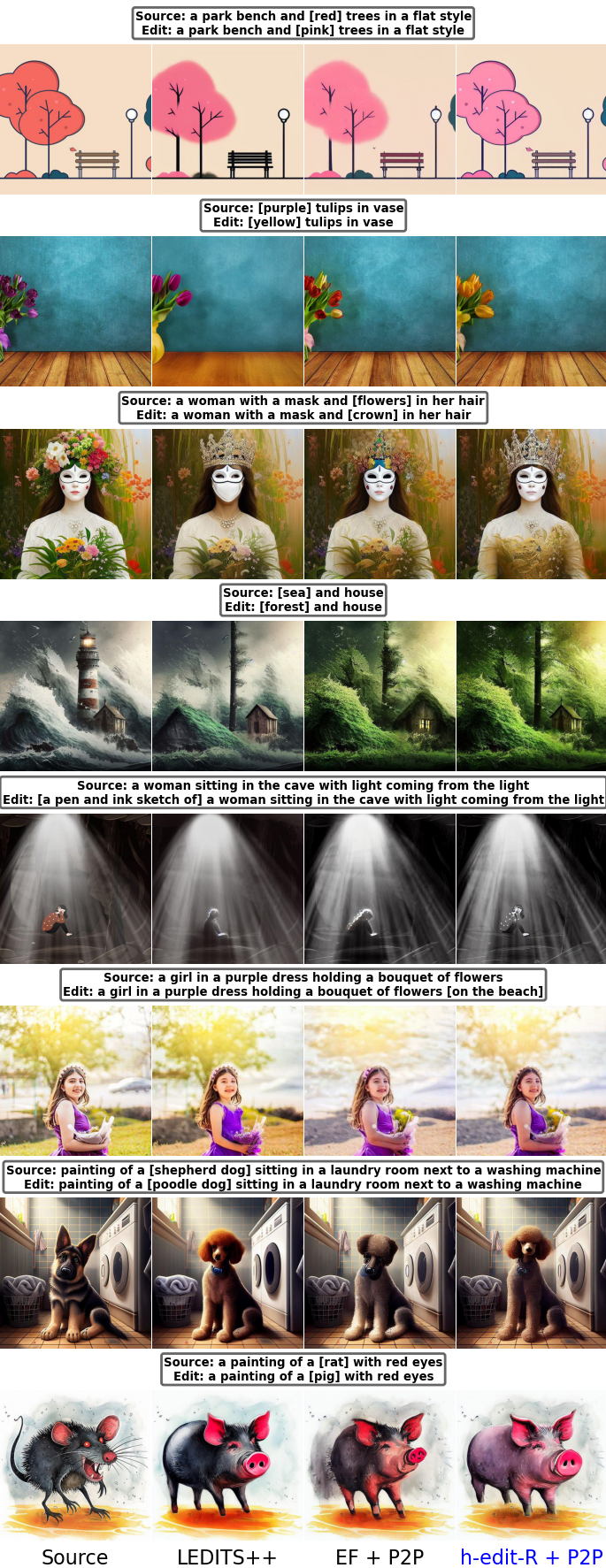}\tabularnewline
\end{tabular}
\par\end{centering}
\caption{Additional qualitative results of $\protect\Model$-R, EF, and LEDITS+++
with P2P.\label{fig:Additional-visualization-h-Edit-R-P2P}}
\end{figure*}

\subsubsection{Random-inversion-based methods}

In Fig.~\ref{fig:Additional-visualization-h-Edit-R-P2P}, we present
additional visual comparisons of $h$-Edit-R + P2P against EF + P2P
and LEDITS++. These visualizations are consistent with the quantitative
results in Table~\ref{tab:Text-Guided-Editing-Results}, confirming
that our method surpasses both EF + P2P and LEDITS++ in editing effectiveness
and faithfulness. Further qualitative results of $h$-Edit-R and EF
without P2P are shown in Fig.~\ref{fig:Additional-visualization-h-Edit-R},
where our method once again demonstrates superior performance.

\begin{figure}
\begin{centering}
\begin{tabular}{cc}
\includegraphics[width=0.3\textwidth]{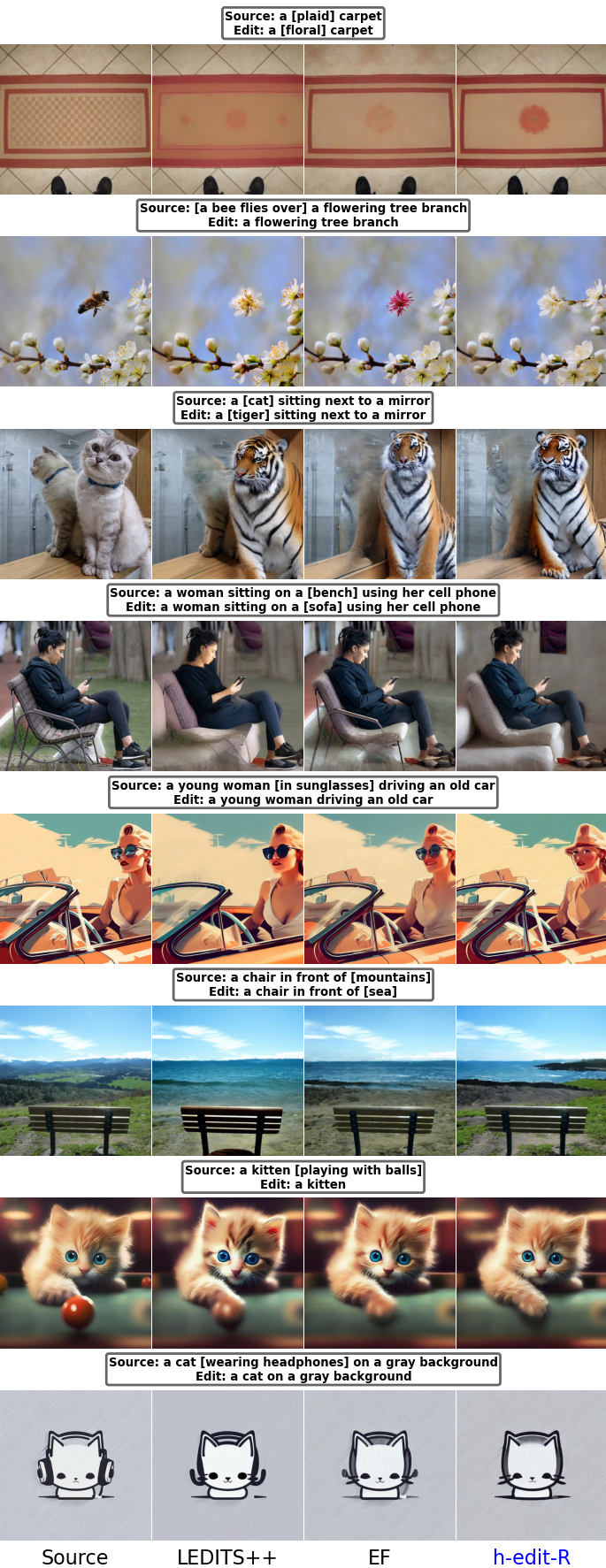} & \includegraphics[width=0.3\textwidth]{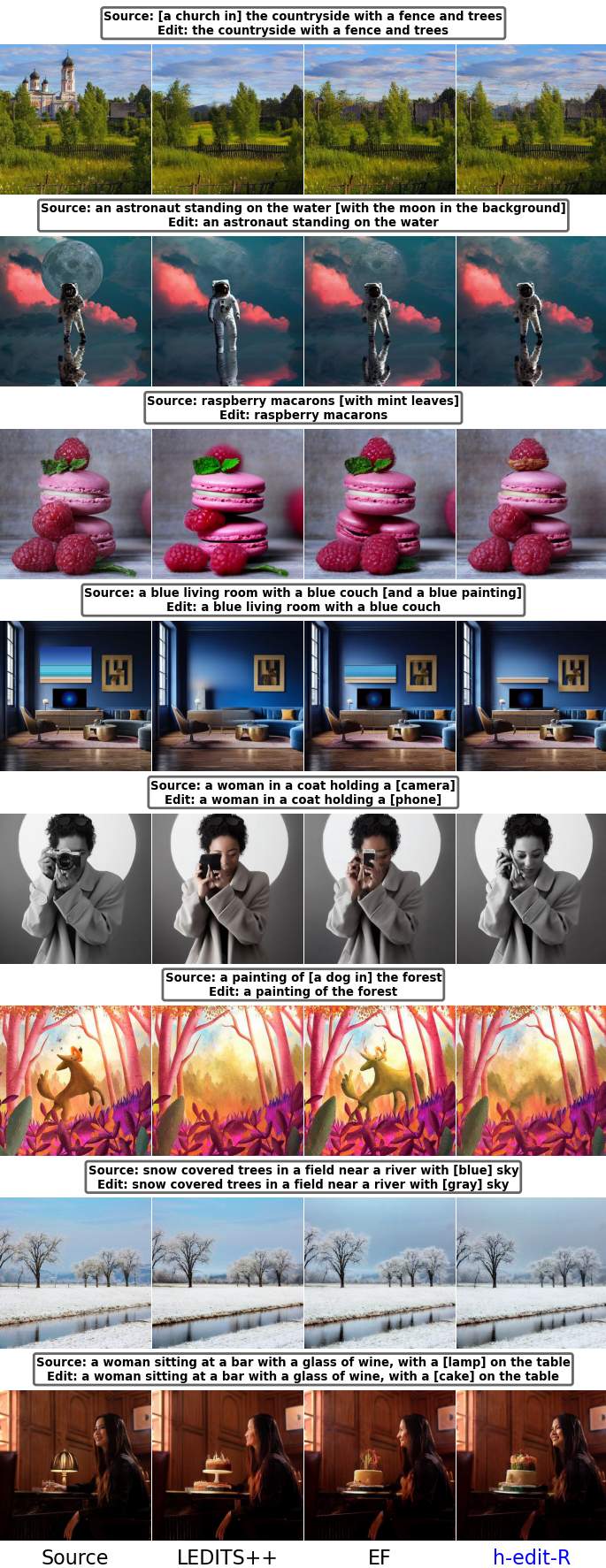}\tabularnewline
\end{tabular}
\par\end{centering}
\caption{Additional qualitative results of $\protect\Model$-R, EF (without
P2P), and LEDITS++.\label{fig:Additional-visualization-h-Edit-R}}
\end{figure}

\subsection{Face Swapping\label{subsec:Additional_results_face_swapping}}

Since $h$-Edit-R, EF, and FaceShifter utilize the same ArcFace model
for both face swapping and evaluation, this may lead to more favorable
identity matching results for these methods compared to other baselines.
To ensure a fair comparison, we reassessed the identity transfer quality
of all methods using alternative face identity representation models.
Specifically, we used VGG-Face \cite{parkhi2015deep}, FaceNet128,
FaceNet512 \cite{schroff2015facenet} and ArcFace with the ResNet34
backbone. These models were implemented in TensorFlow with pretrained
weights available through the \href{https://github.com/serengil/deepface}{DeepFace repository}
\cite{serengil2024lightface,serengil2020lightface}. Quantitative
results of this evaluation are provided in Table~\ref{tab:face_swapping_results_other_networks}.

Interestingly, DiffFace achieves the best performance across all face
identity representation models used for evaluation. $h$-Edit-R (3s)
and $h$-Edit-R rank second and third, respectively, outperforming
EF and FaceShifter but falling slightly short of DiffFace. This demonstrates
that our method is capable of effective face swapping, even without
being explicitly trained for this task like DiffFace, as further illustrated
by the qualitative results in Fig.~\ref{fig:Additional-qualitative-results-face-swapping}.
We hypothesize that DiffFace's good performance may be attributed
to (i) its use of an ArcFace model with a larger backbone (ResNet101)
for face swapping and (ii) training on a larger dataset compared to
the pretrained diffusion model employed by our method.

\begin{table}
\begin{centering}
{\resizebox{0.83\textwidth}{!}{
\begin{tabular}{ccccccccc}
\toprule 
Model & Metric & FaceShifter & MegaFS & AFS & DiffFace & EF & $h$-edit-R & $h$-edit-R (3s)\tabularnewline
\midrule
\midrule 
ArcFace (ResNet34) & Cosine Sim. \textbf{$\uparrow$} & 0.54 & 0.33 & 0.44 & \textbf{0.56} & 0.50 & 0.52 & \uline{0.55}\tabularnewline
\midrule 
VGG-Face & L2 Dist. $\downarrow$ & 0.99 & 1.12 & 1.03 & \textbf{0.96} & 1.02 & 1.00 & \uline{0.97}\tabularnewline
\midrule
FaceNet128 & L2 Dist. $\downarrow$ & 0.83 & 1.02 & 0.86 & \textbf{0.77} & 0.83 & \uline{0.80} & \textbf{0.77}\tabularnewline
\midrule
FaceNet512 & L2 Dist. $\downarrow$ & \uline{0.81} & 1.01 & 0.87 & \textbf{0.77} & 0.83 & \uline{0.81} & \textbf{0.77}\tabularnewline
\bottomrule
\end{tabular}{\small{}}}}{\small\par}
\par\end{centering}
\caption{Face identity transfer results evaluated using face identity representation
models different from the ArcFace model with the IR-SE-50 backbone.\label{tab:face_swapping_results_other_networks}}
\end{table}

\begin{figure*}
\begin{centering}
\begin{tabular}{cc}
\includegraphics[width=0.49\textwidth]{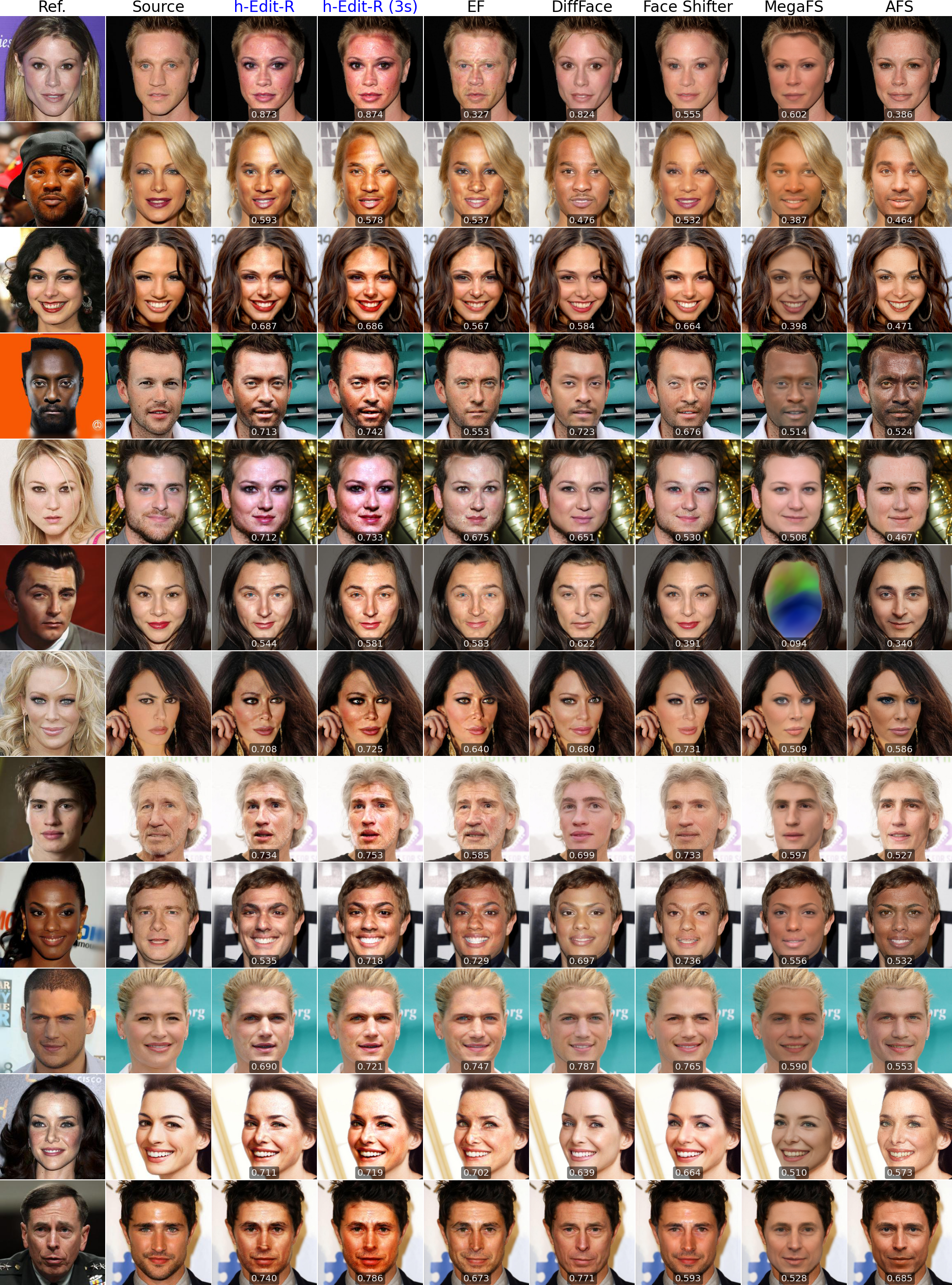} & \includegraphics[width=0.49\textwidth]{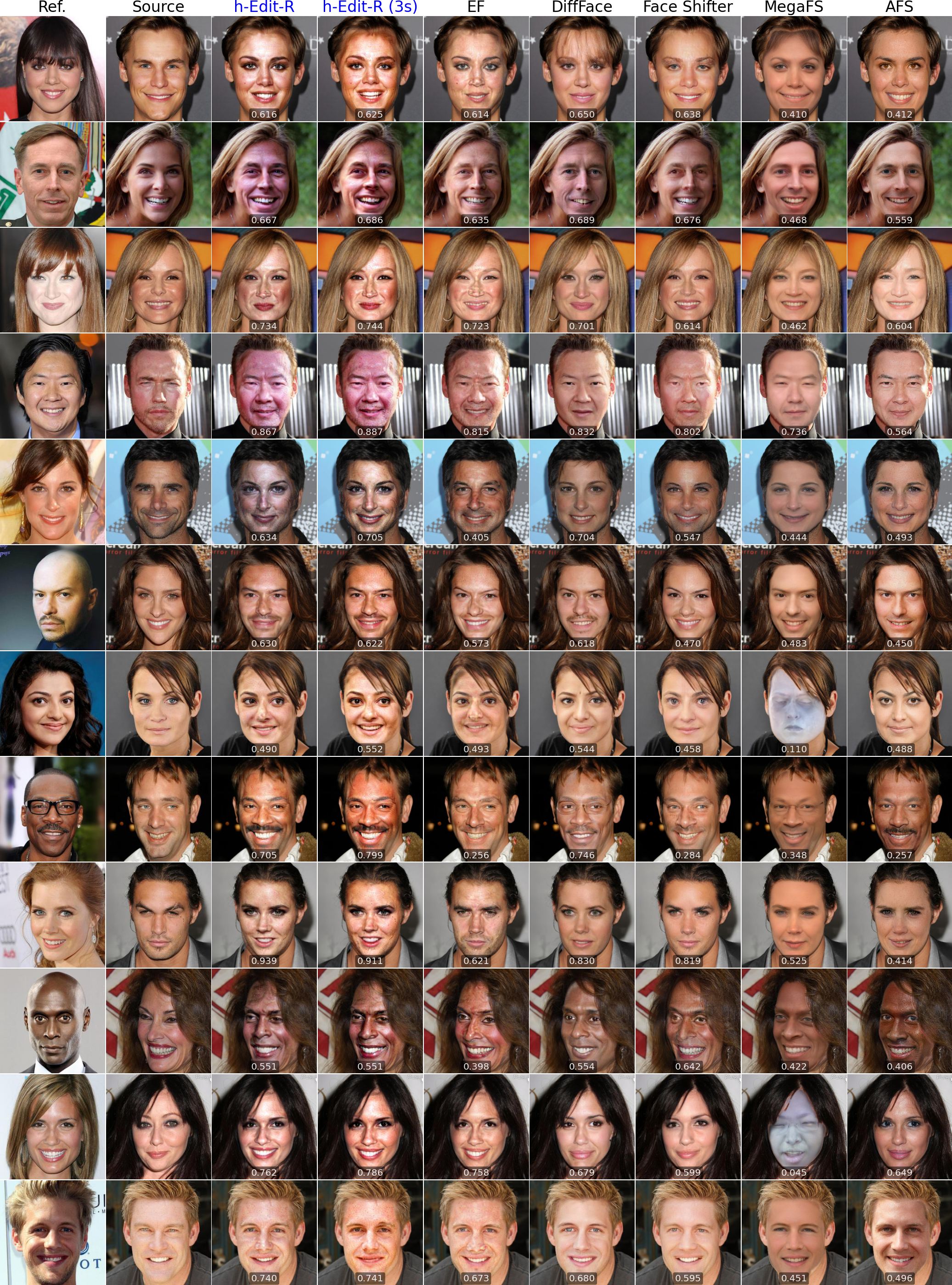}\tabularnewline
\end{tabular}
\par\end{centering}
\caption{Additional qualitative comparisons between our method and other face
swapping baselines. Identity similarity scores (higher is better)
computed using ArcFace with the ResNet34 backbone are displayed below
each output.\label{fig:Additional-qualitative-results-face-swapping}}
\end{figure*}

\subsection{Combined Text-guided and Style Editing\label{subsec:Additional_results_style_text_combined}}

\begin{figure}
\begin{centering}
\begin{minipage}[c][1\totalheight][t]{0.48\textwidth}%
\begin{center}
\includegraphics[width=0.5\textwidth]{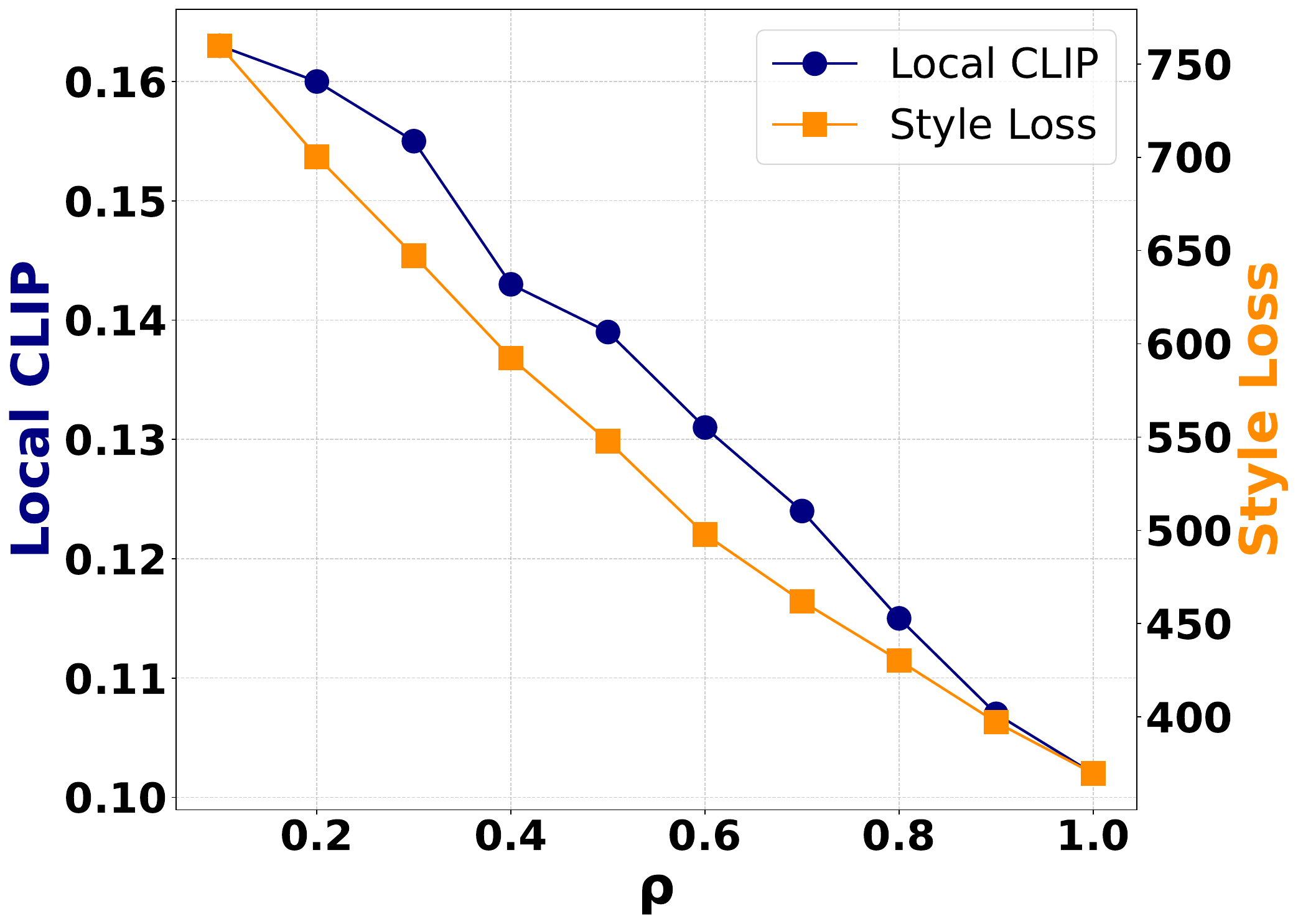}\includegraphics[width=0.5\textwidth]{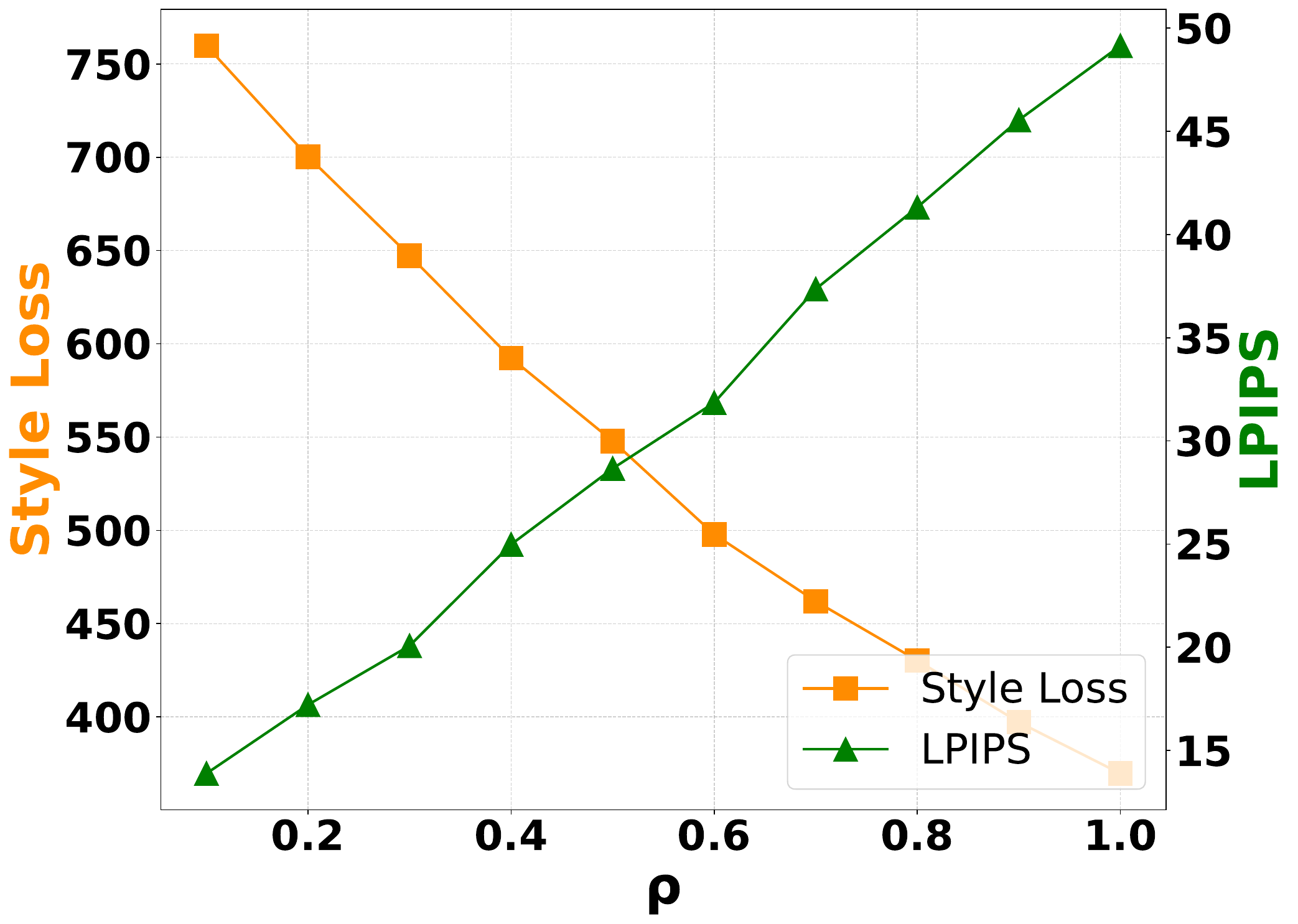}
\par\end{center}
\begin{center}
\subcaption{$h$-Edit-R + P2P}
\par\end{center}%
\end{minipage}\hspace*{0.02\textwidth}%
\begin{minipage}[c][1\totalheight][t]{0.48\textwidth}%
\begin{center}
\includegraphics[width=0.5\textwidth]{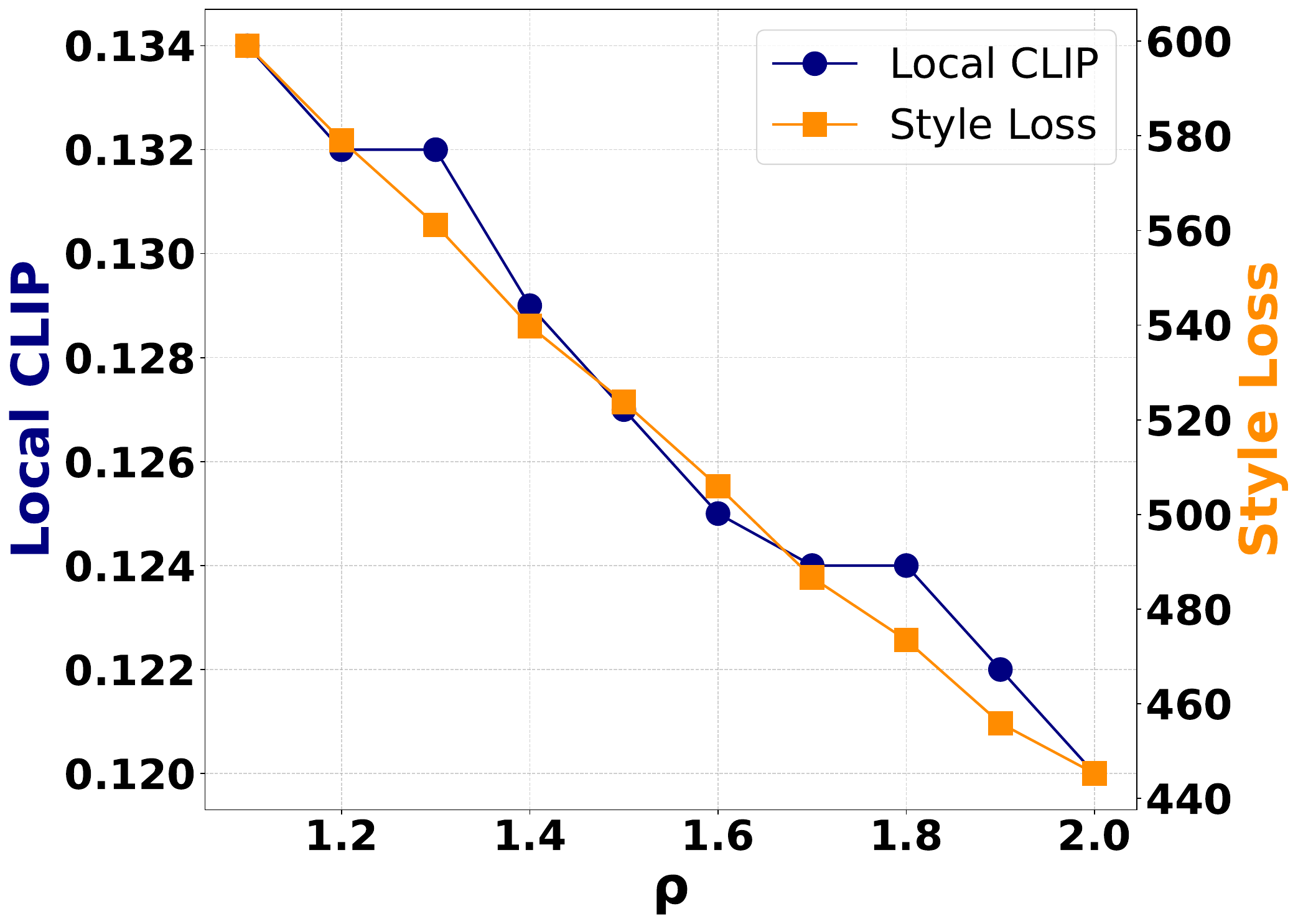}\includegraphics[width=0.5\textwidth]{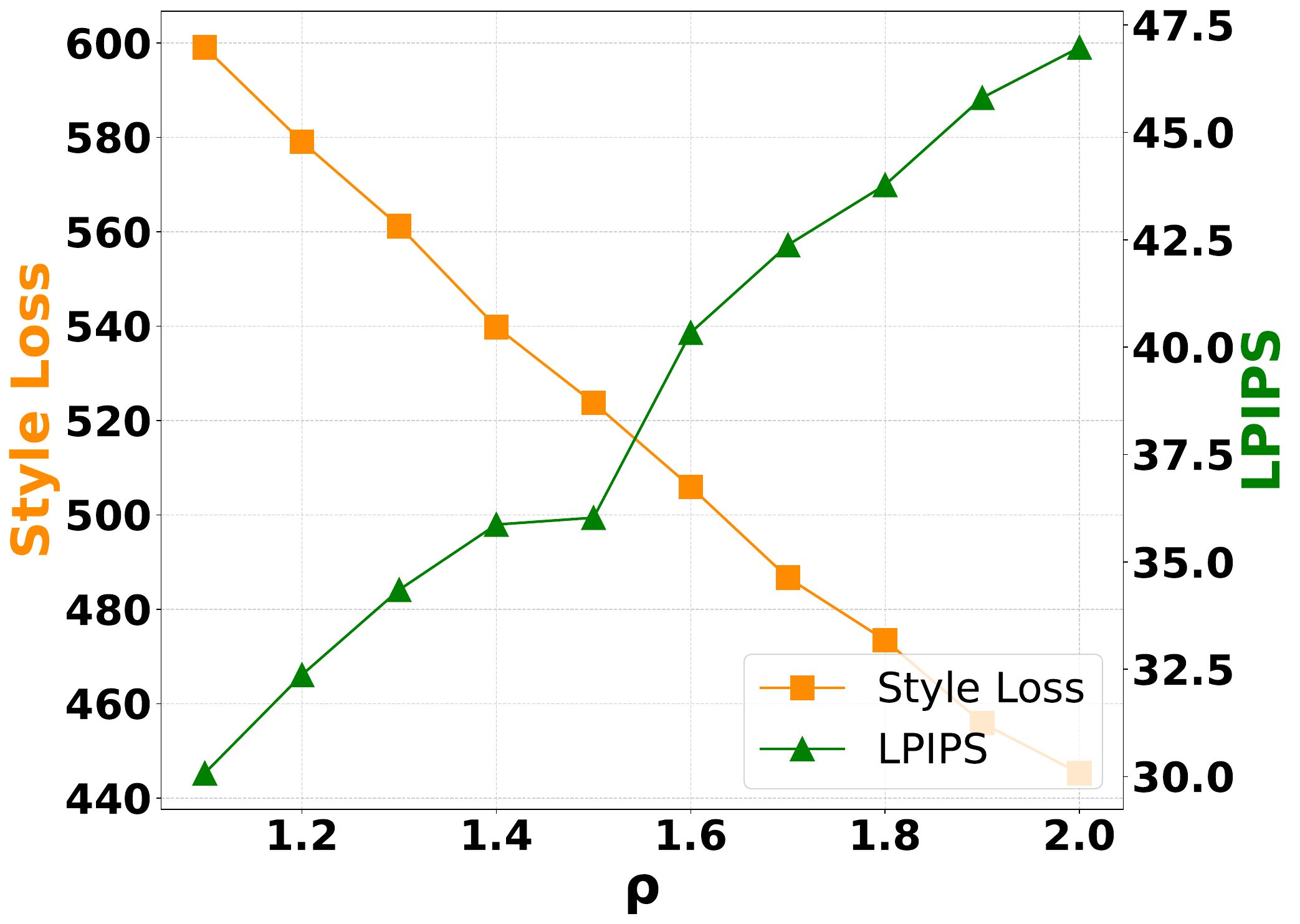}
\par\end{center}
\begin{center}
\subcaption{EF + P2P}
\par\end{center}%
\end{minipage}
\par\end{centering}
\caption{Changes in style loss, local CLIP similarity, and LPIPS score of $h$-Edit-R
+ P2P and EF + P2P when $\rho^{\protect\style}$ is varied from 0.1
to 1.0 for $h$-Edit-R + P2P and from 1.1 to 2.0 for EF + P2P.\label{fig:Quantitative-changes-in-style-coefficient}}
\end{figure}

\begin{figure}
\begin{centering}
\includegraphics[width=0.65\textwidth]{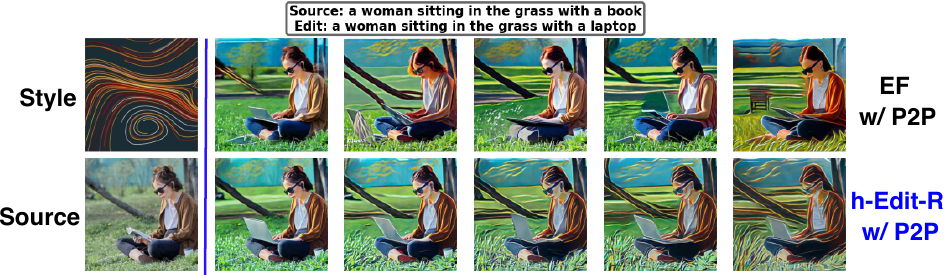}
\par\end{centering}
\caption{Visualizations of edited images with $\rho^{\protect\style}$ values
ranging from \{1.4, 1.5, 1.6, 1.7, 1.8\} for EF + P2P and \{0.4, 0.5,
0.6, 0.7, 0.8\} for $h$-Edit-R + P2P.\label{fig:Qualitative-changes-in-style-coefficient}}
\end{figure}

Fig.~\ref{fig:Quantitative-changes-in-style-coefficient} illustrates
the changes in style loss, local CLIP similarity, and LPIPS score
as the style editing coefficient $\rho^{\style}$ is varied from 0.1
to 1.0 for $h$-Edit-R + P2P and from 1.1 to 2.0 for EF + P2P. While
the ranges of $\rho^{\style}$ differ, the resulting style loss, local
CLIP similarity, and LPIPS score ranges are comparable, validating
the appropriateness of our parameter selection. Increasing $\rho^{\style}$
improves style transfer (lower style loss) but compromises text-guided
editing quality in terms of both effectiveness and faithfulness (lower
local CLIP similarity and higher LPIPS respectively). Since determining
the optimal value of $\rho^{\style}$ for achieving a balance between
style and text-guided editing is nontrivial, we identified candidate
values near the intersection of the style loss and LPIPS curves. Combining
this with visual inspection, we selected $\rho^{\style}$ value of
0.6 for $h$-Edit-R and 1.5 for EF.

Although EF exhibits similar quantitative trends to our method when
$\rho^{\style}$ is varied, its qualitative behavior is notably different.
As shown in Fig.~\ref{fig:Qualitative-changes-in-style-coefficient},
our method smoothly incorporates more style information into the edited
images while preserving their global structure as $\rho^{\style}$
increases. In contrast, EF often modifies the global structure of
the images to accommodate the increased $\rho^{\style}$. This advantage
of our approach likely stems from the natural decomposition of the
update into reconstruction and editing terms (Eq.~\ref{eq:implicit_LMC_3}),
enabling style edits to be added to the text-guided editing term with
minimal impact on the reconstruction term. EF, on the other hand,
lacks such a decomposition, meaning the introduction of the style
editing term directly affects reconstruction. These findings highlight
the limitations of relying solely on quantitative metrics to compare
our method with EF, as they may fail to capture important qualitative
differences.

In Fig.~\ref{fig:Additional-qualitative-results-combined-editing},
we present addition visualizations comparing $h$-Edit-R + P2P and
EF + P2P, with $\rho^{\style}$ set to the optimal values for each
method. The results clearly demonstrate that our method combined with
P2P surpasses EF + P2P in both style transfer and text-guided editing,
achieving superior quality and consistency. 

\begin{figure*}
\begin{centering}
\includegraphics[width=0.98\textwidth]{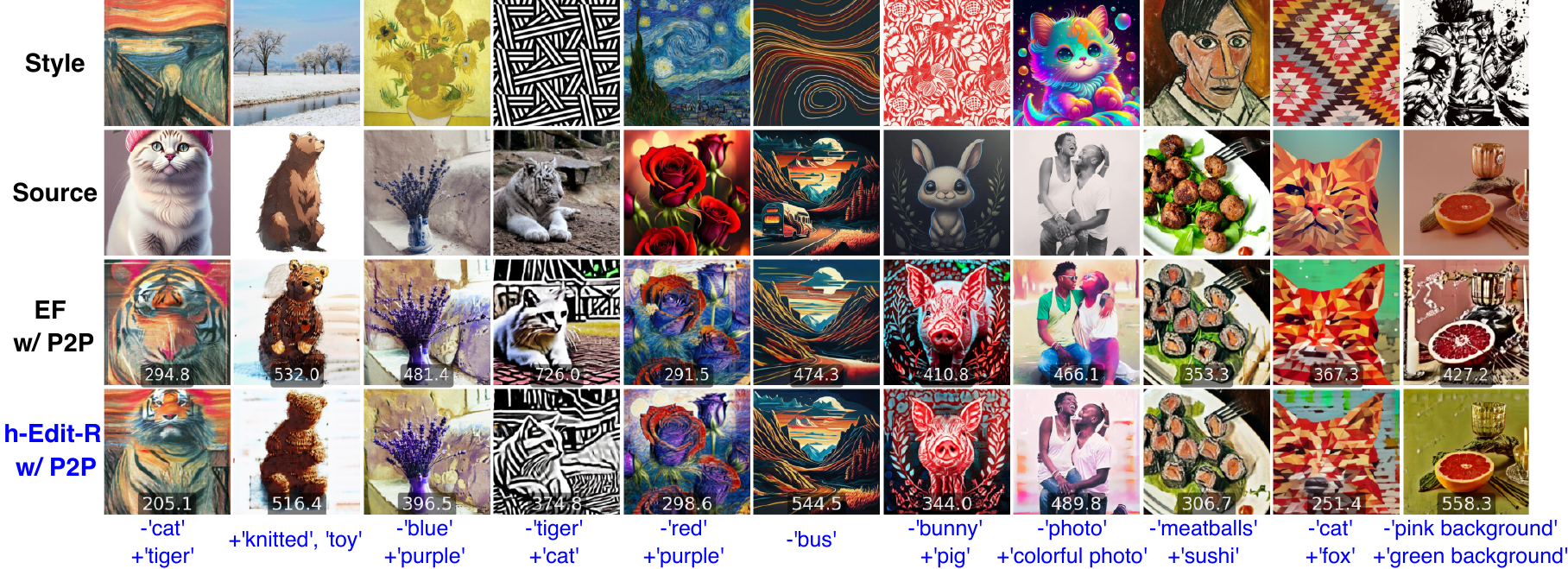}
\par\end{centering}
\caption{Additional qualitative results of $h$-Edit-R + P2P and EF + P2P for
the combined style and text-guided editing task. Style loss values
(lower is better) are displayed below each output image.\label{fig:Additional-qualitative-results-combined-editing}}
\end{figure*}

\subsection{Results when Combining with MasaCtrl and Plug-and-Play}

In this section, we compare the performance of $h$-Edit with other
baselines when combined with MasaCtrl \cite{cao2023masactrl} and
Plug-and-Play (PnP) \cite{tumanyan2023plug}. For MasaCtrl, we adopted
the \href{https://github.com/cure-lab/PnPInversion}{implementation}
from the PnP Inversion paper \cite{ju2023direct} which omits the
source prompt during editing. We observed that this approach yields
more stable results compared to using the source prompt. Since editing
methods like NT, NP and NMG are incompatible with this setting, they
were excluded in our experiments with MasaCtrl.

As shown in Table~\ref{tab:Results-using-other-attn-control}, both
$h$-Edit-R and $h$-Edit-D significantly outperform EF and deterministic-inversion-based
baselines when combined with either MasaCtrl or PnP. For example,
with PnP, $h$-Edit-D and $h$-Edit-R surpass NT and EF by 0.014 and
0.029 on the local directional CLIP metric, while achieving about
0.70 and 0.90 lower LPIPS scores, respectively. It is also evident
that PnP is a more effective attention control method than MasaCtrl
on the PIE-Bench dataset. However, both PnP and MasaCtrl are less
effective and stable than P2P \cite{hertz2023prompttoprompt}, as
indicated by our quantitative results in Tables~\ref{tab:Text-Guided-Editing-Results}
and \ref{tab:Results-using-other-attn-control}, and through our observations.

\begin{table}[H]
\begin{centering}
{\resizebox{0.98\textwidth}{!}{
\begin{tabular}{cccc>{\centering}p{0.12\textwidth}cccc}
\hline 
\multirow{1}{*}{\textbf{Attn.}} & \multirow{1}{*}{\textbf{Inv.}} & \multirow{1}{*}{\textbf{Method}} & \textbf{\textcolor{black}{CLIP Sim.}}$\uparrow$ & \multirow{1}{0.12\textwidth}{\textbf{\textcolor{black}{Local CLIP}}\textcolor{black}{$\uparrow$}} & \multirow{1}{*}{\textbf{\textcolor{black}{DINO Dist.}}\textcolor{black}{$_{\times10^{2}}$$\downarrow$}} & \multirow{1}{*}{\textbf{LPIPS}$_{\times10^{2}}$$\downarrow$} & \multirow{1}{*}{\textbf{SSIM}$_{\times10}$$\uparrow$} & \multirow{1}{*}{\textbf{PSNR}$\uparrow$}\tabularnewline
\hline 
\hline 
\multirow{4}{*}{MasaCtrl} & \multirow{2}{*}{Deter.} & PnP Inv & \textbf{0.243} & 0.068 & 2.47 & 8.79 & 8.13 & 22.64\tabularnewline
 &  & $\Model$-D & \textbf{0.243} & \textbf{0.071} & \textbf{2.38} & \textbf{8.62} & \textbf{8.16} & \textbf{22.85}\tabularnewline
\cline{2-9} \cline{3-9} \cline{4-9} \cline{5-9} \cline{6-9} \cline{7-9} \cline{8-9} \cline{9-9} 
 & \multirow{2}{*}{Random} & EF & 0.241 & 0.059 & 2.75 & 8.57 & 8.15 & 22.49\tabularnewline
 &  & $\Model$-R & \textbf{0.242} & \textbf{0.065} & \textbf{2.46} & \textbf{8.42} & \textbf{8.18} & \textbf{22.68}\tabularnewline
\hline 
\hline 
\multirow{7}{*}{PnP} & \multirow{5}{*}{Deter.} & NP & 0.250 & \uline{0.152} & 1.84 & 8.55 & 8.19 & 25.05\tabularnewline
 &  & NT & 0.251 & 0.144 & \uline{1.58} & \uline{7.94} & \uline{8.24} & \uline{25.53}\tabularnewline
 &  & NMG & \uline{0.253} & 0.101 & 2.08 & 9.96 & 8.02 & 23.20\tabularnewline
 &  & PnP Inv & \uline{0.253} & \centering{}0.109 & 1.75 & 9.29 & 8.15 & 24.18\tabularnewline
 &  & $h$-Edit-D & \textbf{0.254} & \textbf{0.158} & \textbf{1.51} & \textbf{7.28} & \textbf{8.33} & \textbf{25.68}\tabularnewline
\cline{2-9} \cline{3-9} \cline{4-9} \cline{5-9} \cline{6-9} \cline{7-9} \cline{8-9} \cline{9-9} 
 & \multirow{2}{*}{Random} & EF & 0.253 & 0.118 & 1.48 & 6.87 & 8.32 & 24.77\tabularnewline
 &  & $h$-Edit-R & \textbf{0.255} & \textbf{0.147} & \textbf{1.39} & \textbf{5.97} & \textbf{8.43} & \textbf{25.75}\tabularnewline
\hline 
\end{tabular}{\small{}}}}{\small\par}
\par\end{centering}
\caption{Text-guided editing results with MasaCtrl \cite{cao2023masactrl}
and Plug-n-Play \cite{tumanyan2023plug} on PIE-Bench. $h$-Edit significantly
outperforms other baselines in all metrics.\label{tab:Results-using-other-attn-control}}
\vspace{-5.0mm}
\end{table}

\section{Ablation Studies\label{sec:Additional-Ablation-Studies}}

\subsection{Impact of $\texorpdfstring{\hat{w}^{\text{orig}}}{w^{orig}}$ \label{subsec:Impact-of-hat_w_orig}}

\begin{table}[H]
\begin{centering}
{\resizebox{0.9\textwidth}{!}{
\begin{tabular}{ccccccc}
\hline 
\multirow{1}{*}{\textbf{$\hat{w}^{\orig}$}} & \multicolumn{1}{c}{\textbf{\textcolor{black}{CLIP Sim.}}\textcolor{black}{$\uparrow$}} & \multirow{1}{*}{\textbf{Local CLIP}$\uparrow$} & \multirow{1}{*}{\textbf{DINO Dist.}$_{\times10^{2}}$$\downarrow$} & \multirow{1}{*}{\textbf{LPIPS}$_{\times10^{2}}$$\downarrow$} & \multirow{1}{*}{\textbf{SSIM}$_{\times10}$$\uparrow$} & \multirow{1}{*}{\textbf{PNSR}$\uparrow$}\tabularnewline
\hline 
1.0 & \textbf{0.256} & 0.118 & 1.64 & 6.00 & 8.38 & 25.75\tabularnewline
3.0 & 0.255 & 0.137 & 1.52 & 5.49 & 8.44 & 26.36\tabularnewline
5.0 & \textbf{0.256} & 0.159 & \textbf{1.45} & \textbf{5.08} & 8.50 & \textbf{26.97}\tabularnewline
7.0 & 0.254 & \textbf{0.173} & 1.60 & 5.22 & \textbf{8.51} & 26.94\tabularnewline
9.0 & 0.241 & 0.172 & 2.30 & 6.44 & 8.40 & 26.03\tabularnewline
\hline 
\end{tabular}{\small{}}}}{\small\par}
\par\end{centering}
\caption{Quantitative results of $h$-Edit-R + P2P when varying $\hat{w}^{\protect\orig}$
from 1.0 to 9.0 while keeping $w^{\protect\edit}$ and $w^{\protect\orig}$
fixed at $7.5$ and $1.0$, respectively. The best value for each
metric is highlighted in bold.\label{tab:Quantitative-results-of-ablation-studies-hat-w-orig}}
\end{table}
\begin{figure*}
\begin{centering}
\begin{tabular}{cc}
\includegraphics[width=0.49\textwidth]{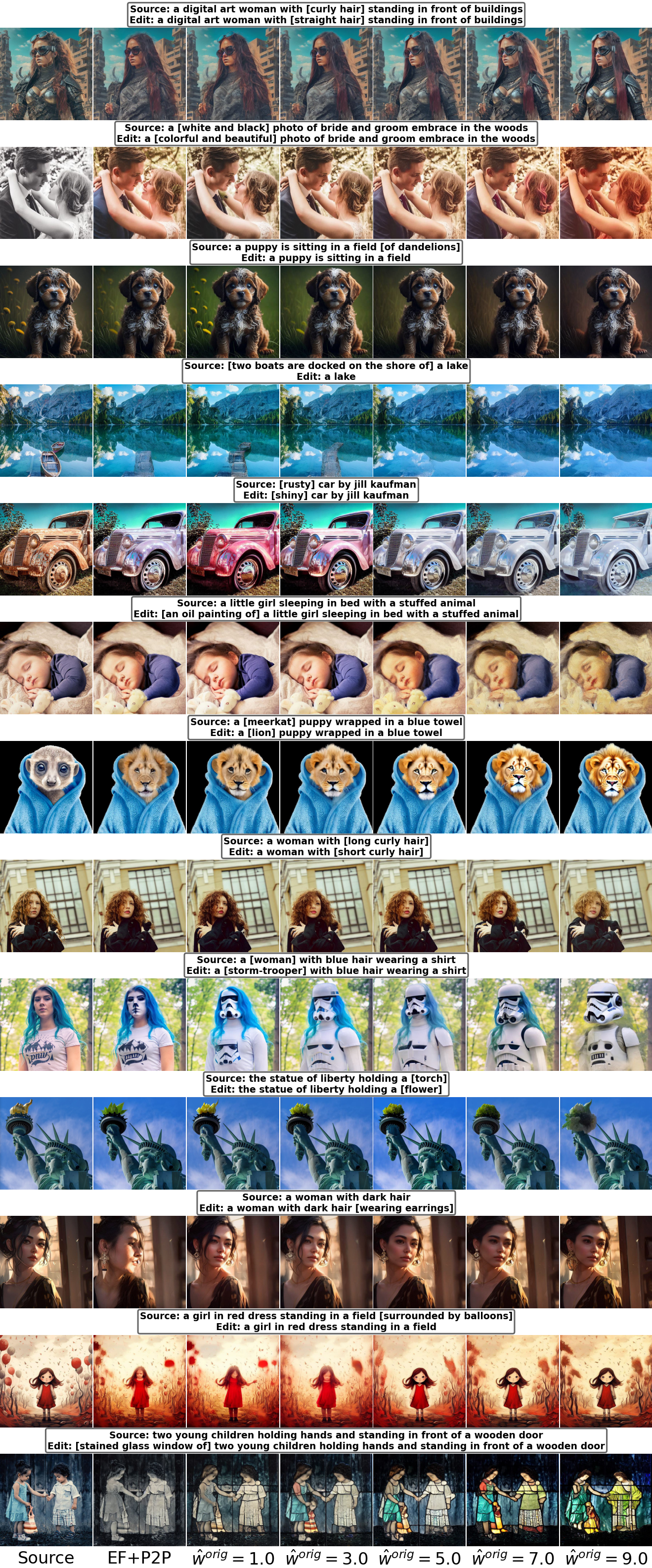} & \includegraphics[width=0.49\textwidth]{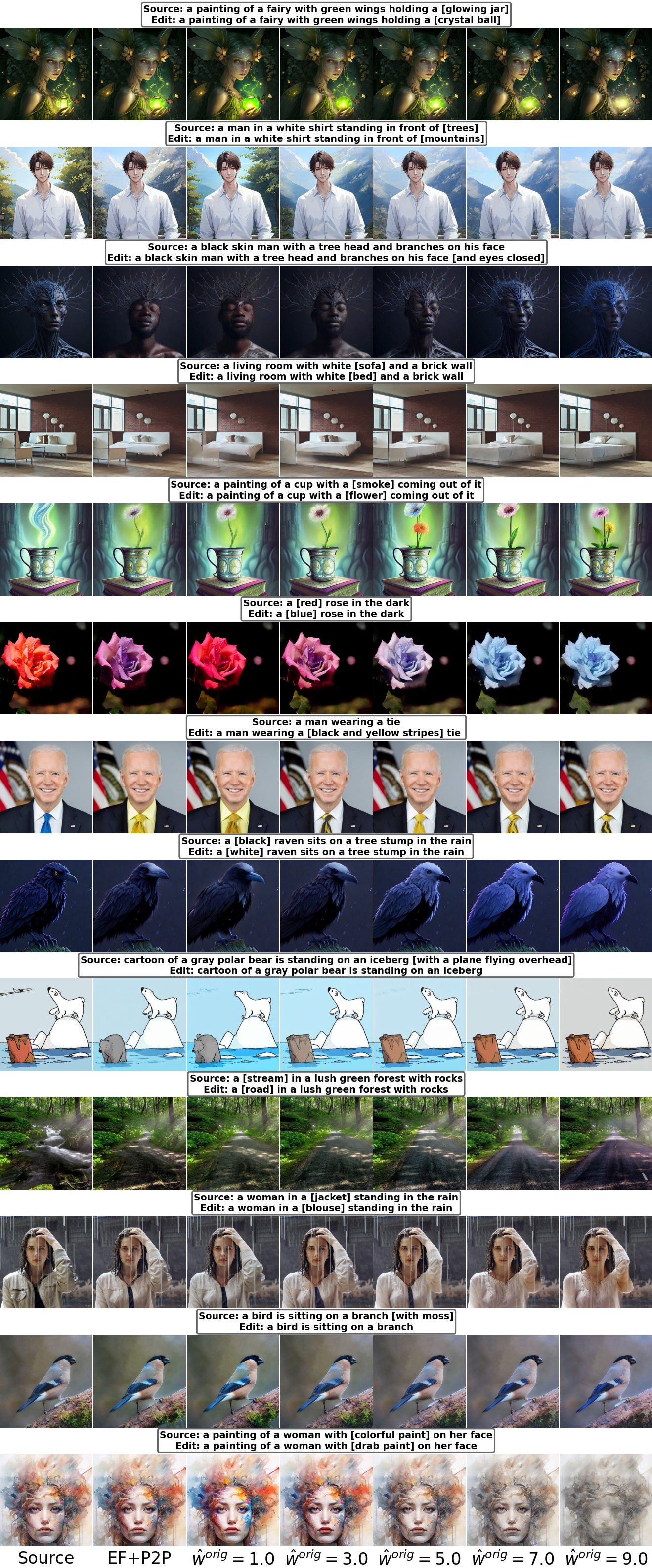}\tabularnewline
\end{tabular}
\par\end{centering}
\caption{Qualitative results of $h$-Edit-R + P2P when varying $\hat{w}^{\protect\orig}$
from 1.0 to 9.0 while keeping $w^{\protect\edit}$ and $w^{\protect\orig}$
fixed at $7.5$ and $1.0$, respectively. Increasing $\hat{w}^{\protect\orig}$
to an appropriate value improves both editing accuracy and fidelity.\label{fig:Visualization-hat-w-orig}}
\end{figure*}

In this section, we study the impact of $\hat{w}^{\orig}$ in Eq.~\ref{eq:hedit_SD_3}
by varying its value within $\left\{ 1.0,3.0,5.0,7.0,9.0\right\} $
while keeping $w^{\orig}=1$ and $w^{\edit}=7.5$ fixed for $h$-Edit-R
+ P2P. Quantitative and qualitative results are shown in Table~\ref{tab:Quantitative-results-of-ablation-studies-hat-w-orig}
and Fig.~\ref{fig:Visualization-hat-w-orig}, respectively. The results
indicate that increasing $\hat{w}^{\orig}$ to a suitable value enhances
both editing accuracy and fidelity. For example, raising $\hat{w}^{\orig}$
from 1.0 to 7.0 restores the woman's armor suit in the first row on
the left of Fig.~\ref{fig:Visualization-hat-w-orig} while also straightening
her hair. Similarly, it effectively removes the balloons in the background
while preserving the original appearance of the girl in a red dress
in the twelfth row on the left. As discussed in Section~\ref{subsec:Closed-form-expressions-for-explicit-and-implicit},
$\hat{w}^{\orig}$ controls how much of the original image\textquoteright s
information is excluded during editing. Larger values of $\hat{w}^{\orig}$
helps isolate essential information, enabling precise localization
of edits. However, excessively high values (i.e., exceeding $w^{\edit}$)
may degrade reconstruction quality by removing too much original information.
This is evident in the case of changing colorful paint to drab paint
in the last row on the right. These observations suggest that the
optimal value of $\hat{w}^{\orig}$ is case-dependent, for $w^{\edit}=7.5$,
we found $\hat{w}^{\orig}=5.0$ achieves the best balance between
editing effectiveness and faithfulness.

\subsection{Impact of $w^{\protect\edit}$\label{subsec:Impact-of-w-edit}}

\begin{figure*}
\begin{centering}
\includegraphics[width=0.52\textwidth]{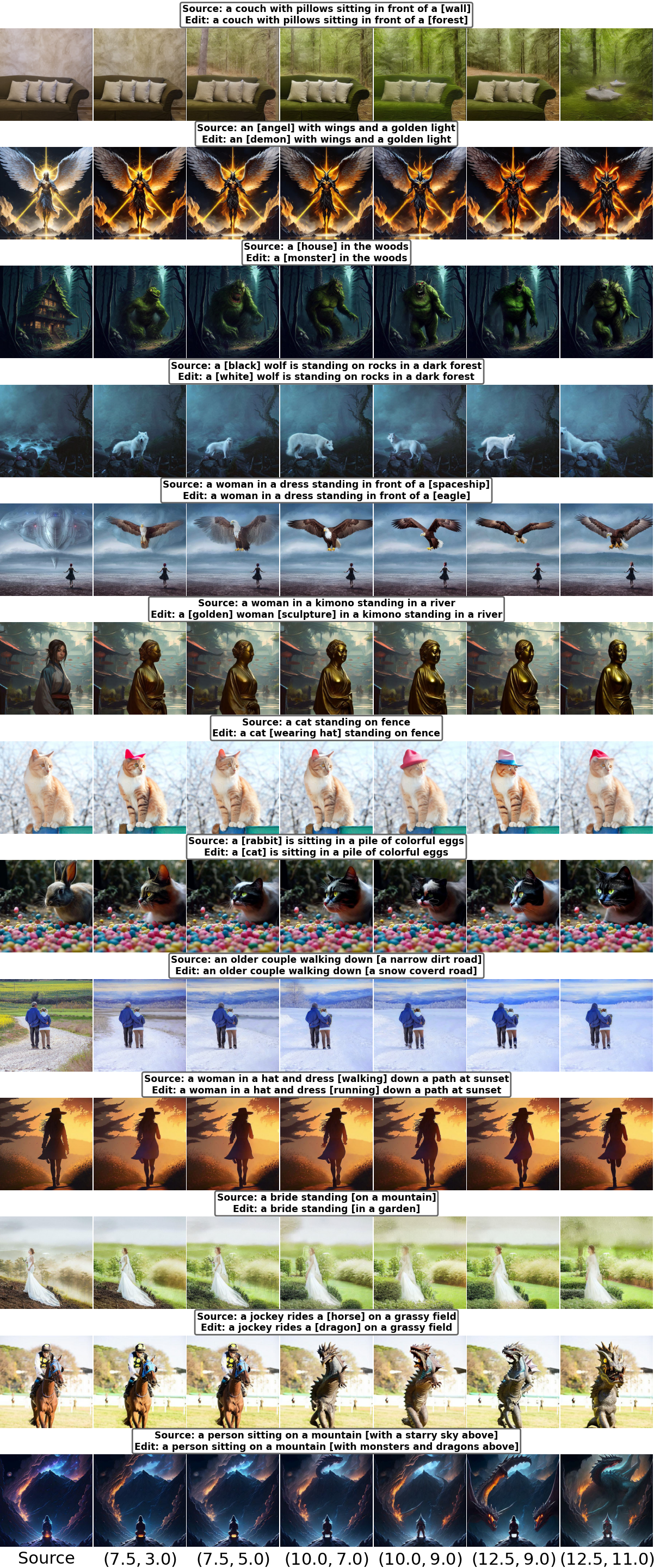}
\par\end{centering}
\caption{Qualitative results of $h$-Edit-R + P2P when varying $\left(w^{\protect\edit},\hat{w}^{\protect\orig}\right)$
within \{$\left(7.5,3.0\right)$, $\left(7.5,5.0\right)$, $\left(10.0,7.0\right)$,
$\left(10.0,9.0\right)$, $\left(12.5,9.0\right)$ $\left(12.5,11.0\right)$\}.
Higher $w^{\protect\edit}$ values effectively handle challenging
edits but may compromise reconstruction quality. \label{fig:Qualitative-results-ablation-study-w-edit}}
\end{figure*}

We investigate the influence of $w^{\edit}$ in Eq.~\ref{eq:hedit_SD_3}
for $h$-Edit-R + P2P by analyzing edited images across different
$\left(w^{\edit},\hat{w}^{\orig}\right)$ pairs: $\left\{ \left(7.5,3.0\right),\left(7.5,5.0\right),\left(10.0,7.0\right),\left(10.0,9.0\right),\left(12.5,9.0\right),\left(12.5,11.0\right)\right\} $.
Qualitative results are provided in Fig.~\ref{fig:Qualitative-results-ablation-study-w-edit}.
In general, higher $w^{\edit}$ values enhance editing effectiveness,
allowing to handle difficult edits. For example, increasing $w^{\edit}$
from 7.5 to 12.5 successfully introduces dragons to the images in
the final row of Fig.~\ref{fig:Qualitative-results-ablation-study-w-edit}.
However, higher $w^{\edit}$ can degrade reconstruction quality in
non-edited regions, requiring a proportional increase in $\hat{w}^{\orig}$
to mitigate this effect. Even so, this approach may not succeed in
all scenarios. We can overcome this issue by using multiple optimization
steps (available for implicit $h$-Edit). This technique progressively
refines edits via applying the score function iteratively, effectively
handling challenging cases while maintaining good reconstruction.

\subsection{Impact of multiple optimization steps in implicit $\protect\Model$\label{subsec:Impact-of-multiple-loops}}

\begin{figure*}
\begin{centering}
\begin{tabular}{cc}
\includegraphics[height=0.9\textheight]{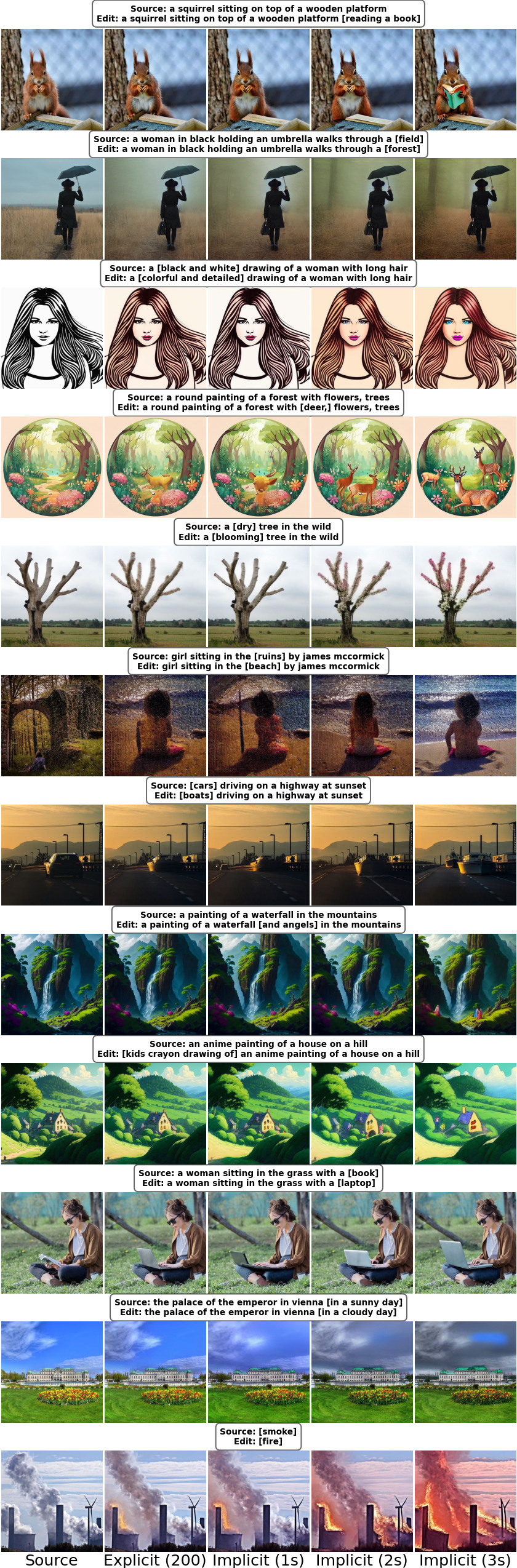} & \includegraphics[height=0.9\textheight]{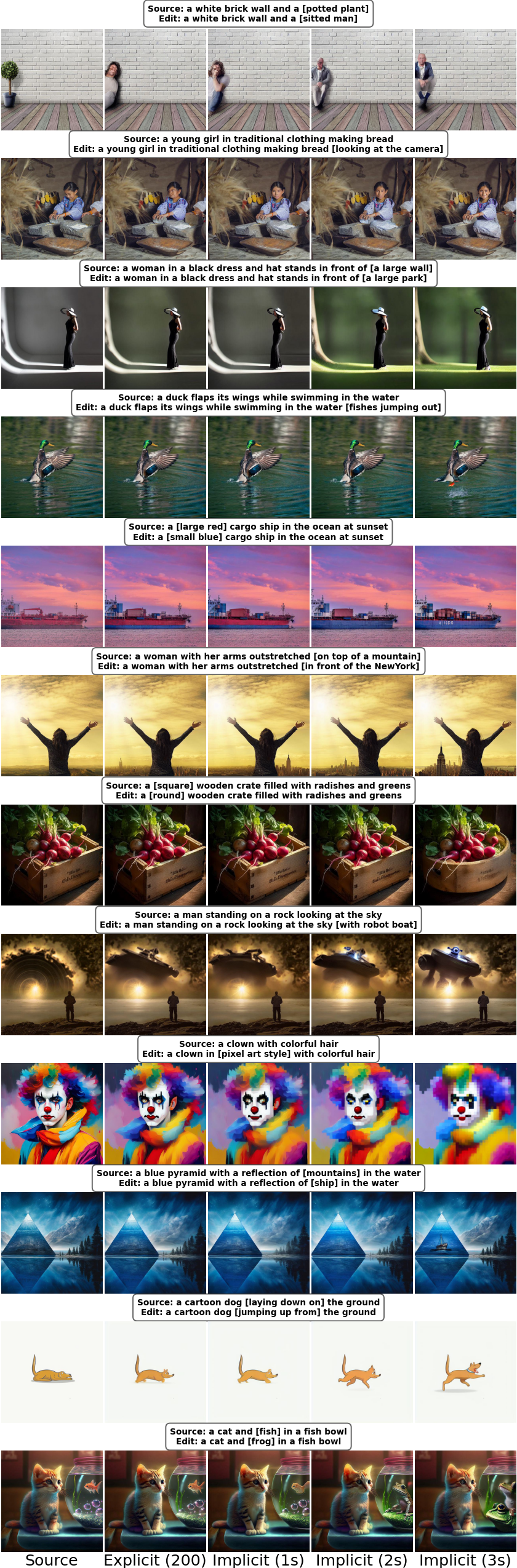}\tabularnewline
 & \tabularnewline
\end{tabular}
\par\end{centering}
\caption{Qualitative examples of implicit $\protect\Model$-R + P2P with 50
sampling steps using one, two and three optimization steps (1s/2s/3s),
compared to its explicit counterpart with 200 sampling steps. More
optimization steps effectively handle challenging cases, outperforming
increased sampling steps in the explicit form. \label{fig:qualitative_multiple_optimization_steps}}
\end{figure*}

Fig.~\ref{fig:qualitative_multiple_optimization_steps} highlights
the advantage of the implicit version of $h$-Edit when utilizing
multiple optimization steps. Increasing the number of optimization
steps significantly enhances editing accuracy while maintaining minimal
degradation in reconstruction quality. This capability is unique to
the implicit version and cannot be replicated by simply increasing
the number of sampling steps. For instance, the explicit version,
even with 200 sampling steps, performs only comparably or slightly
better than the default implicit version with 50 sampling steps and
one optimization step, yet it falls notably short compared to the
implicit version with three optimization steps.

Additionally, the effectiveness of multiple optimization steps is
evident in the face swapping task, where $\Model$-R with three optimization
steps outperforms its one-step counterpart, as presented in Section~\ref{subsec:Additional_results_face_swapping}.

\subsection{Comparison between explicit and implicit versions\label{subsec:Explicit-implicit-comparison}}

\begin{table}[H]
\begin{centering}
{\resizebox{0.95\textwidth}{!}{%
\par\end{centering}
\begin{centering}
\begin{tabular}{cccc>{\centering}p{0.12\textwidth}cccc}
\hline 
\multirow{2}{*}{\textbf{Attn.}} & \multirow{2}{*}{\textbf{\textcolor{black}{Steps}}} & \multirow{2}{*}{\textbf{\textcolor{black}{Method}}} & \multirow{2}{*}{\textbf{\textcolor{black}{CLIP Sim.}}} & \multirow{2}{0.12\textwidth}{\textbf{\textcolor{black}{Local CLIP}}\textcolor{black}{$\uparrow$}} & \multirow{2}{*}{\textbf{\textcolor{black}{DINO Dist.}}\textcolor{black}{$_{\times10^{2}}$$\downarrow$}} & \multirow{2}{*}{\textbf{\textcolor{black}{LPIPS}}\textcolor{black}{$_{\times10^{2}}$$\downarrow$}} & \multirow{2}{*}{\textbf{\textcolor{black}{SSIM}}\textcolor{black}{$_{\times10}$$\uparrow$}} & \multirow{2}{*}{\textbf{\textcolor{black}{PSNR}}\textcolor{black}{$\uparrow$}}\tabularnewline
 &  &  &  &  &  &  &  & \tabularnewline
\hline 
\hline 
\multirow{4}{*}{None} & \multirow{2}{*}{\textcolor{black}{25}} & \textcolor{black}{$h$-Edit-R (ex)} & 0.252 & 0.139 & 1.10 & 5.10 & 8.49 & 26.79\tabularnewline
 &  & \textcolor{black}{$h$-Edit-R (im)} & 0.255 & 0.148 & 1.39 & 5.98 & 8.41 & 25.77\tabularnewline
\cline{2-9} \cline{3-9} \cline{4-9} \cline{5-9} \cline{6-9} \cline{7-9} \cline{8-9} \cline{9-9} 
 & \multirow{2}{*}{\textcolor{black}{50}} & \textcolor{black}{$h$-Edit-R (ex)} & 0.253 & 0.141 & 1.10 & 5.07 & 8.51 & 27.00\tabularnewline
 &  & \textcolor{black}{$h$-Edit-R (im)} & 0.255 & 0.148 & 1.28 & 5.55 & 8.46 & 26.43\tabularnewline
\hline 
\multirow{4}{*}{P2P} & \multirow{2}{*}{\textcolor{black}{25}} & \textcolor{black}{$h$-Edit-R (ex)} & 0.254 & 0.153 & 1.38 & 5.04 & 8.50 & 26.81\tabularnewline
 &  & \textcolor{black}{$h$-Edit-R (im)} & 0.255 & 0.150 & 1.38 & 5.03 & 8.50 & 26.88\tabularnewline
\cline{2-9} \cline{3-9} \cline{4-9} \cline{5-9} \cline{6-9} \cline{7-9} \cline{8-9} \cline{9-9} 
 & \multirow{2}{*}{\textcolor{black}{50}} & \textcolor{black}{$h$-Edit-R (ex)} & 0.256 & 0.158 & 1.47 & 5.10 & 8.50 & 26.85\tabularnewline
 &  & \textcolor{black}{$h$-Edit-R (im)} & 0.256 & 0.159 & 1.45 & 5.08 & 8.50 & 26.97\tabularnewline
\hline 
\end{tabular}{\small{}}}}{\small\par}
\par\end{centering}
\caption{Quantitative comparison of $\protect\Model$-R implicit and explicit
forms, with and without P2P, evaluated over 25 and 50 sampling steps.\label{tab:Quantitative-implicit-and-explicit}}
\end{table}

\begin{figure}
\begin{centering}
\includegraphics[width=0.98\textwidth]{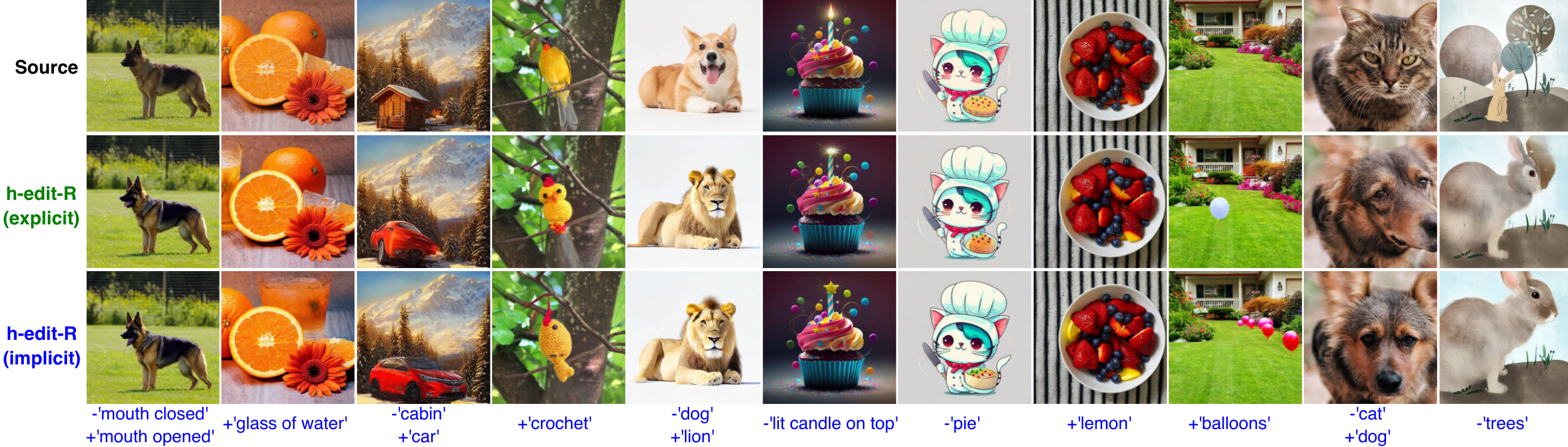}
\par\end{centering}
\caption{Qualitative visualizations comparing the explicit and implicit versions
of $h$-Edit-R with 25 sampling steps.\label{fig:Qualitative-implicit-explicit-h-edit-R}}
\end{figure}

\begin{figure}
\begin{centering}
\includegraphics[width=0.98\textwidth]{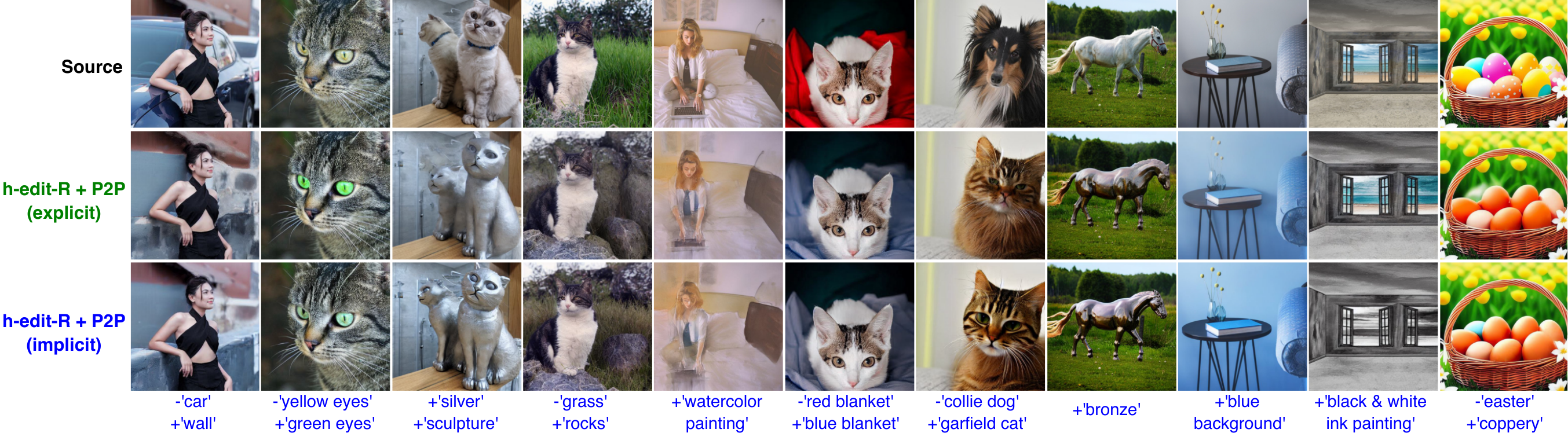}
\par\end{centering}
\caption{Qualitative visualizations comparing the explicit and implicit versions
of $h$-Edit-R + P2P with 25 sampling steps.\label{fig:Qualitative-implicit-explicit-h-edit-R-P2P}}
\end{figure}

In this section, we compare the explicit and implicit versions of
$h$-Edit-R with and without P2P, using either 25 or 50 sampling steps.
Without P2P, the implicit version generally performs more accurate
edits than the explicit counterpart, though the results vary by case,
as shown in Table~\ref{tab:Quantitative-implicit-and-explicit} and
Fig.~\ref{fig:Qualitative-implicit-explicit-h-edit-R}. However,
when combined with P2P, the two versions perform comparably. Instances
where implicit $h$-Edit-R outperforms the explicit version, and vice
versa, are illustrated in Fig.~\ref{fig:Qualitative-implicit-explicit-h-edit-R-P2P}.
Our preference for the implicit version as the default is not primarily
due to its performance relative to the explicit version but rather
its ability to support multiple optimization steps, which offers greater
flexibility.

\subsection{Face swapping without masks\label{subsec:Face-swapping-without-masks}}

We demonstrate that our $\Model$-R method can perform face swapping
without relying on mask postprocessing techniques for reconstruction,
with qualitative results of $\Model$-R (3s) shown in Fig.~\ref{fig:Qualitative_results_face_wo_masks}.
$\Model$-R without masks achieves near-perfect faithful reconstruction,
with minor background changes. For instance, in the third row (left),
it preserves background text, while in more complex backgrounds, such
as dense text (last row, right) or intricate shirt patterns (last
row, left), it maintains individual features with slight background
blurring. This capability is unique to our method, as state-of-the-art
approaches like DiffFace and FaceShifter rely on masks for faithful
reconstruction. These findings suggest that in scenarios where masks
are unavailable, our method is a robust choice for face editing with
minimal reconstruction error.

\begin{figure}
\begin{centering}
\begin{tabular}{cc}
\includegraphics[width=0.4\textwidth]{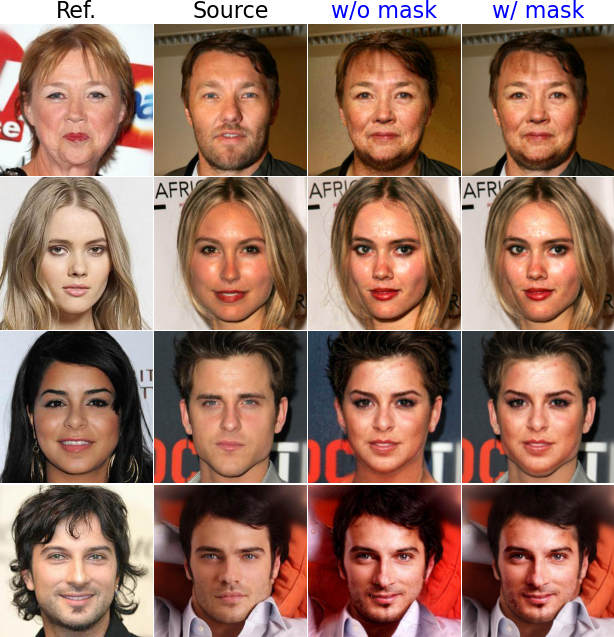} & \includegraphics[width=0.4\textwidth]{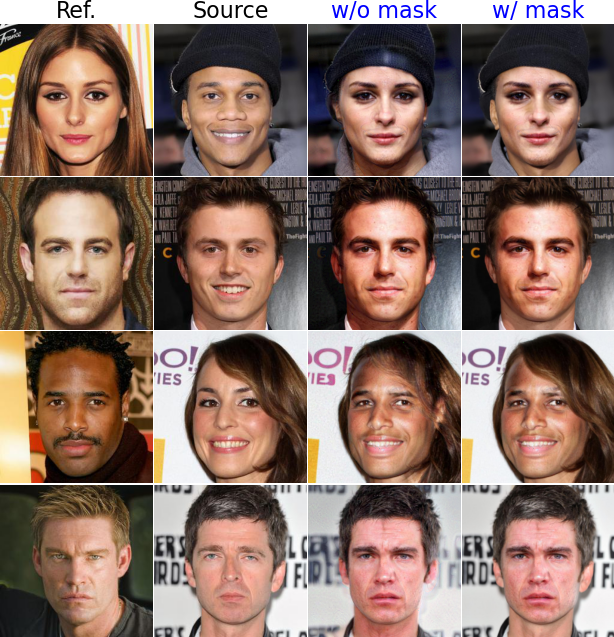}\tabularnewline
\end{tabular}
\par\end{centering}
\caption{Swapped faces generated by $h$-Edit-R (3s) with and without masks.\label{fig:Qualitative_results_face_wo_masks}}
\end{figure}

\subsection{Running time\label{subsec:Running-time}}

\begin{table}
\begin{centering}
\begin{minipage}[c][1\totalheight][t]{0.34\textwidth}%
\begin{center}
{\resizebox{\textwidth}{!}{
\begin{tabular}{cccc}
\hline 
\multirow{1}{*}{\textbf{Inv.}} & \multirow{1}{*}{\textbf{Attn.}} & \multirow{1}{*}{\textbf{Method}} & \multirow{1}{*}{\textbf{Time} \textbf{(s)}$\downarrow$}\tabularnewline
\hline 
\hline 
\multirow{6}{*}{Deter.} & \multirow{6}{*}{P2P} & NP & \textbf{21.68}\tabularnewline
 &  & NT & 186.84\tabularnewline
 &  & StyleD & 467.16\tabularnewline
 &  & NMG & \uline{35.67}\tabularnewline
 &  & PnP Inv & 37.65\tabularnewline
 &  & $h$-Edit-D & 48.63\tabularnewline
\hline 
\hline 
\multirow{5}{*}{Random} & \multirow{3}{*}{None} & EF & \uline{23.20}\tabularnewline
 &  & LEDITS++ & \textbf{18.31}\tabularnewline
 &  & $h$-Edit-R & 33.07\tabularnewline
\cline{2-4} \cline{3-4} \cline{4-4} 
 & \multirow{2}{*}{P2P} & EF & \textbf{32.95}\tabularnewline
 &  & $h$-Edit-R & 50.21\tabularnewline
\hline 
\end{tabular}{\small{}}}}\subcaption{Editing time for text-guided editing methods}
\par\end{center}%
\end{minipage}\hspace*{0.03\textwidth}%
\begin{minipage}[c][1\totalheight][t]{0.36\textwidth}%
\begin{center}
{\resizebox{0.55\textwidth}{!}{
\begin{tabular}{cc}
\toprule 
\textbf{Method} & \textbf{Time} \textbf{(s)}$\downarrow$\tabularnewline
\midrule
\midrule 
FaceShifter & \uline{1.31}\tabularnewline
MegaFS & 2.29\tabularnewline
AFS & \textbf{1.03}\tabularnewline
DiffFace & 46.42\tabularnewline
EF & 26.11\tabularnewline
\midrule
\midrule 
$h$-edit-R & 26.34\tabularnewline
$h$-edit-R (3s) & 51.36\tabularnewline
\bottomrule
\end{tabular}{\small{}}}}{\small\par}
\par\end{center}
\begin{center}
\subcaption{Editing time for face swapping methods}\vspace{0.2em}
\par\end{center}
\begin{center}
{\resizebox{0.81\textwidth}{!}{
\begin{tabular}{cccc}
\hline 
\multirow{1}{*}{\textbf{Inv.}} & \multirow{1}{*}{\textbf{Attn.}} & \multirow{1}{*}{\textbf{Method}} & \multirow{1}{*}{\textbf{Time} \textbf{(s)}$\downarrow$}\tabularnewline
\hline 
\multirow{2}{*}{Random} & \multirow{2}{*}{P2P} & EF & \textbf{44.32}\tabularnewline
 &  & $h$-Edit-R & 50.68\tabularnewline
\hline 
\end{tabular}{\small{}}}}\subcaption{Editing time for combined text-guided and style editing methods}
\par\end{center}%
\end{minipage}
\par\end{centering}
\caption{Editing times per image (in seconds) of our method and baselines across
three tasks: text-guided editing (left), face swapping (top right),
and combined text-guided and style-based editing (bottom right). Experiments
were conducted on an NVIDIA V100 GPU 32GB.\label{tab:Running-time-table}}
\end{table}

Table~\ref{tab:Running-time-table} shows the editing times per image
of our method and baselines for three editing tasks: text-guided editing,
face swapping, and combined text-guided and style editing.

In the text-guided setting, among deterministic-inversion-based methods,
$\Model$-D + P2P requires longer computation time (48.63s) than NP
+ P2P (21.68s), PnP Inv + P2P (37.65s), and NMG + P2P (35.67s) due
to additional U-Net calls for reconstruction and editing term computation.
However, this additional 12-second overhead compared to PnP Inv +
P2P yields significantly improved performance, with a $0.05$ increase
in local CLIP Similarity and $0.6\times10^{-2}$ better LPIPS (Table~\ref{tab:Text-Guided-Editing-Results}).
While NP + P2P achieves the fastest processing time by simply substituting
source embedding for null embedding during editing, it suffers from
substantially lower reconstruction quality. Our favorable trade-off
between computation time and editing quality extends to comparisons
with random-inversion-based methods. LEDITS++ is the fastest as they
leverage high-order solvers \cite{lu2022dpm,Zhao2023,Do2024} - a
feature that could also be incorporated into our method.

In the face swapping task, diffusion-based methods generally require
longer processing time per image compared to GAN-based methods (FaceShifter~\cite{li2020advancing}:
1.31s) or StyleGAN-based approaches (MegaFS~\cite{zhu2021one}: 2.29s,
AFS~\cite{vu2022face}: 1.03s) due to their iterative sampling nature.
Among diffusion-based methods, $\Model$-R (26.34s) and EF (26.11s)
achieve the fastest processing times. Despite sharing the same sampling
steps, $\Model$-R outperforms DiffFace (46.42s) in efficiency as
DiffFace requires additional gaze detection and face parsing models
at each step, beyond the common ArcFace computation. While our $\Model$-R
with 3 optimization steps variant shows slightly increased computation
time (51.36s), it achieves better ArcFace ID similarity compared to
DiffFace with comparable reconstruction quality. Notably, as training-free
approaches, our method and EF offer immediate deployment advantages
over DiffFace and GAN-based methods that require task-specific training.

In the combined text-guided and style editing task, h-Edit-R + P2P
(50.68s) shows only a moderate increase from its text-guided variant
(50.21s) by avoiding U-Net backpropagation for style editing. In contrast,
EF + P2P with FreeDom~\cite{yu2023freedom}'s technique requires
additional backpropagation computation, resulting in a larger time
increase from its text-guided counterpart (32.95s to 44.32s).

\section{Analysis on Metrics\label{sec:Metrics}}

During our text-guided editing experiments, we observed that CLIP
similarity and DINO distance metrics could yield inconsistencies between
quantitative and qualitative results. For CLIP similarity, we hypothesize
that this occurs because the attribute being edited often constitutes
only a small portion of the target prompt. In such cases, even accurate
edits may result in minor improvements in CLIP similarity, whereas
unintended changes to other attributes can lead to significant drops.
Consequently, methods that make no edits and simply preserve the original
image may achieve comparable or better CLIP similarity scores than
methods that successfully perform challenging edits. This phenomenon
is evident with NP and NT - the two strong editing methods capable
of handling challenging edits more effectively than PnP Inv, as shown
in Fig.~\ref{fig:Additional-visualization-h-Edit-D-P2P}. However,
their CLIP similarity scores are lower than that of PnP Inv, as illustrated
in Table~\ref{tab:Text-Guided-Editing-Results}.

In the case of DINO distance, since this metric is computed on the
entire image rather than the non-editing region, it can yield poor
results in significant editing scenarios like changing background
color or removing objects even when original non-editing content is
perfectly preserved.

\section{Ethical Considerations\label{sec:Ethical-Considerations}}

Our work aims to advance the development of effective and efficient
diffusion-based image editing methods, fostering contributions to
both academic research and real-world applications. However, we recognize
that these advancements could be misused for harmful purposes, such
as generating misinformation or damaging individuals' reputations.
To address these risks, it is crucial to implement safeguards that
detect and prevent unethical applications. One potential approach
is to employ a detection framework that analyzes edited images and
flags or discards outputs that violate ethical guidelines or pose
potential harm to society. Such proactive measures can help ensure
that this technology is used responsibly and ethically.